\documentclass[11pt]{article}

\usepackage{authblk}
\usepackage[title]{appendix}

\usepackage{fullpage}
\usepackage[small,bf]{caption}
\usepackage[top=1in,bottom=1in,left=1in,right=1in]{geometry}
\usepackage{fancybox}

\usepackage{mathtools}

\usepackage{algorithm}
\usepackage[noend]{algorithmic}
\usepackage{wrapfig}
\usepackage{microtype}
\usepackage{graphicx}
\usepackage{caption}
\usepackage{subcaption}
\usepackage{booktabs} 
\usepackage{mathrsfs}
\usepackage{amsfonts,amsmath,dsfont,amssymb,amsthm,stmaryrd,bbm}
\usepackage{enumitem}

\newtheorem{conj}{Conjecture}
\newtheorem{thm}{Theorem}
\newtheorem{rmk}[conj]{Remark}

\newtheorem{lem}[conj]{Lemma}
\newtheorem{prop}[conj]{Proposition}

\usepackage{scalerel,stackengine}
\stackMath
\newcommand\reallywidehat[1]{%
\savestack{\tmpbox}{\stretchto{%
  \scaleto{%
    \scalerel*[\widthof{\ensuremath{#1}}]{\kern-.6pt\bigwedge\kern-.6pt}%
    {\rule[-\textheight/2]{1ex}{\textheight}}
  }{\textheight}%
}{0.5ex}}%
\stackon[1pt]{#1}{\tmpbox}%
}

\makeatletter
\newcommand*{\rom}[1]{\expandafter\@slowromancap\romannumeral #1@}
\makeatother

\newcommand{\abs}[1]{\left|#1\right|}

\newcommand{\bB}{\mathbf{B}}
\newcommand {\bC} {\mbox{\boldmath $C$}}
\newcommand {\bD} {\mbox{\boldmath $D$}}
\newcommand {\bE} {\mathbb{E}}
\newcommand {\pr} {\mathbb{P}}

\newcommand {\bI} {\mbox{\boldmath $I$}}

\newcommand {\bP} {\mbox{\boldmath $P$}}
\newcommand {\bQ} {\mbox{\boldmath $Q$}}
\newcommand{\bR}{\mathbf{R}}

\newcommand{\calA}{{\cal A}}
\newcommand{\calB}{{\cal B}}
\newcommand{\calC}{{\cal C}}
\newcommand{\calD}{{\cal D}}
\newcommand{\calE}{{\cal E}}
\newcommand{\calF}{{\cal F}}
\newcommand{\calG}{{\cal G}}
\newcommand{\calH}{{\cal H}}

\newcommand{\calN}{{\cal N}}
\newcommand{\calO}{{\cal O}}

\newcommand{\calQ}{{\cal Q}}
\newcommand{\calR}{{\cal R}}
\newcommand{\calS}{{\cal S}}
\newcommand{\calT}{{\cal T}}
\newcommand{\calU}{{\cal U}}
\newcommand{\calV}{{\cal V}}

\newcommand{\be}{\begin{equation}}
\newcommand{\ee}{\end{equation}}
\newcommand{\beqna}{\begin{eqnarray}}
\newcommand{\eeqna}{\end{eqnarray}}

\newcommand{\p}[1]{\left(#1\right)}
\newcommand{\pp}[1]{\left[#1\right]}
\newcommand{\ppp}[1]{\left\{#1\right\}}
\newcommand{\norm}[1]{\left\|#1\right\|}

\usepackage{xspace}

\newcommand{\avg}{\mathbb{E}}
\newcommand{\prob}{\mathbb{P}}
\newcommand{\bA}{\mathbf{A}}

\begin{document}

\title{Same-Cluster Querying for Overlapping Clusters}

\author{Wasim~Huleihel\thanks{W. Huleihel is with the Department of Electrical Engineering-Systems at Tel-Aviv university, {T}el-{A}viv 6997801, Israel (e-mail:  \texttt{wasimh@mail.tau.ac.il}).} ~~~Arya~Mazumdar\thanks{A. Mazumdar is with the Computer Science Department at the University of Massachusetts Amherst, Amherst, MA 01003, USA (email: \texttt{arya@cs.umass.edu}).} ~~~Muriel~M\'{e}dard\thanks{M. M\'{e}dard is with the department of Electrical Engineering \& Computer Science at Massachusetts Institute of Technology, Cambridge, MA 02139 (email: \texttt{medard@mit.edu}).} ~~~Soumyabrata~Pal\thanks{S. Pal is with the Computer Science Department at the University of Massachusetts Amherst, Amherst, MA 01003, USA (email: \texttt{spal@cs.umass.edu}).}}
\maketitle

\begin{abstract}
Overlapping clusters are common in models of many practical data-segmentation applications. Suppose we are given $n$ elements to be clustered into $k$ possibly overlapping clusters, and an oracle that can interactively answer queries of the form ``do elements $u$ and $v$ belong to the same cluster?'' The goal is to recover the clusters with minimum  number of such queries. This problem has been of recent interest for the case of disjoint clusters. In this paper, we look at the more practical scenario of overlapping clusters, and provide upper bounds (with algorithms) on the sufficient number of queries. We provide algorithmic results under both arbitrary (worst-case) and statistical modeling assumptions. Our algorithms are parameter free, efficient, and work in the presence of random noise. We also derive information-theoretic lower bounds on the number of queries needed, proving that our algorithms are order optimal. Finally, we test our algorithms over both synthetic and real-world data, showing their practicality and effectiveness. 
\end{abstract}

\section{Introduction}
Recently, semi-supervised models of clustering that allow active querying during data segmentation have become quite popular. This includes active learning, as well as data labeling by amateurs via crowdsourcing. Clever implementation of interactive querying framework can improve the accuracy of clustering and help in inferring labels of large amount of data by issuing only a small number of queries. Interactions can easily be implemented (e.g., via captcha), especially if queries involve few data points, like pairwise queries of whether two points belong to the same cluster or not \cite{ahn2016community,ailon2017approximate,ashtiani2016clustering,chien2018query,fss:16,kim2017semi,mazumdar2017clustering,mazumdar2017query,tsourakakis2017predicting,vesdapunt2014crowdsourcing,vinayak2016crowdsourced,wang2013leveraging}.

Until now, the querying model and algorithms/lower bounds are highly tailored towards flat clustering that produces a partition of the data. Consider the problem of clustering from pairwise queries such as above when an element can be part of multiple clusters. 
Such overlapping clustering instances are ubiquitous across areas and many time are more practical model of data segmentation, see \cite{arabie1981overlapping,banerjee2005model,nugent2010overview,zamir1998web}. Indeed, overlapping models are quite natural for communities in social networks or topic models \cite{mao2018overlapping}.
In the supervised version of the problem every element (or data-point) can have multiple labels, and we would like to know all the labels. To see how the querying might work here consider the following input: \{\texttt{Tiger Shark, Grizzly Bear, Blue Whale, Bush Dog,    Giant Octopus, Ostrich, Komodo Dragon}\}. This set can be clustered into the mammals \{\texttt{Grizzly Bear, Blue Whale, Bush Dog}\}, marine-life \{\texttt{Tiger Shark, Blue Whale, Giant Octopus}\}, non-mammals  \{\texttt{Tiger Shark,    Giant Octopus, Ostrich, Komodo Dragon}\}, land-dwellers \{\texttt{Grizzly Bear, Bush Dog,  Ostrich, Komodo Dragon}\}. Quite clearly, this ideal clustering (without labels) is overlapping. If a query of whether two elements belong to the same cluster or not is made then the answer should be `yes' if  there exists a cluster where they can appear together.  If we form a response matrix of size $7 \times 7$ with rows and columns indexed by the order they appeared above in the list and entries being the binary answers to queries, then the matrix would be following:
\begin{wrapfigure}{l}{7cm}
\[
\begin{bmatrix}
&\mathsf{TS} & \mathsf{GB} & \mathsf{BW} & \mathsf{BD} &\mathsf{GO} & \mathsf{Os} &\mathsf{KD}\\
\mathsf{TS} &\ast & 0 & 1 & 0 & 1 &1 &1\\
\mathsf{GB}&0 & \ast & 1 & 1 & 0 & 1 & 1\\
\mathsf{BW}&1 & 1 & \ast & 1 &1 & 0 &0 \\
\mathsf{BD}& 0 & 1 & 1 & \ast & 0 & 1 & 1\\
\mathsf{GO}& 1 & 0 & 1 & 0 & \ast &1 & 1\\
\mathsf{Os}&1 & 1 & 0 & 1 &1 &  \ast &1\\
\mathsf{KD}& 1 & 1 & 0 & 1 &1 &  1 &\ast\\
\end{bmatrix}
\] 
\end{wrapfigure}
What is the minimum number of adaptive queries that we should make to the above  matrix so that it is possible to recover the clusters? In the case when the clusters are not overlapping, the answer is $nk,$ where $n$ is the number of elements and $k$ is the number of possible clusters \cite{mazumdar2017query}. For the case of overlapping clusters it is not clear whether there is a unique clustering that explains the responses. For this, certain extra constraints must be placed: for example, a reasonable assumption is that an element can only be part of $\Delta$ clusters, among the total $k$ possible clusters, $\Delta \ll k$.

Just like the response matrix above, it is possible to form a {\em similarity matrix}. The $(i,j)$th entry of this matrix simply is the number of clusters where the $i$th and the $j$th elements coexists. It is clear that the response matrix is just a quantized version of the similarity matrix. Even when the entire similarity matrix is given, there is no guarantee on uniqueness of overlapping clustering, unless further assumptions are made. In this paper, we primarily aim to recover the clustering from a limited number of adaptive queries to the response matrix. However, in terms of uniqueness guarantees, we often have to stop at the uniqueness of the similarity matrix.

\subsection{Main Results and Techniques} Recovery of overlapping clusters from budgeted same-cluster queries is widely open, and significantly more challenging than the flat clustering counterpart. In fact, none of the techniques proposed in prior literature as mentioned above extends easily to the case when the clusters may overlap. In this paper we tackle this problem for various types of responses. Specifically, in our setting there is an oracle having access to the similarity matrix $\bA\bA^T$, where $\bA$ is the $n\times k$ clustering matrix whose $i$th row is the indicator vector of the cluster membership of $i$th element. In its most powerful mode, when queried the oracle provides the number of clusters where $i$th and the $j$th elements coexists, namely the values of the entries of the matrix $\bA\bA^T$. It turns out, however, that even if one knows the matrix $\bA\bA^T$ perfectly, it is not always possible to recover $\bA$ up to a permutation of columns.\footnote{Given $\bA\bA^T$, it is only possible to recover the clustering matrix $\bA$ up to a permutation of columns, see~Section~\ref{sec:probsetting}.} In fact, if no assumption is imposed on the clustering matrix $\bA$, then recovering $\bA$ from $\bA\bA^T$ is intractable in general. Indeed, even just finding conditions on $\bA$ such that the factorization is unique (up to permutations) is related to the famous unsolved question of \emph{``for which orders a finite projective plane exists?"}~\cite{coxeter2003projective}. It is then clear that we need to impose some assumptions. We tackle this inherent problem in two different approaches. First, in Sections~\ref{subsec:gen1} and~\ref{subsec:gen2}, we propose two \emph{generative} models for $\bA$: 1) a uniform ensemble where a given element can only be part of $\Delta\ge 1$ clusters,\footnote{The case where different items may belong to different numbers of clusters can be handled using the same techniques developed in this paper.} among the total of $k$ clusters, and its membership is drawn uniformly at random among all possible $\binom{k}{\Delta}$ possible placements, 2) the matrix $\bA$ is generated i.i.d. with Bernoulli entries. Then, for these two ensembles, we investigate the above fundamental question, and derive sufficient conditions under which this factorization is unique, along with quasi-polynomial worst-case complexity algorithms for recovering $\bA$ from $\bA\bA^T$. The main immediate implication of this result is that, under certain conditions, the clustering recovery problem reduces to recovering the similarity matrix, placing this objective to be our main task. 

While the above generative models allow us to obtain elegant and neat theoretical results, one might argue that they may not capture many challenges existing in real-world overlapping clustering problems. To this end, in Section~\ref{sec:adversarial}, we go beyond the above generative models and analyze a general worst-case model with no statistical assumptions. Then, under certain realistic assumptions on the clustering matrix, we provide and analyze algorithms solving the recovery problem.

In practice, however, the aforementioned `value' oracle responses might be quite expensive. Accordingly, we study also quantized and noisy variants of these responses. For example, instead of getting direct values from $\bA\bA^T$, the oracle only supplies the learner with (possibly noisy) binary answers on whether arbitrarily picked pair of elements $(i,j)$ appear together in some cluster or not (`same-cluster query'). We consider also the case of dithered oracle, where noise is injected before quantization. For these scenarios (and others), we provide both lower and upper bounds on the number of queries needed for exact recovery. Our lower bounds are obtained using standard information-theoretic results, such as Fano's inequality. For the upper bounds, we design novel randomized algorithms for recovering the similarity matrix, and further show that these algorithms can work when the noise parameter is not given in advance. For example, when $k=O(\log n)$ and $\Delta\ll k$, we show that the sufficient and necessary number of quantized queries is $\Theta(n\log n)$, for the uniform ensemble. Finally, we test our algorithms over both synthetic and real-world data, showing the practicality and effectiveness of our algorithms.

\subsection{Related Work} As mentioned above, there is a series of applied and theoretical works studying the query complexity of `same-cluster' queries for objective-based clustering (such as, $k$-means) and clustering with statistical generative models. In all the cases though, the clusters are assumed to be non-overlapping. From a practical standpoint, entity resolution via crowdsourced pairwise same-entity queries were studied in \cite{fss:16,aaai:17,wang2012crowder,wang2013leveraging,vesdapunt2014crowdsourcing,zou2015crowdsourcing}. The effect of (possibly noisy) `same-cluster' queries in similarity matrix based clustering has been studied in \cite{mazumdar2017clustering,mazumdar2017query,ahn2016community}. On the other hand, the effect of `same-cluster' queries in the efficiency of $k$-means clustering was initiated in \cite{ashtiani2016clustering} - which was then subsequently further studied in \cite{zou2015crowdsourcing,ailon2017approximate,chien2018query}. 

In our approach, we crucially use the `low-rank' structure of the similarity matrix to recover the clustering from a bounded number of responses. Low-rank matrix completion is a well-studied topic in statistics and data science \cite{candes2009exact,keshavan2010matrix}. It is possible to obtain weaker version of some of our results by relying on the results of low-rank matrix completion as black-boxes. However, the specific structure of the similarity matrix under consideration allows us to obtain stronger results. The response matrix is a quantized version of a low-rank matrix. Querying entries of the response matrix can be seen as a so called 1-bit matrix completion problem \cite{davenport20141}. However, most of the recovery guarantees of 1-bit matrix completion depends crucially on certain {\em dither} noise (see, \cite{davenport20141}), which may not be what is allowed in our setting. Finally, we mention \cite{Bonchi13} where the problem of overlapping clustering was considered from a different point of view.

\subsection{Organization} The remaining part of this paper is organized as follows. The model and the learning problem are provided in Section~\ref{sec:probsetting}. Our main results on the query complexity are presented in Section~\ref{sec:mainResults}. In particular, we provide upper bounds (with algorithms) on the sufficient number of queries for each of the scenarios investigated in this paper. These results are also accompanied with information-theoretic lower bounds on the necessary number of queries, which are presented in the appendix due to page limitation. Finally, Section~\ref{sec:experimental} is devoted for a numerical study, where our main results are illustrated empirically. Detailed proofs of all the theoretical results can be found in the supplementary material. 

\section{Model and Learning Problem}\label{sec:probsetting}
\subsection{Oracle Types}
Consider a set of elements $\calN\equiv[n]$ with $k$ latent clusters $\calN_i$, for $i=1,2,\ldots,k$, such that each element in $\calN$ belongs to at least one cluster. This data is represented by an $n\times k$ matrix $\mathbf{A}$, where $\mathbf{A}_{i,j}=1$ if the $i$th elements is in the $j$th cluster. We will denote the $k$-dimensional binary vector representing the cluster membership of the $i$th element by $\mathbf{A}_i$ (i.e., the $i$th row of $\bA$), and will henceforth refer to it as the $i$th \emph{membership vector}. In our setting there is an \emph{oracle} $\calO:\;\calN\times\calN\to\calD$ that when
queried with a pair of elements $(i,j)\in\calN\times\calN$ , returns a natural number $L\in\calD\subset\mathbb{N}$ according to some pre-defined rule. We shall refer to $\calO$ as the \emph{oracle map}. The queries $\Omega\subseteq\calN\times\calN$ can be done adaptively. Our goal is to find the set $\Omega\subseteq\calN\times\calN$ such that $\abs{\Omega}$ is minimum, and it is possible to recover $\ppp{\calN_i}_{i=1}^k$ from the oracle answers. More specifically, the oracle have access to the similarity matrix $\mathbf{A}\mathbf{A}^T$, and when queried with $\Omega$,  answers  according to $\calO$. Given $(i,j)\in\calN\times\calN$, we consider the following oracle maps $\calO$, capturing several aspects of the problem:
\begin{itemize}
\item \emph{Direct responses}: The oracle response is $\calO_{\mathsf{direct}}(i,j) = \mathbf{A}_i^T\mathbf{A}_j$, namely, the number of clusters that elements $i$ and $j$ belong to simultaneously. Note that when the clusters are disjoint, the output is simply an answer to the question ``do elements $i$ and $j$ belong to the same cluster?".
\item \emph{Quantized (noisy) responses}: The oracle response is $\calO_{\mathsf{quantized}}(i,j) = \calQ\p{\mathbf{A}_i^T\mathbf{A}_j}\oplus W_{i,j}$, where $\calQ(x)\triangleq1$ for $x>0$, and $0$, otherwise, and $W_{ij}\sim\mathsf{Bernoulli}(q)$, with $0\leq q\leq1$, independent over pairs $(i,j)$. In the noiseless case, i.e., $q=0$, the oracle response is whether elements $i$ and $j$ appears together in at least one cluster or not. 
In the noisy case, the oracle response is the quantized response with probability $1-q$, and flipped with probability $q$. This can be interpreted as if the quantized responses are further sent through a binary symmetric channel $\mathsf{BSC}(q)$.
\item \emph{Dithered responses}: The oracle response is $\calO_{\mathsf{dithered}}(i,j) = \calQ\p{\mathbf{A}_i^T\mathbf{A}_j+Z_{i,j}}$, where $Z_{ij}\sim\mathsf{Normal}(0,\sigma^2)$, independent over pairs $(i,j)$. In other words, the oracle outputs a dithered and quantized version of the direct responses.
\end{itemize}
For simplicity of notation, throughout the rest of this paper we will denote the oracle response to query $(i,j)\in\calN$ by $\mathbf{Y}_{ij}$, irrespective of the oracle model, which will be clear from the context. 

Even if we permute the columns of $\bA$, the gram matrix $\bA\bA^T$ will be same. However, finding $\bA$ up to a permutation of columns gives a unique clustering. Unfortunately, it turns out that even if we know   $\mathbf{A}\mathbf{A}^T$ perfectly it is not always possible to find $\mathbf{A}$ up to a permutation of columns, namely, the factorization may not be unique.
As an example, consider the following matrices 
\begin{align*}
    \begin{bmatrix}
    1 & 1 & 0 & 0 & 0 & 0 \\ 
    0 & 1 & 1 & 0 & 0 & 0 \\
    0 & 0 & 1 & 1 & 0 & 0 \\
    0 & 0 & 0 & 1 & 1 & 0 \\
    0 & 0 & 0 & 0 & 1 & 1 \\
    1 & 0 & 0 & 0 & 0 & 1 
    \end{bmatrix}
\quad 
\mathsf{and}
\quad 
\begin{bmatrix}
    1 & 0 & 0 & 1 & 0 & 0 \\ 
    1 & 1 & 0 & 0 & 0 & 0 \\
    0 & 1 & 1 & 0 & 0 & 0 \\
    0 & 0 & 1 & 0 & 1 & 0 \\
    0 & 0 & 0 & 0 & 1 & 1 \\
    0 & 0 & 0 & 1 & 0 & 1 
    \end{bmatrix}
\end{align*}
which have the same gram matrix but evidently are not column permutations of each other. Hence, even if we observe all the entries of the gram matrix, it is not possible to distinguish between these two matrices. We tackle this inherent problem in two different approaches.

\subsection{Generative Models} We consider two generative models (random ensembles) for $\mathbf{A}$; \emph{uniform} and \emph{i.i.d.} ensembles, defined as follows. Given $k,\Delta\in\mathbb{N}$, define the set 
\begin{align}
T_k(\Delta)\triangleq\{\mathbf{c}\in\ppp{0,1}^k:\;\mathsf{w_H}(\mathbf{c})=\Delta\},
\end{align}
as the set of all $k$-length binary sequence with Hamming weight ($\mathsf{w_H}$) $\Delta$. Then, we say that $\bA$ belongs to the uniform ensemble if $\mathbf{A}$ is formed by drawing independently its $n$ rows from $T_k(\Delta)$. In the latter ensemble, the matrix $\mathbf{A}$ is an i.i.d. matrix, with each entry being a $\mathsf{Bernoulli}(p)$ random variable, where $0\leq p\leq1$. As mentioned above we are interested in exact recovery of the clusters $\ppp{
\calN}_{i=1}^k$, or equivalently, the clustering matrix $\mathbf{A}$ up to a permutation of the columns of $\mathbf{A}$. More precisely, we define the average probability of error associated with an algorithm which outputs an estimate $\hat{\bA}$ of $\bA$ by
$\mathsf{P_{error}}\triangleq\prob\p{\bigcap_{\pi\in \mathbb{S}_k}\{\hat{\bA}\neq\bA\mathbf{P}_{\pi}\}},
$  
where $\mathbf{P}_\pi$ is the permutation matrix corresponding to the permutation $\pi:[k]\to[k]$, and $\mathbb{S}_k$ is the symmetric group acting on $[k]$. Accordingly, we say that an algorithm properly recovered $\bA$ if $\mathsf{P_{error}}$ is small. This recovery criterion follows from the fact that clustering is invariant to a permutation of the labels. 

In contrast to the above negative example where two different matrices have the same gram matrix, under certain weak conditions, we can show that if the matrix $\mathbf{A}$ is generated according to either one of the above random ensembles, then the factorization is unique up to column permutations. Specifically, we have the following two results, proved in Appendix~\ref{app:1}. 
\begin{lem}{[Uniform Ensemble Uniqueness]}\label{lem:uniq1}
Let $k \ge 2\Delta-2$, $\Delta>2$, and $n >c\cdot\binom{k}{\Delta}\log\binom{k}{\Delta}$, for some $c>0$. Consider two $n\times k$ binary matrices $\bA$ and $\bB$, drawn from the uniform ensemble, and assume that $\bA\bA^{T}=\bB\bB^{T}$. Then, $\bB$ is  a column-permuted version of $\bA$, namely, $\bB = \bA\mathbf{P}$, where $\mathbf{P}$ is a permutation matrix,  with overwhelming probability.
\end{lem}
\begin{lem}{[i.i.d. Ensemble Uniqueness]}\label{lem:uniq2}
Let $\bA$ and $\bB$ be two $n\times k$ binary matrices, drawn from the i.i.d. ensemble with parameter $p$. Assume that $\bA\bA^{T}=\bB\bB^{T}$. If $n>\frac{c\log n+\log k}{1-p\cdot(1-p)^{k}}$ for some $c>0$, then $\bB$ is a column-permuted version of $\bA$, namely, $\bB = \bA\mathbf{P}$, where $\mathbf{P}$ is a permutation matrix, with overwhelming probability.
\end{lem}

These results have a straightforward implication: under the conditions of the Lemmas~\ref{lem:uniq1} and \ref{lem:uniq2}, the clustering recovery problem (i.e., recovering $\bA$) reduces to the recovery of the similarity matrix $\bA\bA^T$, given partial observations $\Omega$ of its entries through the oracle $\calO$. To actually recover $\bA$ from $\bA\bA^T$ we propose Algorithms~\ref{algo:factorization} and \ref{algo:factorization2}, for the uniform and i.i.d. ensembles, respectively. It can be shown that the worst-case computational complexities of these algorithms are $O\p{nk^3\binom{k}{\Delta}^k}$ and $O\p{\binom{n}{k}}$, respectively. This means that the when $k$ is fixed, the computational complexities are polynomial in $n$, while if $k$ grows with $n$, e.g., $k=O(\log n)$, then the computational complexities are quasi-polynomial in $n$.

\begin{algorithm}[b]
\caption{\texttt{Factorization1} Algorithm for recovering $\bA$ from $\bA\bA^{T}$, when rows of $\bA$ belong to $T_k(\Delta)$\label{algo:factorization}}
\begin{algorithmic}[1]
\REQUIRE Similarity matrix $\bA\bA^T$.
\STATE Find a full rank $k\times k$ binary submatrix $\mathbf{T}$ of $\bA\bA^{T}$. 
\STATE Find the set of matrices $\calQ\in \{0,1\}^{k \times k}$ such that for any $\mathbf{Q}\in\calQ$, $\mathbf{Q}\mathbf{Q}^{T}=\mathbf{T}$, and $\norm{\mathbf{Q}_i}_{0}=\Delta, \; \forall i$. 
\FOR{$\mathbf{Q}\in\calQ$}
\STATE For the remaining $n-k$ elements, consider the inferred values with the row indices in $\mathbf{T}$. This creates a system of $k$ linear equations which can be solved for the membership vector of that element. 
\IF{All membership vectors belong to $T_{k}(\Delta)$}
 \STATE Exit the outer FOR loop
\ENDIF 
\ENDFOR
\STATE Return the matrix $\bA$.
\end{algorithmic}
\end{algorithm}
\begin{algorithm}[t]
\caption{\texttt{Factorization2} Algorithm for recovering $\bA$ of $\bA\bA^{T}$ when elements of $\bA$ are i.i.d  $\mathsf{Bernoulli}(p)$ random variables \label{algo:factorization2}}
\begin{algorithmic}[1]
\REQUIRE Similarity matrix $\bA\bA^{T}$ .
\STATE Find a $k\times k$ identity submatrix $\mathbf{T}$ of $\bA\bA^{T}$.
\STATE Choose a permutation matrix $\mathbf{P}$ which forms the membership vector for the indices of the rows in $\mathbf{T}$. 
\STATE For each remaining $n-k$ elements, consider the inferred values with the row indices in $\mathbf{T}$. This creates a system of $k$ linear equations which can be solved for the membership vector of that element.
\STATE Return the matrix $\bA$.
\end{algorithmic}
\end{algorithm}

\subsection{Worst-case Model} While the above generative models allow us to obtain elegant theoretical results, they may be too idealistic for real-world clustering applications. To this end, we also consider a general clustering model where each element in $\calN$ can be belong to at most $\Delta\leq k$ clusters. Note that here each element may belong to \emph{different} number of clusters. We show that under certain geometric conditions on the clusters, recovery is possible using a simple, efficient, and parameter free algorithm.

\section{Algorithms and Their Performance Guarantees}\label{sec:mainResults}
In this section, we present our main theoretical results. Specifically, in Subsections~\ref{subsec:gen1} and \ref{subsec:gen2} our main results about generative models are given. Subsection~\ref{sec:adversarial} is devoted to the worst-case model. 
\subsection{Direct Responses}\label{subsec:gen1}
As a warm up we start with the case of disjoint clusters. In this case, no statistical generative assumption on $\bA$ is needed. The simple algorithm (Algorithm~\ref{algo:oracle_dis}) for this serves as a building block for the other more complicated scenarios considered in this paper.
\begin{algorithm}[tb]
\caption{\texttt{Findmembership} The algorithm for extracting the clustering matrix via queries to oracle. \label{algo:oracle_dis}}
\begin{algorithmic}[1]
\REQUIRE Number of elements: $N$, number of clusters $k$, oracle responses $\calO_{\mathsf{direct}}(i,j)$ for query $(i,j)\in\Omega$, where $i,j \in [N]$.
\STATE Choose a set $\mathcal{S}$ of elements drawn uniformly at random from $\calN$, and perform all pairwise queries corresponding to these $|\calS|$ elements.
\STATE Extract the membership of all the $|\calS|$ elements and find representatives $\calT$ for the $k$ clusters.
\STATE Query each of the remaining $n-|\calS|$ elements with all elements present in $\calT$.
\STATE Return the clusters.
\end{algorithmic}
\end{algorithm}
\begin{algorithm}[tb]
\caption{\texttt{FindSimilarity} The algorithm for extracting the similarity matrix $\bA\bA^{T}$ via queries to oracle. \label{algo:oracle_overlap}}
\begin{algorithmic}[1]
\REQUIRE Number of elements: $N$, number of clusters $k$, oracle responses $\calO_{\mathsf{direct}}(i,j)$ for query $(i,j)\in\Omega$, where $i,j \in [N]$.
\STATE Choose a set $\calS$ of elements drawn uniformly at random from $\calN$, and perform all pairwise queries corresponding to these $|\calS|$ elements.
\STATE Extract a valid membership of all the $|\calS|$ elements  by rank factorization of $\mathbf{A}_{\mathcal{S}}\mathbf{A}_{\mathcal{S}}^{T}$. Then, find a set $\calT \subseteq \calS$ that forms a basis of $\mathbb{R}^{k}$.
\STATE Query each of the remaining $n-|\calS|$ elements with all elements present in $\calT$. Subsequently solve for the membership vector of the unknown element.
\STATE Return the similarity matrix $\bA\bA^{T}$.
\end{algorithmic}
\end{algorithm}
\begin{prop}\label{th:disjoint}
There exists a poly-time algorithm, i.e. Algorithm~\ref{algo:oracle_dis}, which with probability at least $1-n^{-\varepsilon}$ recovers exactly the set of clusters $\calN_1,\ldots,\calN_k$, $\calN_i\cap\calN_j=\emptyset$, for $i\neq j$, using $\abs{\Omega}\geq k\cdot(n-m)+\binom{m}{2}$ queries, $m =\p{n/n_{\min}}\log(kn^{\varepsilon})$, where $\varepsilon>0$ and $n_{\min}$ is the size of the smallest cluster.
\end{prop}
\begin{proof}[Proof Outline:] Pick $m$ elements uniformly at random from $\calN$, and perform all $\binom{m}{2}$ pairwise queries among these $m$ elements. It can be shown that if $m\geq\p{n/n_{\min}}\log(kn^{\varepsilon})$, then with probability $1-n^{-\varepsilon}$, among these $m$ elements there will exist at least one element (representative) from each cluster. Finally, for the remaining $(n-m)$ items, we perform at most $k$ queries to decide which cluster they belong to.
\end{proof}
From Proposition~\ref{th:disjoint}, when $n_{\min}=\Omega(n/k)$, the number of queries needed are $\Omega(kn)$. This result should be contrasted with standard matrix completion results with uniform sampling, which state that $O(kn\log n)$ queries are needed~\cite{CandesT10}. Next, we consider the overlapping case, where $\calN_i\cap\calN_j\neq\emptyset$. In this case the similarity matrix $\bA\bA^T$ is not binary anymore. For a set $\calS\subseteq[n]$, with $m=\abs{\calS}$, we let $\mathbf{A}_{\mathcal{S}}$ be the $m\times k$ projection matrix formed by the rows of $\mathbf{A}$ that correspond to the indices in $\mathcal{S}$. We have the following result.
\begin{thm}\label{th:overlappdirect}
There exists a polynomial-time algorithm, given in Algorithm~\ref{algo:oracle_overlap}, which with probability at least $1-\mathsf{poly}(n^{-1})$ recovers exactly the set of (overlapping) clusters $\calN_1,\ldots,\calN_k$, using $\abs{\Omega}\geq \binom{|\calS|}{2}+k\cdot (n-|\calS|)$ queries, where $|\calS|>\mathsf{S_{uniform}}\triangleq \frac{\binom{k}{\Delta}}{\binom{k-\Delta}{\Delta-1}}[1+c_1\log k+c_2\log n]$, for the uniform ensemble, and $|\calS|>\mathsf{S_{i.i.d.}}\triangleq k-1-\frac{\log k + c_3\log n}{\log\max(p,1-p)}$, for the i.i.d. ensemble, with $c_1,c_2,c_3>0$ arbitrary positive numbers. 
\end{thm}

Let us explain the main idea behind Theorem~\ref{th:overlappdirect}. It is evident from Algorithm \ref{algo:oracle_overlap} that as long as we get a valid subset of elements $\calT \subseteq \calS$ whose membership vectors form a basis of $\mathbb{R}^{k}$, then querying a particular element $i\in\calN$ with all elements in $\calT$ gives $k$ linearly independent equations in $k$ variables that denote the membership of $i$th element to the different clusters. Subsequently, we can solve this system of equations uniquely to obtain the membership vector of $i$th element. Hence, if we choose $|\calS|$ such that there exists a valid subset of $\calS$ forming a basis of $\mathbb{R}^{k}$ with high probability, then we will be done and the sample complexity will be $\binom{|\calS|}{2}+k(n-|\calS|)$. Lemmas~\ref{lem:rank_uniform} and \ref{lem:rank_iid} (see, Appendix~\ref{app:2}), respectively, show that if $|\calS|>\mathsf{S_{uniform}}$ for the uniform ensemble, and $|\calS|>\mathsf{S_{i.i.d.}}$ for the i.i.d. ensemble, then the above property holds.
\begin{rmk}
Note that in the second step of Algorithm~\ref{algo:oracle_overlap} we perform a rank factorization of the matrix $\bA_{\calS}\bA_{\calS}^T$. However, this factorization is not guaranteed to be unique, and accordingly, the resultant rank factorized matrix might be wrong. However, we show in the supplementary material, that even if this is the case, Algorithm~\ref{algo:oracle_overlap} will nevertheless recover the true similarity matrix. 
\end{rmk}
\subsection{Quantized Noisy Responses}\label{subsec:gen2}
We next move to the case where the oracle responses are quantized and noisy, namely, when queried with $(i,j)$, the oracle output is $\calO_{\mathsf{quantized}}(i,j) = \calQ\p{\mathbf{A}_i^T\mathbf{A}_j}\oplus W_{i,j}$, where $W_{i,j}\sim \mathsf{Bernoulli}(q)$. We start with the uniform ensemble, for which we have the following result. 

\begin{thm}\label{th:quantizedUniform}
Assume that $\bA$ was generated according to the uniform ensemble, with $k\geq3\Delta$. Then, there exists a polynomial-time algorithm, given in Algorithm~\ref{algo:oracle}, which with probability $1-n^{-\epsilon}$, recovers the similarity matrix $\bA\bA^T$, using $\abs{\Omega}\geq \binom{|\calS|}{2}+|\calS|\cdot(n-|\calS|)$ queries, where for any $\varepsilon>0$,
\begin{align}
|\mathcal{S}|>2(1-2q)^{-4}\binom{k}{\Delta}^2\pp{\binom{k-2\Delta+1}{\Delta}-\binom{k-2\Delta}{\Delta}}^{-2}\log(2 n^{2+\varepsilon}).\label{th3:S_size}
\end{align}
\end{thm}
\begin{algorithm}[tb]
\caption{\texttt{Noisy Quantized Responses} The algorithm for extracting membership of elements via queries to oracle. \label{algo:oracle}}
\begin{algorithmic}[1]
\REQUIRE Number of elements: $N$, number of clusters $k$, oracle responses $\calO_{\mathsf{quantized}}(i,j)$ for query $(i,j)\in\Omega$, where $i,j \in [N]$.
\STATE Choose a set $\mathcal{S}$ of elements drawn uniformly at random from $\calN$, and perform all pairwise queries corresponding to these $|\calS|$ elements.
\STATE Run Algorithm \texttt{NoisyInferSupport1} to infer $\langle \mathbf{A}_i,\mathbf{A}_j \rangle$ for each pair of entries $(i,j) \in \mathcal{S}$.
\STATE Extract the membership of all the $|\calS|$ elements up to a permutation of the clusters.
\STATE Query each of the remaining $n-|\calS|$ elements with all elements present in $\calS$. Subsequently run algorithm \texttt{NoisyInferSupport2} for each query and solve for the membership vector of the unknown element.
\STATE Return the similarity matrix $\bA\bA^{T}$. 
\end{algorithmic}
\end{algorithm}
\begin{algorithm}[tb]
\caption{\texttt{NoisyInferSupport1} The algorithm for inferring $\langle \mathbf{A}_i,\mathbf{A}_j \rangle$ for two fixed entries $(i,j) \in \mathcal{S}$. \label{algo:Noisyinfersupport1}}
\begin{algorithmic}[1]
\REQUIRE  Set $\mathcal{S}$ where every pairwise value is observed, and indices $i,j \in \mathcal{S}$
   \STATE Define $\Delta+1$ numbers $E_{\ell}= (|\mathcal{S}|-2)\Big((1-q)^{2}-2(1-2q)(1-q)\frac{{k-\Delta \choose \Delta}}{{k \choose \Delta}}+(1-2q)^{2}\frac{{k-2\Delta+\ell \choose \Delta}}{{k \choose \Delta}} \Big)$ for $\ell=0,1,\dots,\Delta$
   \STATE Calculate $T_{ij}=\sum_{\substack{r \in \mathcal{S} \\ r\neq i,j }} \mathds{1}[\mathbf{Y}_{ir}=1 \cap \mathbf{Y}_{jr}=1]$
   \STATE Return $\mathrm{arg}\min_{\ell} |T_{ij}-E_{\ell}|$ 
\end{algorithmic}
\end{algorithm}
\begin{algorithm}[tb]
\caption{\texttt{NoisyInferSupport2} The algorithm for inferring $\langle \mathbf{A}_i,\mathbf{A}_j \rangle$ for $i\in\mathcal{S}, j \not \in \mathcal{S}$. \label{algo:Noisyinfersupport2}}
\begin{algorithmic}[1]
\REQUIRE  Set $\mathcal{S}$ where every pairwise value is observed, Indices $i \in \mathcal{S}, j \not \in \mathcal{S}$.
   \STATE Define $\Delta$ numbers $E_{\ell}= (|\mathcal{S}|-1)\Big((1-q)^{2}-2(1-2q)(1-q)\frac{{k-\Delta \choose \Delta}}{{k \choose \Delta}}+(1-2q)^{2}\frac{{k-2\Delta+\ell \choose \Delta}}{{k \choose \Delta}} \Big)$ for $\ell=0,1,\dots,\Delta$
   \STATE Calculate $T_{ij}=\sum_{\substack{r \in \mathcal{S} \\ r \neq i }} \mathds{1}[\mathbf{Y}_{ir}=1 \cap \mathbf{Y}_{jr}=1]$
   \STATE Return $\mathrm{arg}\min_{\ell} |T_{ij}-E_{\ell}|$ 
\end{algorithmic}
\end{algorithm}
The main idea behind Algorithm~\ref{algo:oracle} is the following: we first choose a random subset $\calS\subseteq\calN$ of elements, such that \eqref{th3:S_size} holds, and perform all pairwise queries among these elements. Using the resultant queries we infer the unquantized inner products of $\bA_i^T\bA_j$, for any $(i,j)\in\calS$. To this end, we count the number of elements which are similar to both the profile of elements $i$ and $j$ (see, the definition of $T_{ij}$ in Algorithm~\ref{algo:oracle}). Intuitively, it makes sense that the more similar the two elements $i$ and $j$ themselves are, the more the number of elements should be which are similar to both of them. We show that the condition in \eqref{th3:S_size} suffices to make the count highly concentrated around its mean, and accordingly, outputs the true value of $\bA_i^T\bA_j$. Finally, the remaining $(n-|\calS|)$ elements are queried with the elements in $\calS$, and then we apply the above inferring procedure once again. We emphasize here that the exponential dependency of the upper bounds on $\Delta$ is inherent, as the information-theoretic lower bounds in Appendix~\ref{app:IT_limit} suggest.

It turns out that the above idea is capable to handle the other scenarios considered in this paper, albeit with certain technical modifications. Indeed, for the i.i.d. ensemble, we need an additional step before we can use the idea mentioned above. This is mainly because of the fact that analyzing the aforementioned count statistic requires the knowledge of support size of $\bA_i$ and $\bA_j$ (which is fixed in the uniform ensemble). An easy way around this problem is to infer first the $\ell_{0}$-norm of every element by counting the number of other elements that are similar. As before, under certain conditions, this count behaves differently for different values of the actual $\ell_{0}$-norm value and therefore we can infer the correct value. Once this step is done, everything else falls into place. Due to space limitation we relegate the pseudo-algorithm for the i.i.d. setting to the appendices. We have the following result.

\begin{thm}\label{th:quantizedBern}
Assume that $\bA$ was generated according to the i.i.d. ensemble. Then, there exists a polynomial-time algorithm, given in Algorithm~\ref{algo:oracle2}, which with probability $1-n^{-\epsilon}$, recovers the similarity matrix $\bA\bA^T$, using $\abs{\Omega}\geq {|\calS| \choose 2}+|\calS|\cdot(n-|\calS|)$ queries, where for any $\varepsilon>0$,
\begin{align}
|\mathcal{S}|>2p^{-2}(1-2q)^{-4}(1-p)^{2-2k}\log(2 n^{2+\varepsilon}).
\end{align}
\end{thm} 


In practice, the value of the noise parameter $q$ might be unknown to the learner. In this case, we will not know the expected values of the triangle counts under the different hypotheses a-priori, and thus our previous algorithms cannot be used directly. Fortunately, however, it turns out that with a simple modification, our algorithms can be used also when $q$ is unknown. We have the following result stated for the uniform ensemble. A similar result can be obtained also for the i.i.d. ensemble.

\begin{algorithm}[tb]
\caption{\texttt{Noisy Responses} The algorithm for extracting membership of elements via queries to oracle. \label{algo:noisy}}
\begin{algorithmic}[1]
\REQUIRE Number of elements: $N$, number of clusters $k$, oracle responses $\calO_{\mathsf{quantized}}(i,j)$ for query $(i,j)\in\Omega$, where $i,j \in [N]$.
\STATE Choose a set $\mathcal{S}$ of elements drawn uniformly at random from $\calN$, and perform all pairwise queries corresponding to these $|\calS|$ elements. Compute 
$T_{ij}=\sum_{\substack{r \in \mathcal{S} \\ r\neq i,j }} \mathds{1}[\mathbf{Y}_{ir}=1 \cap \mathbf{Y}_{jr}=1]$ for all  $i,j \in \mathcal{S}$.

\STATE Query the remaining $n-|\calS|$ elements with all elements present in $\calS$. Subsequently compute for all $i \in \mathcal{S},j \notin \mathcal{S}$, 
$T_{ij}=\sum_{\substack{r \in \mathcal{S}\setminus \{x_j\} \\ r\neq i,x_j }} \mathds{1}[\mathbf{Y}_{ir}=1 \cap \mathbf{Y}_{jr}=1]$
where $x_j$ is an arbitrarily selected element from $\calS$ such that $x_j \neq i$ 
\STATE Group all the ${|\mathcal{S}| \choose 2}+|\calS|\cdot (n-|\calS|)$ counts $T_{ij}$ into $\Delta+1$ groups such that the difference between any two intra-group points is smaller than the difference between any two inter-group points. If not possible, return NOT POSSIBLE.
\STATE Order the groups by their value and label them  by assigning the hypothesis $H_{\ell}$ to the $\ell^{th}$ group in the order.
\STATE Assign $\langle \mathbf{A}_i,\mathbf{A}_j \rangle$ to be $\ell$ for queries $Q=(i,j)$ such that $T_{ij}$ belonged to the $\ell^{th}$ group.
\STATE Extract the membership of all the $|\calS|$ elements present in $\mathcal{S}$ by a rank factorization (may not be unique). Obtain $k$ linear independent vector from the solution space and represent
them as a $\mathcal{T}$.
\STATE Solve the membership vectors of all elements by solving the $k$ linearly independent equations obtained by getting the inner product with $\mathcal{T}$.
\STATE Return the similarity matrix $\bA\bA^{T}$.
\end{algorithmic}
\end{algorithm}
\begin{thm}\label{th:quantizednoisyunkown}
Assume that $\bA$ was generated according to the uniform ensemble with $k\geq3\Delta$, and $n>10\binom{k}{\Delta}\log n$. Then, there exists a polynomial-time algorithm, given in Algorithm~\ref{algo:noisy}, independent of the noise parameter $q$, which with probability $1-n^{-\epsilon}$, recovers the similarity matrix $\bA\bA^T$, using $\abs{\Omega}\geq {|\calS| \choose 2}+|\calS|\cdot(n-|\calS|)$ queries, where for any $\varepsilon>0$,
\begin{align}
|\mathcal{S}|>18(1-2q)^{-4}\binom{k}{\Delta}^2\pp{\binom{k-2\Delta+1}{\Delta}-\binom{k-2\Delta}{\Delta}}^{-2}\log(2 n^{2+\varepsilon}).
\end{align}
\end{thm}
Comparing Theorems~\ref{th:quantizedUniform} and \ref{th:quantizednoisyunkown}, we notice that the query complexity grows by a multiplicative constant factor only. Note that the additional technical condition $n>10\binom{k}{\Delta}\log n$ is rather weak, and naturally satisfied, for example, in the regime $k=O(\log n)$. We mention here that the computational complexities of each of the above algorithms are roughly of the order of $O(|\Omega|+|\calS|^3)$, dominated by querying $|\Omega|$ random samples and applying a rank factorization on the gram matrix $\bA_{\calS}\bA_{\calS}^T$. Finally, note that since we deal with quantized responses without any continuous dithering, matrix completion results cannot be used. In fact, without dithering matrix completion algorithms will fail on quantized data \cite{davenport20141}, as they do not exploit the discrete structure of the data, which is the main source for the success of our algorithms.

\subsection{Dithered Responses}

In this subsection, we present our main result concerning dithered responses, i.e., $\calO_{\mathsf{dithered}}(i,j) = \calQ\p{\mathbf{A}_i^T\mathbf{A}_j+Z_{i,j}}$, where $Z_{ij}\sim\mathsf{Normal}(0,\sigma^2)$, independently over pairs $(i,j)$. Here, we consider the uniform ensemble only, but using the same techniques developed in this paper, the i.i.d. ensemble can be handled too. Let $Q(\cdot)$ denote the $Q$-function, namely, for any $x\in\mathbb{R}$, $Q(x)\triangleq\int_{x}^\infty\frac{1}{\sqrt{2\pi}}e^{-t^2/2}\mathrm{d}t$. Finally, for $\ell=0,1$, define
\begin{align}
G_\ell(k,\Delta)\triangleq \mathbb{E}\pp{\left.Q\p{\frac{\mathbf{A}_1^T\mathbf{A}_3}{\sigma}}Q\p{\frac{\mathbf{A}_2^T\mathbf{A}_3}{\sigma}}\right|\mathbf{A}_1^T\mathbf{A}_2=\ell},
\end{align}
where $\{\mathbf{A}_i\}_{i=1}^3$ are three statistically independent random vectors drawn from $T_{k}(\Delta)$. The algorithm in this setting is in fact the same as Algorithm~\ref{algo:oracle}, but with Algorithms~\ref{algo:Noisyinfersupport1} and \ref{algo:Noisyinfersupport2} replaced with Algorithms~\ref{algo:infersupport1dithered} and \ref{algo:infersupport2dithered} in Appendix~\ref{app:dithered}. With these definitions, we are ready to state our main result.
\begin{thm}\label{th:ditheredExact}
Assume that $\bA$ was generated according to the uniform ensemble. Then, there exists a polynomial-time algorithm, which with overwhelming probability, recovers the similarity matrix $\bA\bA^T$, using $\abs{\Omega}\geq {|\calS| \choose 2}+|\calS|\cdot(n-|\calS|)$ queries, where for any $\varepsilon>0$,
\begin{align}
|\mathcal{S}|>\frac{2\log(2 n^{2+\varepsilon})}{|G_1(k,\Delta)-G_0(k,\Delta)|^2}.\label{th3:S_size_dithered}
\end{align}
\end{thm}

\subsection{Information-Theoretic Lower Bounds}

In this subsection, we provide information-theoretic lower-bounds on the query complexity for exact recovery of the clustering matrix $\bA$, associated with the scenarios considered in this paper. We denote by $\calH_2(x)$ the binary entropy of $x\in(0,1)$, namely, $\calH_2(x)\triangleq-x\log_2x-(1-x)\log_2(1-x)$, and denote by $\star$ the binary convolution, i.e., $p\star q \triangleq (1-p)q+p(1-q)$. We have the following results proved in the sequel.
\begin{thm}{[i.i.d. Ensemble]}\label{thm:coverse_iid}
Assume that $\bA$ was generated accordingly to the i.i.d. ensemble with parameter $p$. Then, for any adaptive algorithm, in order to achieve $\mathsf{P_{error}}\leq\delta$, the necessary query complexity is
\begin{enumerate}
\setlength\itemsep{-1em}
\item For $\calO_{\mathsf{direct}}$:
\begin{align}
\abs{\Omega}\geq \frac{nk}{\log k}\cdot[\calH_2(p)-\delta].\label{conv_iid_direct}
\end{align}
\item For $\calO_{\mathsf{quantized}}$:
\begin{align}
\abs{\Omega}\geq nk\cdot\frac{\calH_2(p)-\delta}{\calH_2\p{q\star\pp{1-(1-p^2)^{k}}}-\calH_2(q)}.\label{conv_iid_quantized}
\end{align}
\item For $\calO_{\mathsf{dithered}}$:
\begin{align}
\hspace{-0.4cm}\abs{\Omega}\geq \frac{nk\cdot[\calH_2(p)-\delta]}{\calH_2\pp{\bE Q\p{\frac{\bA_1^T\bA_2}{\sigma}}}-\bE\calH_2\pp{ Q\p{\frac{\bA_1^T\bA_2}{\sigma}}}}.\label{conv_iid_dithered}
\end{align}
\end{enumerate}
\end{thm}

\begin{thm}{[Uniform Ensemble]}\label{thm:coverse_uniform}
Assume that $\bA$ was generated accordingly to the uniform ensemble with parameter $\Delta$. Then, for any adaptive algorithm, in order to achieve $\mathsf{P_{error}}\leq\delta$, the necessary query complexity is
\begin{enumerate}
\setlength\itemsep{-1em}
\item For $\calO_{\mathsf{direct}}$:
\begin{align}
\abs{\Omega}\geq nk\cdot\frac{\frac{1}{k}\log\binom{k}{\Delta}-\delta}{\log\Delta}.\label{conv_uniform_direct}
\end{align}
\item For $\calO_{\mathsf{quantized}}$:
\begin{align}
\abs{\Omega}\geq nk\cdot\frac{\frac{1}{k}\log\binom{k}{\Delta}-\delta}{\calH_2\p{q\star\frac{{{k-\Delta}\choose{\Delta}}}{{{k}\choose{\Delta}}}}-\calH_2(q)}.\label{conv_uniform_quantized}
\end{align}
\item For $\calO_{\mathsf{dithered}}$:
\begin{align}
\hspace{-0.4cm}\abs{\Omega}\geq \frac{nk\cdot[\frac{1}{k}\log\binom{k}{\Delta}-\delta]}{\calH_2\pp{\bE Q\p{\frac{\bA_1^T\bA_2}{\sigma}}}-\bE\calH_2\pp{ Q\p{\frac{\bA_1^T\bA_2}{\sigma}}}}.\label{conv_uniform_dithered}
\end{align}
\end{enumerate}
\end{thm}

\subsection{Beyond Generative Models: Arbitrary Worst-Case Instances}\label{sec:adversarial}

In this subsection, we consider the worst-case model, where we do not impose any statistical assumptions, and assume that each element belong to at most $\Delta$ clusters. We focus on noiseless quantized oracle responses, but also discuss  the direct responses scenario in Section~\ref{sec:experimental}. For this case, we propose Algorithm~\ref{algo:adv}. We have the following result.
\begin{thm}\label{thm:DeltaBigger2_2}
Let $\calN_i$ be the set of elements which belong to the $i$'th cluster. If, for every cluster $i \in [k]$, we have $|\calN_i \setminus \{\bigcup_{j:j\neq i} \calN_j\}|>\alpha\cdot n$, for some $\alpha>0$, then by using Algorithm~\ref{algo:advgen}, ${{|S|} \choose{ 2}}+|S|\cdot(n-|S|)$ queries are sufficient to recover the clusters, where $\alpha\cdot |S|=\log k+\log n$.
\end{thm}
As mentioned above, Algorithm~\ref{algo:advgen} is parameter free, do not require the knowledge of $\Delta$, and efficient. For the special case of $\Delta=2$, we show in Appendix~\ref{app:Delta2} (see, Theorem~\ref{thm:advDelta2}) that the same result holds under less restrictive conditions than those in Theorem~\ref{thm:DeltaBigger2_2}. In fact, in Appendix~\ref{app:GeneralDelta} we conjecture that Theorem~\ref{thm:DeltaBigger2_2} holds true under a similar assumption as in Theorem~\ref{thm:advDelta2}. Depending on the dataset, the scaling of $\alpha$ in Theorem~\ref{thm:DeltaBigger2_2} w.r.t. $(\Delta,k,n)$ may vary widely. For example, in the non-overlapping case, $\alpha = k_{\min}/n\leq1/k$, where $k_{\min}$ is the size of the smallest cluster, which implies that the query complexity in the best scenario is $O(nk\log n)$, which is consistent with our results in the previous section. In the worst-case, a positive $\alpha$ could be as small as $1/n$ (unreasonable in real-world datasets), which implies a query complexity of $O(n^2)$. This is much higher than our average case results, as expected. More generally, note that $\alpha$ decreases as a function of $\Delta$, which implies that the query complexity increases with $\Delta$. For example, consider the example of 3 equally-sized clusters $A$, $B$ and $C$. Suppose $\Delta =1$ and in that case  $|A \setminus {B \cup C}| = |A| = n/3$, implying that $\alpha =1/3$. Now suppose that $\Delta =2$. In this case $A \cap B$ and $A\cap C$ are non-empty and therefore  $|A \setminus {B \cup C}| = |A|- |A\cap B| -|A\cap C| < n/3$, namely, $\alpha$ is less than $1/3$.
\begin{algorithm}[htb]
\caption{\texttt{Worst-case quantized responses} \label{algo:advgen}}
\begin{algorithmic}[1]
\REQUIRE $N$, $k$, and oracle responses $\calO_{\mathsf{quantized}}(i,j)$ for every query $(i,j)\in\Omega$.
\STATE Choose a set $\mathcal{S}$ of elements drawn uniformly at random from $[N]$, and perform all pairwise queries corresponding to these $|\calS|$ elements.
\STATE Construct a graph $\calG=(\calV,\calE)$ where the vertices are the $|\calS|$ sampled elements. There exist an edge between elements $(i,j)$ only if they are determined to be similar by the oracle. 
\STATE Construct the maximal cliques of the graph $\calG$ such that all edges in $\calE$ are covered. Each maximal clique forms a cluster. 
\STATE Query each of the remaining $n-|\calS|$ elements with all elements present in $\calS$. For each cluster, if an element is similar with all the elements in that particular cluster, then assign the element to that cluster. Return the obtained clusters.
\end{algorithmic}
\end{algorithm}

\begin{table}
  \caption{Sample complexities for $k=O(\log n)$ and $\Delta\ll k$}
  \label{table:1}
  \centering
  \begin{tabular}{lll}
    \toprule
    \cmidrule(r){1-2}
    Oracle Type     & Lower-Bound     & Upper-Bound \\
    \midrule
    Direct responses (disjoint) & $O(nk)$  & $\Omega(nk)$     \\
    Direct responses (overlapping) & $O(nk)$  & $\Omega(nk)$     \\
    Quantized responses & $O(n\cdot\mathsf{polylog}\;n)$  & $\Omega(n\cdot\mathsf{polylog}\;n)$   \\
    Quantized responses (worst-case, $\alpha = n^{-c}$) & $\mathsf{NA}$  & $\Omega(n^{1+c}\cdot\mathsf{polylog}\;n)$ \\
    \bottomrule
  \end{tabular}
\end{table}

\begin{figure*}[tb]
  
  \begin{subfigure}[t]{0.5\textwidth}
     \centering\includegraphics[height=1.5in]{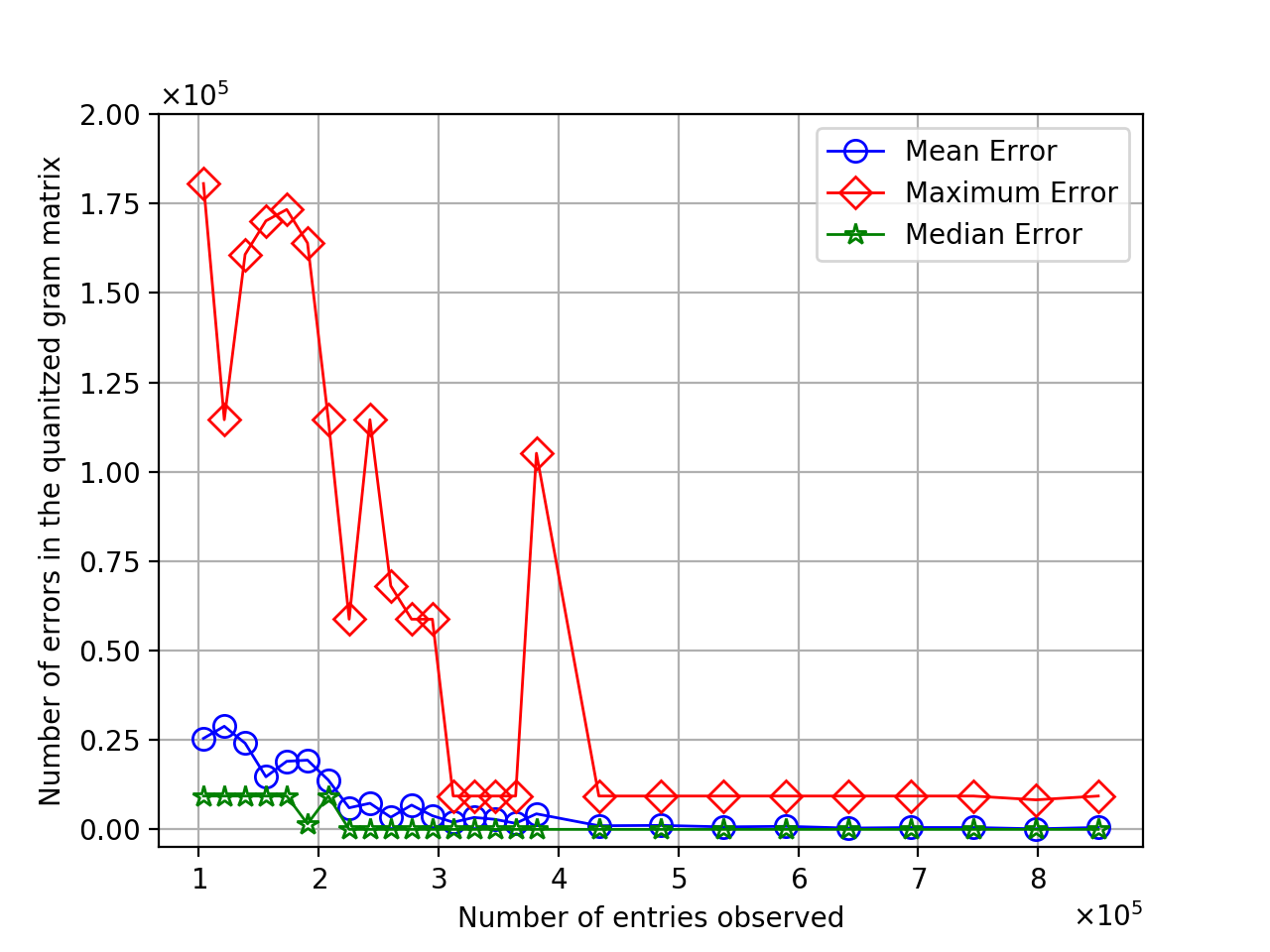}
    \caption{\small Mean, median and maximum errors for $\Delta=2$.}
          ~\label{fig:err}
  \end{subfigure}
\begin{subfigure}[t]{0.5\textwidth}
   \centering\includegraphics[height=1.5in]{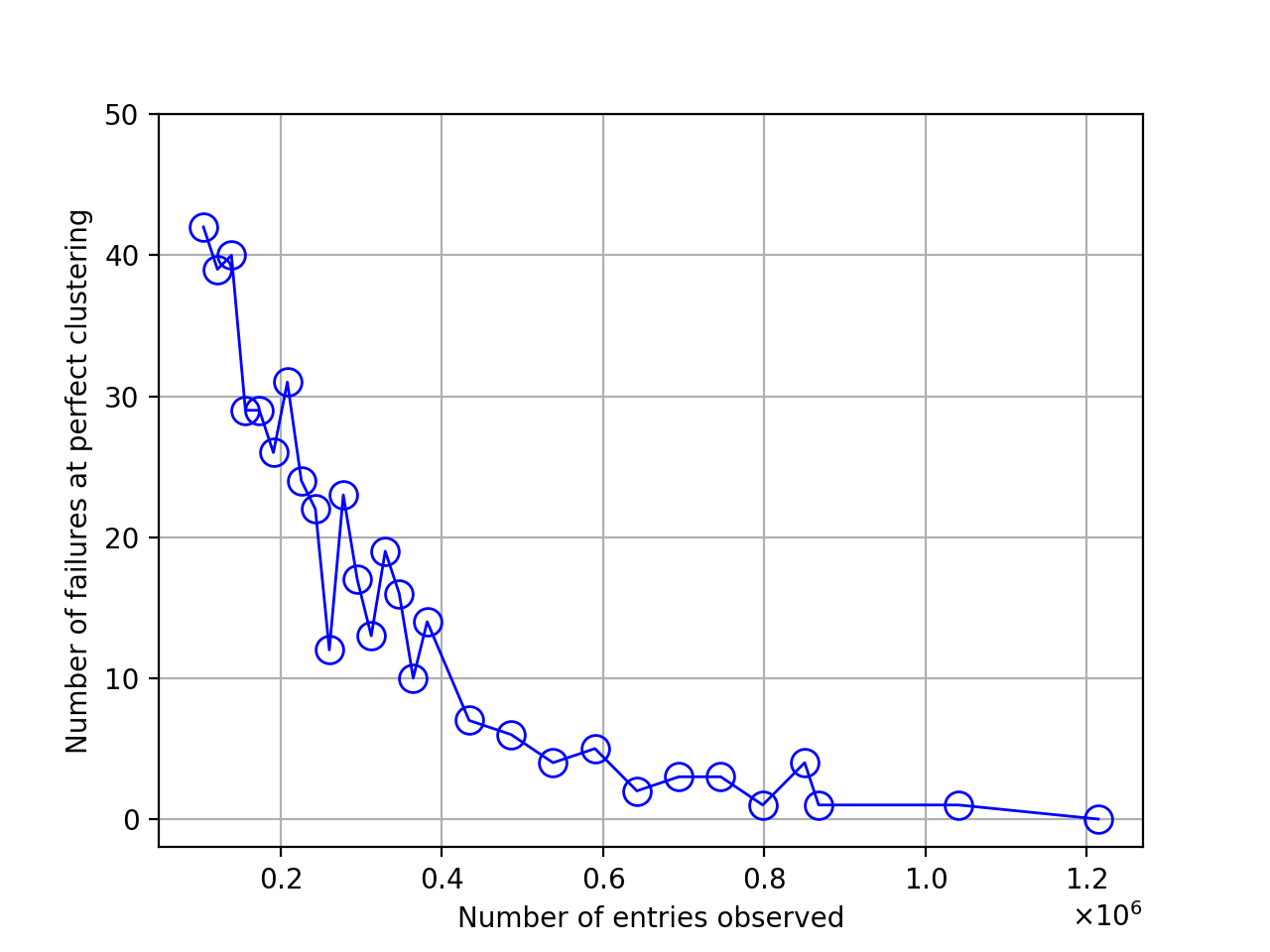}
  \caption{\small Number of failures for $\Delta=2$.}
       ~\label{fig:err2}
 \end{subfigure}
  \begin{subfigure}[t]{0.5\textwidth}
     \centering\includegraphics[height=1.5in]{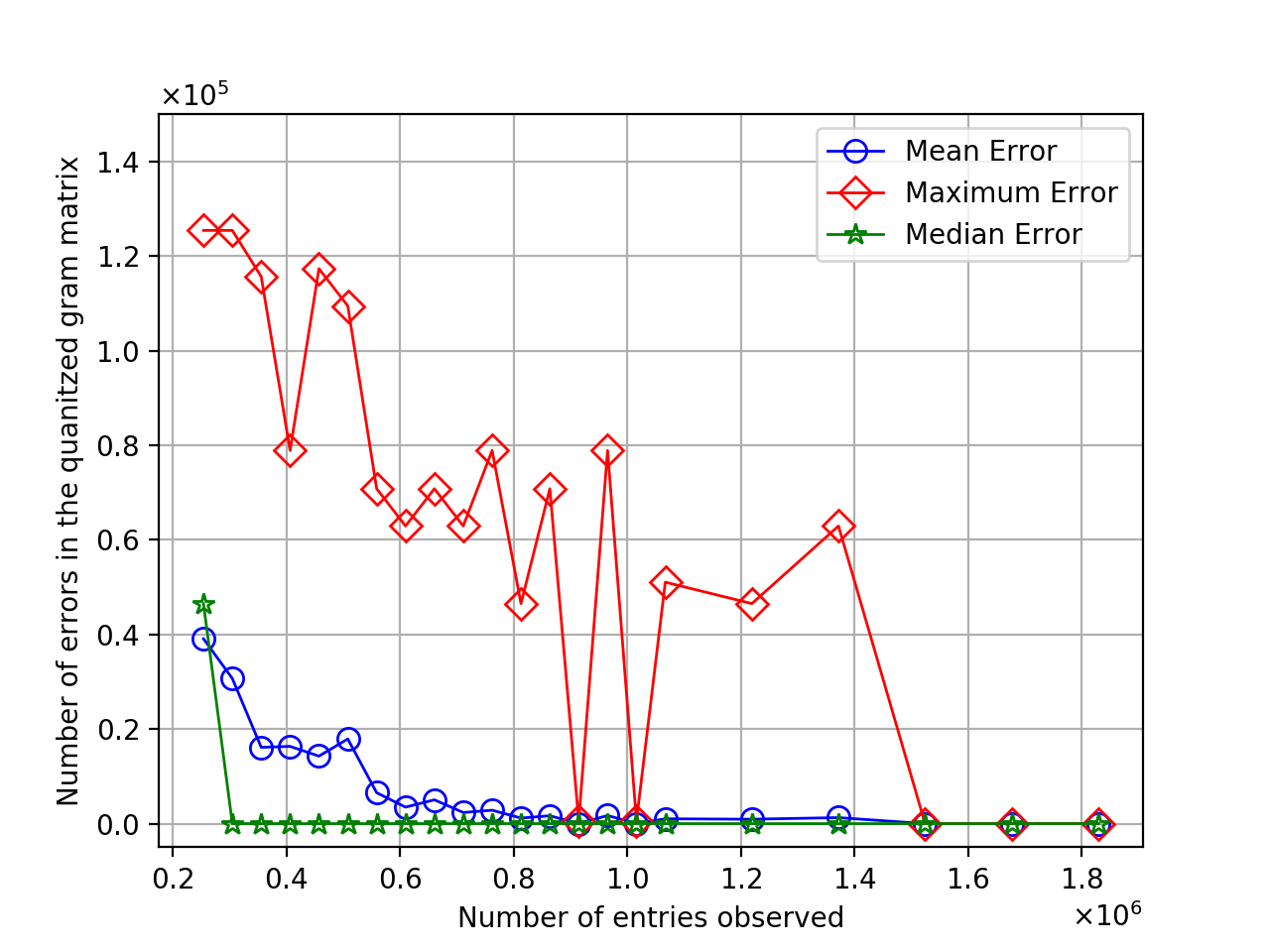}
    \caption{\small Mean, median and maximum errors for $\Delta=3$.}
          ~\label{fig:err3}
  \end{subfigure}
\begin{subfigure}[t]{0.5\textwidth}
   \centering\includegraphics[height=1.5in]{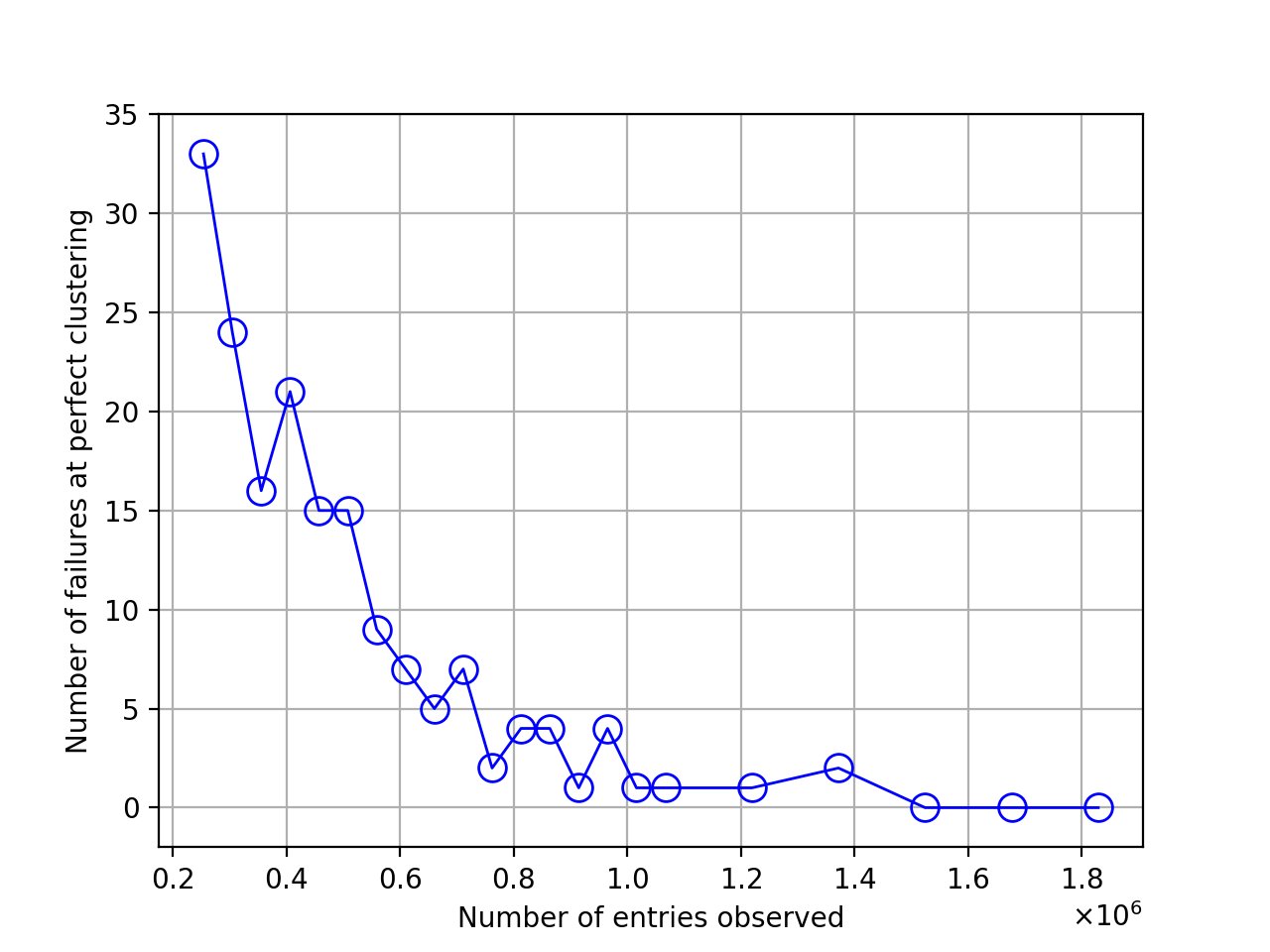}
   \caption{\small Number of failures for $\Delta=3$.}
       ~\label{fig:err4}
 \end{subfigure}%
\hfill
\caption{\small Results of our techniques on MovieLens dataset.}
\end{figure*}

\subsection{Summary Table}\label{subsec:summary_table}

Table~\ref{table:1} summarizes the scaling of our lower and upper bounds on the sample complexities for each of the different oracle types considered in this paper. In the table, we opted to focus on the regime where $k=O(\log n)$ and $\Delta\ll k$, as we found it to be the most interesting one. We also assume that the noise parameter $q$ is fixed. Note, however, that our theoretical results are general and apply for any scaling of $k$ and $\Delta$. Also, since the scalings of the sample complexities associated with the uniform and the i.i.d. ensembles, as well as when the noise parameter $q$ is unknown, are similar, we choose to combine them together. We also present the worst-case scenario in Theorem~\ref{thm:DeltaBigger2_2}, assuming that $\alpha=n^{-c}$, for some $c\in(0,1)$. For simplicity of presentation, we do not explicitly present the scaling of the lower and upper bound on $\mathsf{polylog}$ factors. For the regime above, we can see that the scaling of the upper and lower bounds w.r.t. $n$ is the same up to constants for direct responses. For quantized responses there is a $\mathsf{polylog}$ factor difference between the obtain upper and lower bounds.

\section{Experimental Results}\label{sec:experimental}
We focus on real-world data from 
the popular Movielens dataset for our experiments. The dataset we used describes 5-star rating and free-text tagging activity from Movielens, a movie recommendation service. It contains $100836$ ratings and $3683$ tag applications across $9742$ movies.
\subsection{Quantized Query Responses}
In order to establish our results we chose the following categories \texttt{Mystery, Drama, Sci-Fi, Horror,} and \texttt{Crime}, and first selected only those movies that belong to at most two categories and at least one category (i.e., $\Delta=2$). The total number of such movies were $3470$ and in accordance to the statement in Theorem~\ref{thm:advDelta2}, we have $\alpha=0.0152$ (there are $53$ movies that belong to \texttt{Mystery} but does not belong to \texttt{Sci-fi} and \texttt{Horror}). The total number of possible queries is about $1.2 \times 10^7$ and the number of queries that are sufficient theoretically is \texttt{2948935} (theoretical value of $|\calS|$ is $245$). We ran Algorithm \ref{algo:advgen} (running Algorithm \ref{algo:adv} requires parsing all possible clique covers which is computationally hard) with different values of $|\calS|$ (number of movies randomly chosen in the first step) and since the movies are sampled randomly, we ran $50$ trials for each value of $|\calS|$. Finally, after the final clustering is provided by the algorithm, we calculated the gram matrix from the resulting clustering and compared it with the gram matrix of the ground truth clustering. Figure~\ref{fig:err} shows the mean, median, and maximum error as a function of the total number of queries accrued by Algorithm~\ref{algo:advgen}. Here, the error refers to the total number of different entries in the estimated and true gram matrices. We can observe that the mean error almost reaches zero around $3\times 10^5$ queries (about $2.5\%$ of total). Figure~\ref{fig:err2} presents the total number of failures in perfect clustering (trials when error is larger than $1$) among the $50$ trials for each value of $|\calS|$ we have chosen. We obtain perfect clustering in all the 50 trials first using $1.2 \times 10^{6}$ queries ($\approx10\%$ of total). Note that since Theorem~\ref{thm:DeltaBigger2_2} gives a sufficient condition on $T$ only, in practice we can take smaller values for $T$ and still guarantee recovery. Of course, the smaller the size of $\mathcal{S}$ is, the sample and computational complexities are smaller as well. We repeated the experiment with the same set of categories as in the previous one but this time, we  included movies that belonged to at most three clusters at the same time i.e ($\Delta=3$). The total number of such movies is $5082$ and therefore the total number of possible queries is about $2.59 \times 10^{7}$. Again, we conducted $50$ trials for each chosen value of $|\calS|$ and as before, we plotted the mean, median and maximum error in Figure \ref{fig:err3} and the number of failures in perfect clustering in Figure \ref{fig:err4} for this setting.
Notice that the mean error drops almost to zero at about $6\times 10^5$ queries ($2.32\%$ of total) and perfect clustering over all $50$ trials is achieved at $1.5\times 10^{6}$ queries ($5.8\%$ of total).  
We would like to point out over here that Algorithm~\ref{algo:advgen} is \emph{parameter free}, and provides a non-trivial solution even when the number of queries are far below the theoretical threshold limit. It turns out that the experimental threshold is  better than the theoretical threshold on queries for perfect clustering. Moreover, Algorithm~\ref{algo:advgen} is efficient in partial clustering as well since the error drops very fast as the number of queries is increased. 

\subsection{Unquantized Query Responses.}

In this experiment, we used the Movielens dataset again and chose $5$ classes \texttt{Mystery, Drama, IMAX, Sci-Fi,} and \texttt{Horror}, and selected those elements who belonged to at most two categories and at least one category (i.e., $\Delta=2$). The total number of such movies are $3270$, and for each query involving two movies, we obtain back the unquantized similarity (total number of categories they both belong to). We follow Algorithm~\texttt{FindSimilarity} very closely but with a small modification. Indeed, note that Algorithm~\ref{algo:oracle_overlap} is designed so that the guarantees hold under a specific stochastic assumption. More concisely, the necessary size of $\mathcal{S}$ is not defined for arbitrary real-world datasets. Note, however, that the main objective in the first part of the algorithm is to select a number of elements so that the gram matrix is of full rank. Therefore, for a real-world dataset, instead of sampling a fixed number of movies a-priori, we randomly select $k=5$ movies and make all pairwise queries restricted to those $5$ movies. Then, we check if the $5 \times5$ gram matrix (with the $(i,j)$ entry being the unquantized similarity between the $i$th and $j$th movies) is of rank $5$ and if yes, then we will use that matrix for further calculations. If not, we sample again until we succeed. We then proceed to factorize the obtained $5\times5$ gram matrix into the form of $\mathbf{B}\mathbf{B}^{T}$, where $\mathbf{B}$ is binary. Finally, we query every movie with the $5$ movies already sampled (and clustered). This provides us with five linearly independent equations in five variables (each corresponding to whether the movie belongs to a particular cluster). Solving the equations for each movie, we finally obtain the categories each movie belongs to. Hence, the number of queries is at most
\begin{align*}
\textup{Number of trials to obtain a rank } 5 \textup{ matrix}\times25+3270\times5.
\end{align*}
Since the algorithm is randomized, we simulated this 50 times and we found that the Mean query complexity is \texttt{232126} (with a standard deviation of \texttt{269315.36}) which is only $4.34\%$ of the total number of possible queries.

\subsection{Synthetic Data}\label{app:moreExperiments}
\begin{figure*}[tb]
  
  \begin{subfigure}[t]{0.33\textwidth}
     \includegraphics[height=1.4in]{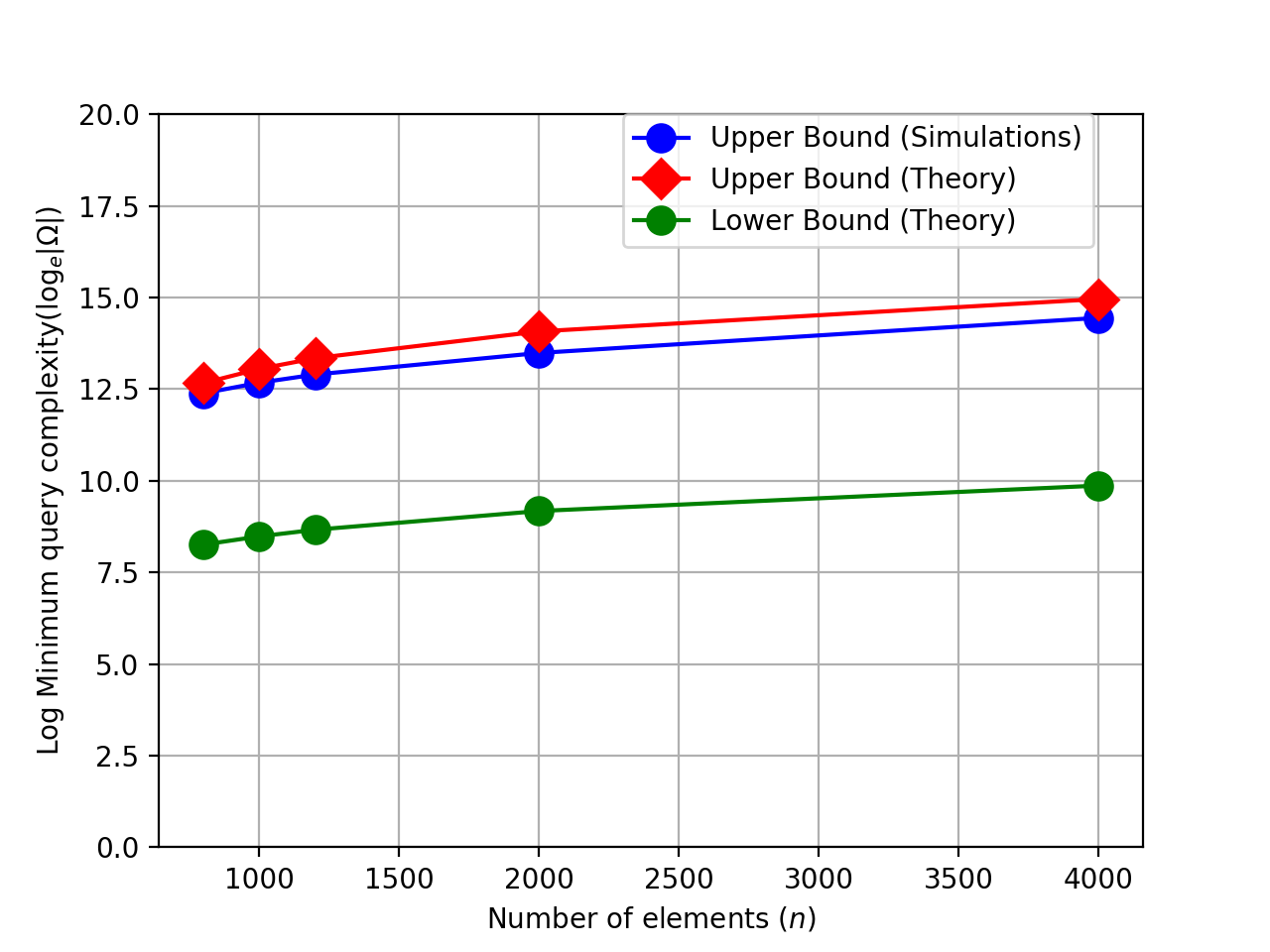}
     \caption{\small Query complexity $\log\abs{\Omega}$ of Algorithm~\ref{algo:oracle} (blue) as a function of $n$, for $\Delta=2$, $k=8$, and $q=0$. The green and red curves represent the lower and upper bound in Thms.~\ref{th:quantizedUniform} and \ref{thm:coverse_uniform}.}
          ~\label{fig:varyn}
  \end{subfigure}\hfill
\begin{subfigure}[t]{0.3 \textwidth}
   \includegraphics[height=1.4in]{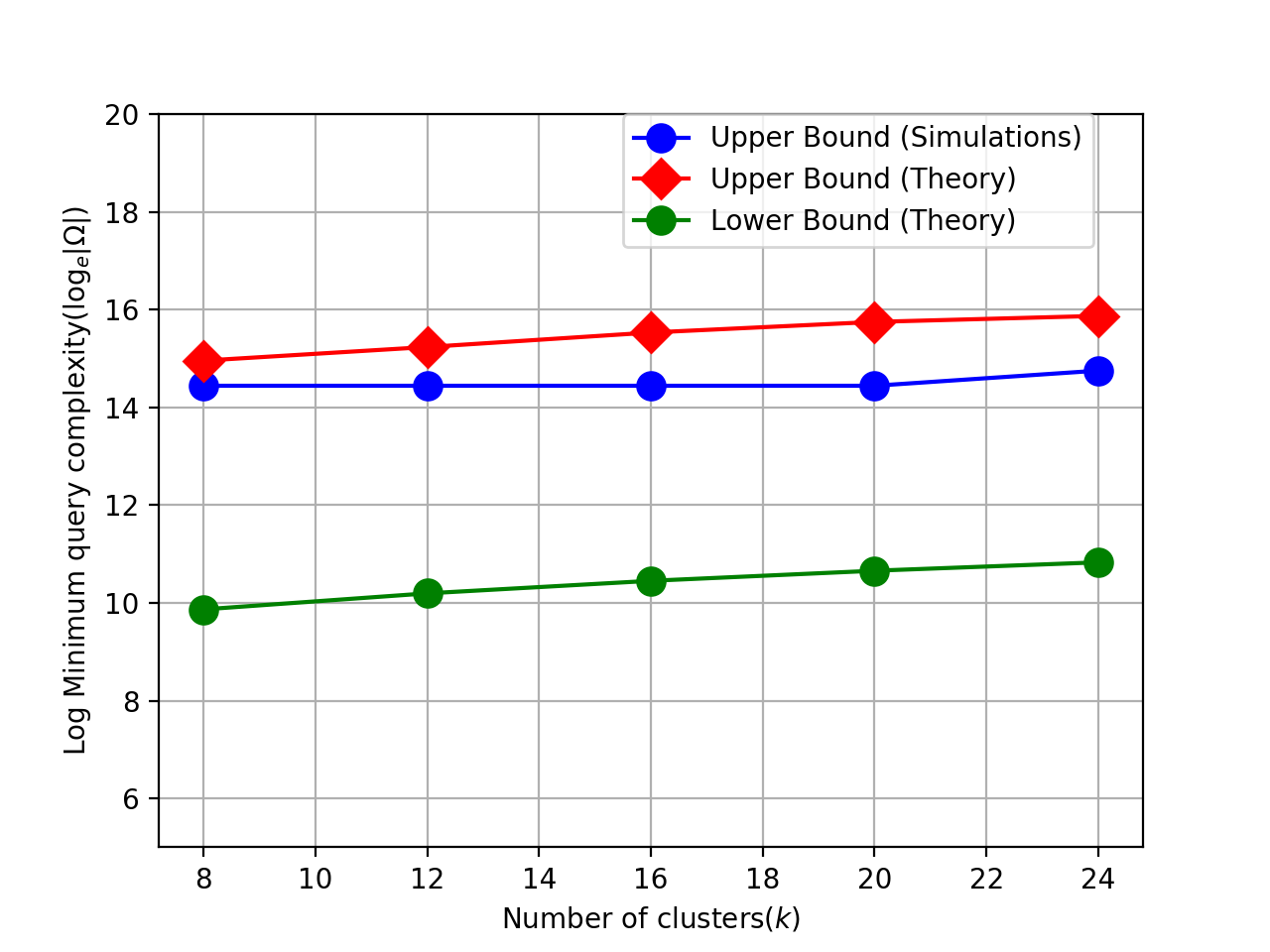}
   \caption{\small Query complexity $\log\abs{\Omega}$ of Algorithm~\ref{algo:oracle} (blue) as a function of $k$, for $\Delta=2$, $n=4000$, and $q=0$. The green and red curves represent the lower and upper bound in Thms.~\ref{th:quantizedUniform} and \ref{thm:coverse_uniform}.}
       ~\label{fig:varyk}
 \end{subfigure}%
\hfill
 \begin{subfigure}[t]{0.33\textwidth}
    \includegraphics[height=1.4in]{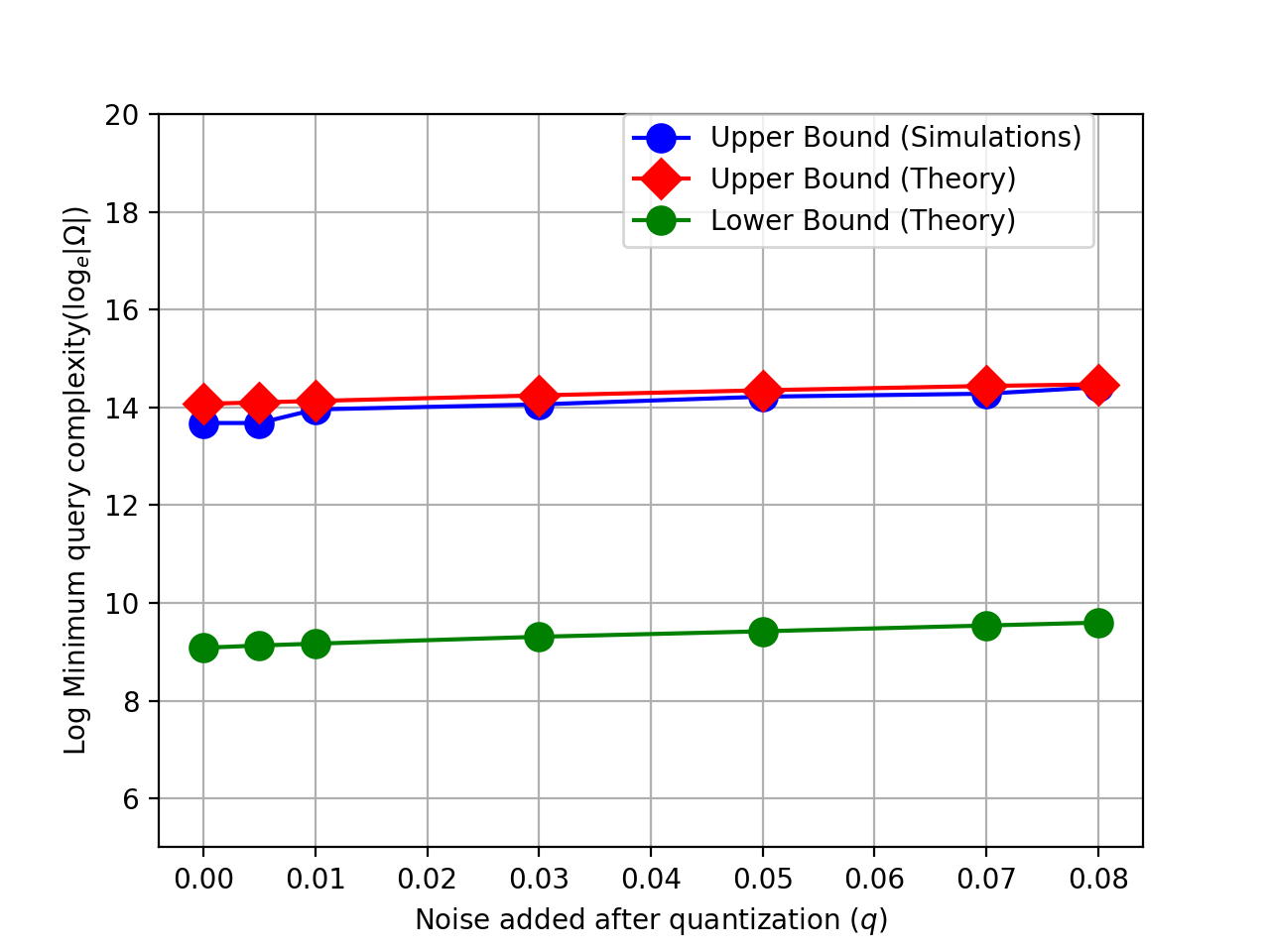}
    \caption{Query complexity $\log\abs{\Omega}$ of Algorithm~\ref{algo:oracle} (blue) as a function of $q$, for $\Delta=2$, $k=7$, and $n=2000$. The green and red curves represent the lower and upper bound in Thms.~\ref{th:quantizedUniform} and \ref{thm:coverse_uniform}.}
          ~\label{fig:varyq}
  \end{subfigure}
  
  \centering
  \begin{subfigure}[t]{0.33\textwidth}
    \includegraphics[height=1.4in]{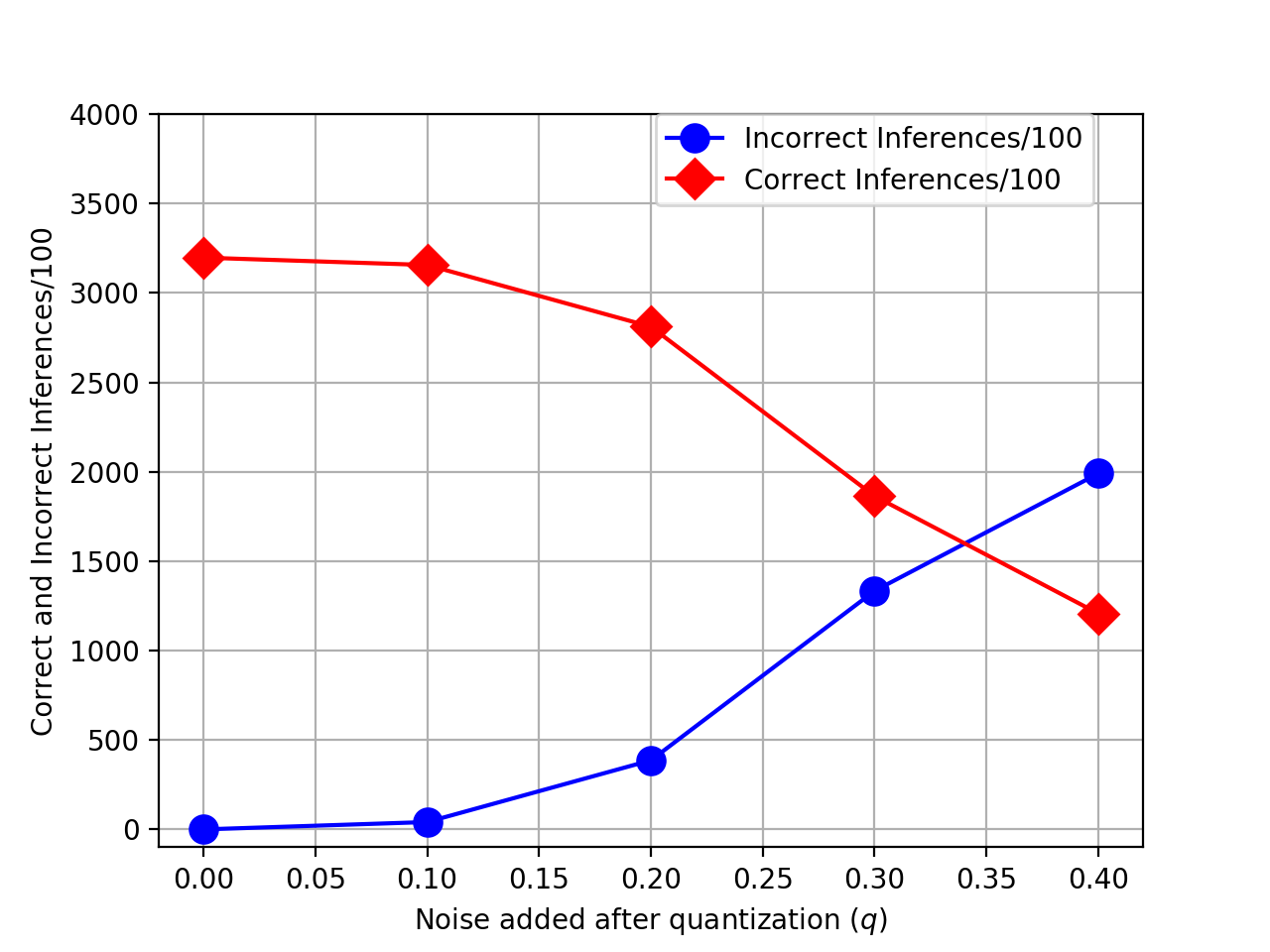}
    \caption{Number of correct/wrong inferences as a function of $q$, for $n=1000$, $k=8$, $\Delta=2$, and $|\calS|=400$.}
          ~\label{fig:corrinfvaryq}
  \end{subfigure}
  \hspace{0.5cm}
 \begin{subfigure}[t]{0.33\textwidth}
     \includegraphics[height=1.4in]{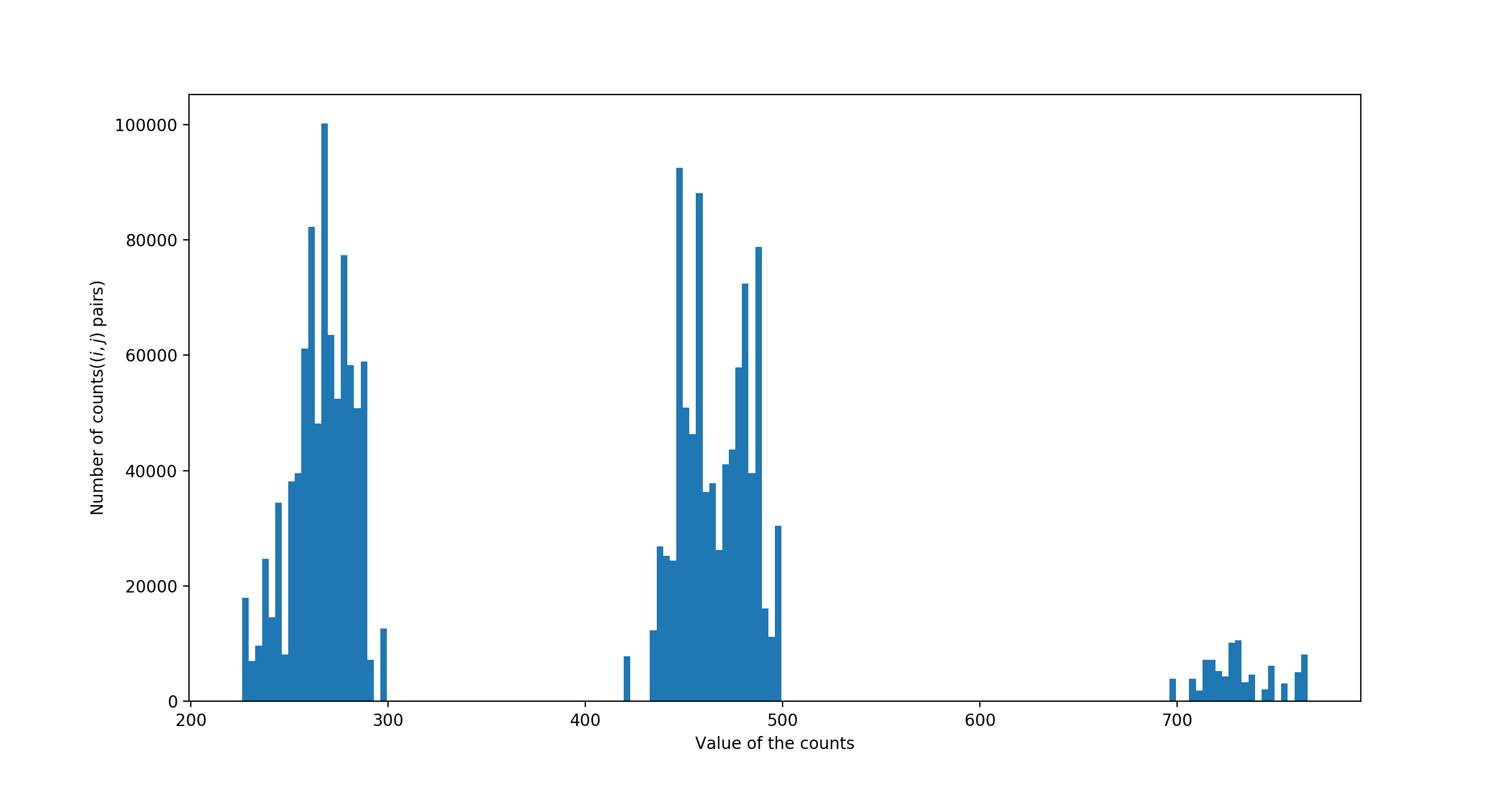}
     \caption{\small Histogram of the counts statistic $\{T_{ij}\}$ used for inferring in Algorithm~\ref{algo:noisy}, for $n=2000$, $k=2$, and $\Delta=2$.}
          ~\label{fig:histcounts}
  \end{subfigure}
\hfill

 \caption{\small Results of our techniques on simulated datasets.}
\end{figure*}
We also conduct in-depth simulations of the proposed techniques over synthetic data. We focus on the uniform ensemble and the quantized noisy oracle $\calO_{\mathsf{quantized}}$. Recall that in our proposed algorithms (see, e.g., Algorithm~\ref{algo:oracle}), we make $|\Omega|=\binom{|\calS|}{2}+|\calS|(n-|\calS|)$ queries, and for every query $(i,j)\in\Omega$, we infer using the count $T_{ij}$ the unquantized value $\bA_i^T\bA_j$. Accordingly, for evaluation, we investigate the amount of incorrect inferences made by our algorithm. It is only possible to recover the original matrix $\bA$ only if all the inferences are correct (using Algorithm~\ref{algo:factorization}). 
Fig.~\ref{fig:varyn} presents the $\log$-query complexity ($\log_e\abs{\Omega}$) as a function of the number of items $n$, for $\Delta=2$, $k=8$, and $q=0$. We compare the simulated performance of Algorithm~\ref{algo:oracle} with the theoretical lower and and upper bounds in Theorems~\ref{th:quantizedUniform} and \ref{thm:coverse_uniform}, respectively. It can be seen that our theoretical upper bound follows closely the numerically evaluated performance of Algorithm~\ref{algo:oracle}.
Fig.~\ref{fig:varyk} shows $\log\abs{\Omega}$ as a function of the number of clusters $k$, for $\Delta=2$, $n=4000$, and $q=0$, and the same conclusions as above remain true. Then, in Figs.~\ref{fig:varyq}--\ref{fig:histcounts} we consider the noisy scenario with $q$ controlling the ``amount" of noise. We first assume that the value of $q$ is known. Specifically, in Fig.~\ref{fig:varyq} we present $\log\abs{\Omega}$ as a function of the noise parameter $q$, for $n=2000$, $k=7$, and $\Delta=2$. Again, it can be seen that our theoretical upper bound match the simulated performance of Algorithm~\ref{algo:oracle}. We notice that the effect of the noise on the query complexity is not drastic, which imply that the proposed algorithm is robust. To illustrate the underlying mechanism of Algorithm~\ref{algo:oracle}, in Fig.~\ref{fig:corrinfvaryq} we present the amount of correct and wrong inferences occurred at the end of the second step of Algorithm~\ref{algo:oracle}, for $n=1000$, $k=8$, $\Delta=2$. In this figure, we took $|\calS|=400$, which is the sufficient size for recovery in the noiseless case (but not for the noisy regime). It can be seen that the number of wrong inferences grows moderately up to $q\approx 0.1$, and then the effect of choosing an insufficient $|\calS|$ becomes more severe. This suggests the potential application of our algorithms also when partial, rather than exact, recovery is the performance criterion. Finally, we illustrate how Algorithm~\ref{algo:noisy} works in the absence of noise. Specifically, in Fig.~\ref{fig:histcounts}, we provide a histogram of the counts $T_{ij}$ defined in Algorithm \ref{algo:noisy}, for $\Delta=2,k=7,n=2000$ and $q=0$. It is evident that the data can be separated into three groups (recall the third step of Algorithm~\ref{algo:noisy}) and therefore it is possible to infer correctly $\bA_i^T\bA_j$ for all pairwise queries. 

\subsection*{Acknowledgement}
This research is supported in part by NSF Grants CCF 1642658, 1642550, 1618512, and 1909046.
\bibliographystyle{abbrv}
\bibliography{main}

\begin{appendices}

\section{Uniqueness of Factorization}\label{app:1}

\subsection{Proof of Lemma~\ref{lem:uniq1}}
First, note that if $n>c{k \choose \Delta}\log {k \choose \Delta}$ for some constant $c>1$, then all vectors in $T_k(\Delta)$ will be present in $\bA$, with high probability, by a simple coupon collector argument.

\textit{Idea of the solution for $\Delta>2$:}  We can think of the mapping from $\bA_i$ (ith row of $\bA$) to $\bB_i$ (ith row of $\bB$) as a permutation $\sigma_i$ of the columns i.e. 
\begin{align*}
\bB_i=\sigma_i(\bA_i).
\end{align*}
Notice that there can be multiple permutations $\sigma_i$ that can explain the mapping $\bA_i \rightarrow \bB_i$. A permutation $\sigma$ can be described a series of swaps $\{(i_k,j_k)\}$ which implies that at the $kth$ step, the element in the $i_k^{th}$ position is swapped with the element in the $i_k^{th}$ position. The composition of permutation $\sigma_1$ with $\sigma_2$ implies implementing the swaps corresponding to $\sigma_2$ after the swaps corresponding to $\sigma_1$.  Now, if there exists a permutation $\sigma$ which is the same for all rows $i$, then definitely $\bB$ can be constructed by a permutation of the columns of $\bA$. In fact if $\sigma_i$ is the same for at least $k$ linearly independent rows of $\bA$, then $\sigma_i$ is same for all rows of $\bA$ because it uniquely defines the rotation matrix $\bR$. Hence if one can construct a set of at least $k$ $\Delta$-sparse vectors which are linearly independent and their gram matrix will have a unique factorization, then we are done (recall that those vectors will be in $\bA$ since $\bA$ contains all vectors from $T_k(\Delta)$.

\textit{Construction:} Consider the following matrices $\bC_1 \equiv [\bD \; \bI]$ and $\bC_2 \equiv [\bI \; \bD]$ of dimension $k-\Delta+1 \times k$. Here $\bD$ is a matrix of dimension $k-\Delta+1 \times \Delta-1$ and all the entries of $\bD$ is \texttt{1} and   $\bI$ is the identity matrix. The matrix that we will use for our construction is the following:
\begin{align*}
\bQ \equiv
\begin{bmatrix}
\bC_1 \\
\bC_2
\end{bmatrix}.
\end{align*}
The number of rows in $\bQ$ is at least $k$ since we know that $k \ge 2\Delta-2$. Notice that the only possible solution for the gram matrix of $\bC_1$ (and $\bC_2$) is of the form $[\bP_{11} \; \bD_{11} \; \bD_{12} \; \dots \; \bP_{1r}]$ where $[\bP_{11} \; \bP_{12} \; \bP_{13} \dots \; \bP_{1r} ]$ forms a permutation matrix and $[\bD_{11} \; \bD_{12} \; \dots \; \bD_{1r'}]=\bD$ (This is true only for $\Delta>2$. In order to see why, consider the rows of $\bC_1$ and think of a game between these  rows where the first row makes a swap in its columns and then all the rows tries to make swaps so that the inner products remain preserved. Suppose the first row makes a swap in the columns present in the span of $\bI$ in which case the other rows have to make the same swap. Now suppose the first row makes a swap in the columns present in the span of $\bD$ which is equivalent to no swap at all. Lastly, suppose the first row makes a swap in the two columns in which one belongs to $\bI$ and the other belongs to $\bD$. Again it can be checked that the other rows has to make the same swap to preserve the inner product. This implies that there exists a set of permutations $\Sigma_1$ ($\Sigma_2$) that explains the mapping of all rows in $\bC_1$ ($\bC_2$). If we can only show that there must exist a  permutation $\sigma$ that belongs to both $\Sigma_1$ and $\Sigma_2$, then we are done.  Suppose the new solution is 
\begin{align*}
\hat{\bQ} \equiv 
\begin{bmatrix}
\hat{\bQ}_1 \\
\hat{\bQ}_2 
\end{bmatrix} \equiv
\begin{bmatrix}
\bP_{11} \; \bD_{11} \; \bD_{12} \; \dots \; \bP_{1r} \\
\bP_{21} \; \bD_{21} \; \bD_{22} \; \dots \; \bP_{2s} \\
\end{bmatrix}.
\end{align*} 
The set of columns containing $\bD_{11},\bD_{12},\dots,\bD_{1r'}$ and the columns containing $\bD_{21},\bD_{22},\dots,\bD_{2s'}$ are disjoint otherwise there must exist two rows in $\bQ$ whose inner product violates its original value. 1) For $k\ge 2\Delta$, there exists two rows in $\bC_1$ and $\bC_2$ whose inner product is $0$ which cannot be the case if the columns are not disjoint. 2) For $k=2\Delta-1$, there exists a row in $\bC_2$ whose inner product is $1$ with all the rows in $\bC_1$. Again this cannot be the case if the columns are not disjoint. 3) For $k=2\Delta-2$, the inner product between any row in $\bC_1$ and $\bC_2$ is $2$ which is not possible if the columns are not disjoint). Therefore if we fix $\hat{\bQ}_1$, then the inner product of the any row of $\hat{\bQ}_2$ with all the rows in $\hat{\bQ}_1$ exactly specifies the position of the \texttt{1}'s in all the columns except the columns spanned by $\bD_{11},\bD_{12},\dots,\bD_{1r'}$. Hence a permutation in $\Sigma_1$ exactly specifies all the swaps in $\Sigma_2$ except those swaps restricted to the columns spanned by $\bD_{11},\bD_{12},\dots,\bD_{1r'}$ (again can be verified by a simple case study of the three cases: 1) $k=2\Delta-2$, 2) $k=2\Delta-1$, and 3) $k>2\Delta-1$. However, if $\Sigma_1$ contains a particular permutation, then a composition of that permutation with other permutations that only contains swaps restricted to the columns in $\bD_{11},\bD_{12},\dots,\bD_{1r'}$ does not change the mapping from $\bC_1 \rightarrow \hat{\bQ}_1$ and therefore $\Sigma_1$ contains these permutations as well. Hence there exist a single permutation $\sigma$ in $\Sigma_1$ and $\Sigma_2$ that explains the mapping from $\bQ \rightarrow \hat{\bQ}$. Hence the proof is complete. 

\subsection{Proof of Lemma~\ref{lem:uniq2}}
Let $\mathbf{e}_i\in\{0,1\}^{k}$ be the $k$-dimensional binary unit vector, namely, $\mathbf{e}_i$ is all zero except at its $i$th position. Suppose that all the $\{\mathbf{e}_i\}_{i=1}^{k}$ vectors are present in the matrix $\bA$, and let us denote the sub-matrix formed by these unit vectors by $\mathbf{Q}$. Thus, $\mathbf{Q}$ is a $k \times k$ matrix with rank $k$. It is easy to see that for any $k\times k$ binary matrix $\mathbf{R}$ such that $\mathbf{R}\mathbf{R}^{T}=\mathbf{Q}\mathbf{Q}^{T}$, $\mathbf{R}$ can be constructed by a permutation of the columns of $\mathbf{Q}$. Hence if the event ``$\mathbf{Q}$ is a sub-matrix of $\bA$" is true, then for any matrix $\bB$ such that $\bB\bB^{T}=\bA\bA^{T}$, $\bB$ can be constructed by a permutation of the columns of $\bA$. Let $\mathsf{E}_i$ denote the event that the vector $\mathbf{e}_i$ is not present in $\bA$. Therefore,
\begin{align}
\prob\p{\mathbf{Q}\text{ is sub-matrix of }\mathbf{A}}&=1-\prob\p{\bigcup_{i=1}^{k}\mathsf{E}_i}\nonumber \\
&\ge 1-k\cdot\prob(\mathsf{E}_1)\nonumber \\
&=1-k\cdot\pp{1-p\cdot(1-p)^{k-1}}^{n} \nonumber\\
&\ge 1-k\cdot e^{-c\log n-\log k}\nonumber\\
&= 1-\frac{1}{n^{c}}
\end{align} 
where the last inequality follows by substituting the condition on $n$ in the theorem statement.
\section{The Rank of Random Matrices}\label{app:2}

In this section, we state two important lemmas concerning the rank of the clustering matrix under both ensembles. For a set $\calS\subseteq[n]$, with $m=\abs{\calS}$, we let $\mathbf{A}_{\mathcal{S}}$ be the $m\times k$ projection matrix formed by the rows of $\mathbf{A}$ that correspond to the indices in $\mathcal{S}$.
\begin{lem}{[Rank of Uniformly Drawn Matrix]}\label{lem:rank_uniform}
Let $\bA$ be a random matrix drawn uniformly from $T_{k}(\Delta)$. Also, let $\mathcal{S}$ be a set of size $m>k$, drawn uniformly at random from $[n]$. Then, if $m\geq \frac{\binom{k}{\Delta}}{\binom{k-\Delta}{\Delta-1}}[1+c_1\log k+c_2\log n]$, for some $c_1>1$ and $c_2>0$, then
\begin{align}
\prob\left[\mathsf{rank}_{\mathbb{R}}(\mathbf{A}_{\mathcal{S}})=k\right]\geq 1-\frac{1}{n^{c_2}k^{c_1}}.\label{highprob}
\end{align}
\end{lem}
\begin{lem}{[Rank of i.i.d. Matrix]}\label{lem:rank_iid}
Let $\bA$ be an i.i.d. matrix with $\mathsf{Bernoulli}(p)$ entries. Also, let $\mathcal{S}$ be a set of size $m>k$, drawn uniformly at random from $[n]$, and define $\alpha\triangleq\max(p,1-p)$. Then, 
\begin{align}
\prob\left[\mathsf{rank}_{\mathbb{R}}(\mathbf{A}_{\mathcal{S}})=k\right]\geq 1-\min\p{1,k\cdot\alpha^{m-k+1}}.
\end{align}
\end{lem}
The above results imply that by taking $m$ large enough (as a function of $k$, $\Delta$, and $n$) we can guarantee that a sub-matrix formed by a random subset of rows taken from $\bA$ is of full rank with high probability. Specifically, for the i.i.d. ensemble, if 
\begin{align}
|\calS|>k-1-\frac{\log k + c\log n}{\log\max(p,1-p)}\triangleq \mathsf{S_{i.i.d.}},\label{iidthresholdrank}
\end{align}
for some $c>0$, then, $\prob\left[\mathsf{rank}_{\mathbb{R}}(\mathbf{A}_{\mathcal{S}})=k\right]\geq 1-n^{-c}$. Similarly, for the uniform ensemble, if 
\begin{align}
|\calS|>\frac{\binom{k}{\Delta}}{\binom{k-\Delta}{\Delta-1}}[1+c_1\log k+c_2\log n]\triangleq \mathsf{S_{uniform}},\label{uniformthresholdrank}
\end{align}
for some $c_1>0$ and $c_2>0$, then, \eqref{highprob} holds.

\subsection{Proof of Lemma~\ref{lem:rank_uniform}}
\label{sec:uniform_rank}
Given $k,\Delta\in\mathbb{N}$, define the set 
\begin{align}
T_k(\Delta)\triangleq\ppp{\mathbf{c}\in\ppp{0,1}^k:\;w_H(\mathbf{c})=\Delta},
\end{align}
namely, the set of all $k$-length binary sequence with Hamming weight $\Delta$. Let $\mathbf{A}$ be an $n\times k$ matrix formed by drawing independently $n$ sequences from $T_k(\Delta)$ and putting those as rows of $\mathbf{A}$. Let $\mathcal{S}$ be a set of size $m>k$ drawn uniformly at random from $[1:n]$. Let $\mathbf{A}_{\mathcal{S}}$ be an $m\times k$ matrix with the rows in $\mathbf{A}$ that correspond to the indices in $\mathcal{S}$. We would like to understand how large $m$ should be such that
$$\prob\left[\mathsf{rank}_{\mathbb{R}}(\mathbf{A}_{\mathcal{S}})<k\right]$$
decays to zero as $\mathsf{poly}(n^{-1})$. By symmetry, it is clear that
\begin{align}
\prob\left[\mathsf{rank}_{\mathbb{R}}(\mathbf{A}_{\mathcal{S}})<k\right]&=\prob\left[\mathsf{rank}_{\mathbb{R}}(\mathbf{B})<k\right]\\
&\leq\prob\left[\mathsf{rank}_{\mathbb{F}_2}(\mathbf{B})<k\right]
\end{align}
where $\mathbf{B}$ refers to a submatrix of $\mathbf{A}$ formed by taking, for example, the first $m$ rows, and the last inequality follows from the fact that for any filed $\mathbb{F}$, and any binary matrix $\mathbf{M}$ it holds $\mathsf{rank}_{\mathbb{F}}(\mathbf{M})\leq\mathsf{rank}_{\mathbb{R}}(\mathbf{M})$. 

We next analyze the probability term on the r.h.s. of the above inequality. To this end, we note that the event $\mathsf{rank}_{\mathbb{F}_2}(\mathbf{B})<k$ is in fact equivalent to the existence of a set $\calR$ with $|\calR|\leq k$ of column indices such that each row of $\mathbf{B}$ has an even number of $1$'s in $\calR$. Indeed, if this is the case, then some columns will be linearly dependent and thus the rank must be smaller than $k$. Accordingly, given a set of column indices $\calR$, let $\calE_{\calR}$ denote the event that each row of $\mathbf{B}$ has an even number of $1$'s in $\calR$. Then, using the above observation and the union bound,
\begin{align}
\prob\left[\mathsf{rank}_{\mathbb{F}_2}(\mathbf{B})<k\right]
&= \prob\left[\bigcup_{|\calR|\leq k}\calE_{\calR}\right]\\
&\leq \sum_{|\calR|=1}^k\binom{k}{|\calR|}\prob(\calE_{\calR}).
\end{align}
It is left to understand the behavior of $\prob(\calE_{\calR})$. The number of rows that have an odd number of non-zero elements in $\calR$ is simply
\begin{align}
N_{|\calR|,\Delta,k}=\sum_{\ell:\;\mathsf{odd}}^\Delta\binom{|\calR|}{\ell}\binom{k-|\calR|}{\Delta-\ell}
\end{align}
following a simple counting argument. Accordingly, 
\begin{align}
\Pr(\calE_{\calR}) = \p{1-\alpha}^{N_{|\calR|,\delta,k}}
\end{align}
where $\alpha\triangleq m\cdot[{{k}\choose{\Delta}}]^{-1}$. To get a simple upper bound on the probability of interest, we next lower bound $N_{|\calR|,\Delta,k}$. It is clear that
\begin{align}
N_{|\calR|,\Delta,k}&\geq {{|\calR|}\choose{1}}{{k-|\calR|}\choose{\Delta-1}}\\
&\geq |\calR|\cdot\binom{k-\Delta}{\Delta-1}.
\end{align} 
Then,
\begin{align}
\prob\left[\mathsf{rank}_{\mathbb{F}_2}(\mathbf{B})<k\right] &\leq \sum_{\ell=1}^k{{k}\choose{\ell}}(1-\alpha)^{\ell\cdot\binom{k-\Delta}{\Delta-1}}\nonumber\\
&\leq \sum_{\ell=1}^k\p{\frac{ek}{\ell}}^\ell e^{\ell\log(1-\alpha)\cdot \binom{k-\Delta}{\Delta-1}}\nonumber\\
&\leq \sum_{\ell=1}^ke^{\ell\pp{\log(ek)-\alpha\cdot \binom{k-\Delta}{\Delta-1}}}.
\end{align}
Now, taking $m = \frac{{{k}\choose{\Delta}}}{\binom{k-\Delta}{\Delta-1}}[\log(ek)+c_1\log k+c_2\log n]$, for some $c_1>1$ and $c_2>0$, we get that,
\begin{align}
\prob\left[\mathsf{rank}_{\mathbb{F}_2}(\mathbf{B})<k\right] &\leq \frac{1}{n^{c_2}}\sum_{\ell=1}^k\frac{1}{k^{c_1\ell}} \leq \frac{1}{k^{c_1-1}n^{c_2}}.
\end{align}

\subsection{Proof of Lemma~\ref{lem:rank_iid}} \label{sec:iid_rank}

Let $\mathbf{A}$ be an $n\times k$ i.i.d. matrix with each element distributed as $\mathsf{Bernoulli}(p)$, for some $0<p<1$. Let $\mathcal{S}$ be a set of size $m>k$ drawn uniformly at random from $[1:n]$. Let $\mathbf{A}_{\mathcal{S}}$ be the $m\times k$ matrix formed by the rows of $\mathbf{A}$ that correspond to the indices in $\mathcal{S}$. We would like to understand how large $m$ should be such that
$$\prob\left[\mathsf{rank}_{\mathbb{R}}(\mathbf{A}_{\mathcal{S}})=k\right]$$
goes to one as $1-\mathsf{poly}(n^{-1},k^{-1})$. By symmetry, it is clear that
\begin{align}
\prob\left[\mathsf{rank}_{\mathbb{R}}(\mathbf{A}_{\mathcal{S}})=k\right]=\prob\left[\mathsf{rank}_{\mathbb{R}}(\mathbf{B})=k\right]
\end{align}
where $\mathbf{B}$ refers to any submatrix of $\mathbf{A}$. Without loss of generality, let us take it to be formed by the first $m$ rows of $\mathbf{A}$. Also, we note that
\begin{align}
\prob\left[\mathsf{rank}_{\mathbb{R}}(\mathbf{B})=k\right] &= 1-\prob\left[\mathsf{rank}_{\mathbb{R}}(\mathbf{B})<k\right]\\
&\geq1-\prob\left[\mathsf{rank}_{\mathbb{F}_2}(\mathbf{B})<k\right]\\
& = \prob\left[\mathsf{rank}_{\mathbb{F}_2}(\mathbf{B})=k\right]
\end{align}
where the inequality follows from the fact that for any filed $\mathbb{F}$, and any binary matrix $\mathbf{M}$ it holds $\mathsf{rank}_{\mathbb{F}}(\mathbf{M})\leq\mathsf{rank}_{\mathbb{R}}(\mathbf{M})$. Therefore, it is suffice to lower bound $\prob\left[\mathsf{rank}_{\mathbb{F}_2}(\mathbf{B})=k\right]$. 

Let $\calF_i$ designate the event that the first $i$ columns of $\mathbf{B}$, denote by $\mathbf{B}_1,\ldots,\mathbf{B}_i$, are linearly independent. Then, it is clear that
\begin{align}
\prob\pp{\calF_{i+1}} &= \prob\pp{\calF_{i+1}\vert\calF_{i}}\prob\pp{\calF_{i}}+\prob\pp{\calF_{i+1}\vert\calF^c_{i}}\prob\pp{\calF^c_{i}}\nonumber\\
& = \prob\pp{\calF_{i+1}\vert\calF_{i}}\prob\pp{\calF_{i}}
\end{align}
where the second equality is because conditioned on $\calF^c_{i}$, the event $\calF_{i+1}$ cannot occur. Inductively, we then may write
\begin{align}
\prob\left[\mathsf{rank}_{\mathbb{F}_2}(\mathbf{B})=k\right] = \prod_{i=0}^{k-1}\prob\pp{\calF_{i+1}\vert\calF_{i}}
\end{align}
with $\calF_{i}=\emptyset$. We next lower bound each term in the product. To this end, recall that the fact that $\mathbf{B}_1,\ldots,\mathbf{B}_i$ are linearly independent implies that the $m\times i$ submatrix formed by these columns can be transformed into a matrix with the first $i$ columns forming an identity matrix, namely, the $i\times i$ identity matrix appears as a sub-block. Accordingly, this implies that \emph{any} vector contained in the span of $\mathbf{B}_1,\ldots,\mathbf{B}_i$ can be represented as follows: its first $i$ entries can have arbitrary values, and the rest $m-i$ entries must be uniquely determined by the first $i$ entries. With this fact in mind, the $(i+1)$th column of $\mathbf{B}$ is linearly independent of the previous $i$ columns if and only if it is not spanned by these columns, or equivalently, if its last $m-i$ entries can be arbitrary. The probability of this being happen is clearly lower bounded by $1-\alpha^{m-i}$ with $\alpha\triangleq\max(p,1-p)$. Combining the last observations, we obtain
\begin{align}
\prob\left[\mathsf{rank}_{\mathbb{F}_2}(\mathbf{B})=k\right]&\geq \prod_{i=0}^{k-1}(1-\alpha^{m-i})\\
& = \prod_{\ell=m-k+1}^{m}(1-\alpha^{\ell}).
\end{align}
The above result is general, but note that 
\begin{align}
\prob\left[\mathsf{rank}_{\mathbb{F}_2}(\mathbf{B})=k\right]&\geq (1-\alpha^{m-k+1})^k\\
&\geq 1-\min(1,k\alpha^{m-k+1}),
\end{align}
which concludes the proof.

\section{Proof of Proposition~\ref{th:disjoint}}

In this section we analyze Algorithm~\ref{algo:oracle_dis}, which extracts the clustering matrix when the clusters are disjoint. Pick $m$ elements uniformly at random from the set of elements $\calN$. We denote this set by $\calR$. Perform all pairwise queries among these $m$ elements, resulting in a total of $\binom{m}{2}$ queries. We want to take $m$ large enough such that we have representatives of all clusters, namely, among these $m$ elements there will exists at least one element (representative) from each cluster. We next show that if $m\geq\frac{n}{n_{\min}}\log(kn^{\varepsilon})$ than with probability decaying to zero polynomially in $n$, this is possible. Let $\calE_\ell$ denote the event that no item in $\calR$ appears in the $\ell$'th cluster. Then, we note that
\begin{align}
\prob\pp{\bigcup_{\ell=1}^k\calE_\ell}&\leq k\cdot\max_{1\leq \ell\leq k}\prob\pp{\calE_\ell}\\
&\leq k\cdot\p{1-\frac{n_{\min}}{n}}^m\\
&\leq k\cdot e^{-m\frac{n_{\min}}{n}}\\
&\leq \frac{1}{n^{\varepsilon}}.
\end{align}
To wit, after the second stage of Algorithm~\ref{algo:oracle_dis} with high probability we found $k$ representatives $\calT$ for the clusters. Finally, for the remaining $n-m$ items, we perform at most $k$ queries to decide which cluster they are in. Thus, the total number of quires is $k(n-m)+\pp{\frac{n}{n_{\min}}}^2\log^2(kn^{\varepsilon})$.

\section{Proof of Theorem~\ref{th:overlappdirect}}

In this section we analyze the performance of Algorithm~\ref{algo:oracle_overlap}, for the i.i.d. ensemble. The uniform ensemble is handled in the same way. In the first step of Algorithm~\ref{algo:oracle_overlap}, we pick $\calS$ elements uniformly at random from $\calN$ such that $m=|\calS|>\mathsf{S_{i.i.d.}}$, where the latter is defined in \eqref{iidthresholdrank}. According to Lemma~\ref{lem:rank_iid}, this ensures that $\mathsf{rank}(\bA_{\calS})=k$ with high probability. We perform all pairwise queries among these $m$ elements, resulting in a total of $\binom{m}{2}$ queries. Then, in the second stage, we extract a valid membership of all chosen $m$ element by a simple rank factorization procedure. We denote by $\hat{\bA}_{\calS}$ the resultant rank factorized matrix, and we note that it might be not unique. Nonetheless, since $m>\mathsf{S_{i.i.d.}}$, we can find a subset of elements $\calT \subseteq \calS$ whose membership vectors form a basis of $\mathbb{R}^{k}$. Denote the $k\times k$ membership matrix corresponding to $\calT$ by $\tilde{\bA}_{\calT}$. Then, in the third step of Algorithm~\ref{algo:oracle_overlap}, we query each of the remaining elements in $[n]\setminus\calS$ with all the elements in $\calT$. Accordingly, for any $i\in[n]\setminus\calS$, let $\mathbf{c}_i$ be the $k$-length vector containing the $k$ queries of element $i$ with $\calT$. Subsequently, given $\{\mathbf{c}_i\}_i$, we find the membership vector $\mathbf{m}_i$ of the $i$th element by solving $\tilde{\bA}_{\calT}\mathbf{m}_i=\mathbf{c}_i$, which form $k$ linearly independent equations in the $k$ variables. Thus, we can solve this system of equations uniquely to obtain the membership vector of $i$th element. Note that despite the fact that the second step of Algorithm \ref{algo:oracle_overlap} is not unique (and then $\hat{\bA}_{\calS}$ might be different from the true $\bA_\calS$), our algorithm will correctly recover the similarity matrix. 

Indeed, let $\mathbf{B}_1$ and $\mathbf{B}_2$ be two solutions obtained by the rank factorization procedure, such that $\mathbf{B}_1\mathbf{B}_1^T=\mathbf{B}_2\mathbf{B}_2^T=\mathbf{A}_{\mathcal{S}}\mathbf{A}_{\mathcal{S}}^{T}$. Consider two elements, say, $\{1,2\}$ whose membership vectors $\mathbf{m}_1$ and $\mathbf{m}_2$ are unknown after the second step of Algorithm~\ref{algo:oracle_overlap}. Since we query these elements with all the elements in $\mathcal{S}$, we must have the following set of equations
\begin{align}
\begin{cases}
\mathbf{B}_1\mathbf{m}_1=\mathbf{c}_1, \\
\mathbf{B}_1\mathbf{m}_2=\mathbf{c}_2,
\end{cases}
\quad 
\begin{cases}
\mathbf{B}_2\mathbf{m}_1=\mathbf{c}_1, \\ \mathbf{B}_2\mathbf{m}_2=\mathbf{c}_2.
\end{cases}
\end{align}
Denote by $\hat{\mathbf{m}}_1$ and $\hat{\mathbf{m}}_2$ the solutions of $\mathbf{m}_1$ and $\mathbf{m}_2$, respectively, if $\mathbf{B}_1$ is the solution used. Similarly, let  $\bar{\mathbf{m}}_1$ and $\bar{\mathbf{m}}_2$ be the solutions of $\mathbf{m}_1$ and $\mathbf{m}_2$, respectively, if $\mathbf{B}_2$ is the solution used. Then,
\begin{align}
\hat{\mathbf{m}}_1^T\hat{\mathbf{m}}_2&=(\mathbf{B}_1^{-1}\mathbf{c}_1)^{T}\mathbf{B}_1^{-1}\mathbf{c}_2\nonumber\\
&=\mathbf{c}_1^{T}(\mathbf{B}_1^{-1})^{T} \mathbf{B}_1^{-1}\mathbf{c}_2\nonumber\\
&=\mathbf{c}_1^{T}(\mathbf{B}_1^{T})^{-1} \mathbf{B}_1^{-1}\mathbf{c}_2\nonumber\\
&=\mathbf{c}_1^{T}(\mathbf{B}_1\mathbf{B}_1^{T})^{-1}\mathbf{c}_2\nonumber\\
&=\mathbf{c}_1^{T}(\mathbf{B}_2\mathbf{B}_2^{T})^{-1}\mathbf{c}_2\nonumber\\
&=\bar{\mathbf{m}}_1^T\bar{\mathbf{m}}_2,
\end{align}
which means that the inner products will be preserved. Hence we will get the same similarity matrix irrespective of the intermediate solution produced by the rank factorization which may be incorrect.

Finally, note that the number of queries needed in the above algorithm is $\binom{|\calS|}{2}+k(n-|\calS|)$, which concludes the proof.

\section{Proof of Theorem~\ref{th:quantizedUniform}}\label{app:quantizedUniform}
In this section we analyze the Algorithm~\ref{algo:oracle} for quantized noisy oracle, under the uniform ensemble. Recall that we deal with the setting where the oracle responses are $\mathbf{Y}_{ij}=\calO_{\mathsf{quantized}}(i,j)=\calQ(\bA_i^T\bA_j)\oplus W_{i,j}$, and we assume that $\bA$ was generated according to the uniform ensemble with $k>3\Delta$. Let $\mathcal{S}$ be a set drawn uniformly at random from $\calN$, whose size will be determined in the sequel.

We next analyze the probability of error associated with Algorithm~\ref{algo:oracle}, by investigating each of its steps. Accordingly, in the first step of Algorithm~\ref{algo:oracle}, we observe $\mathbf{Y}_{ij}$ for all pairs $(i,j)\in\mathcal{S}$. Then, in the second step of Algorithm \ref{algo:oracle}, using these $\binom{|\calS|}{2}$ observations we infer $\langle \mathbf{A}_i,\mathbf{A}_j \rangle$, for any $(i,j)\in\mathcal{S}$. This is done using the procedure in Algorithm~\ref{algo:Noisyinfersupport1}. To wit, at the end of the second step of Algorithm~\ref{algo:oracle}, we should have an exact estimate of $\bA_{\calS}\bA_{\calS}^T$ with high probability. In the following, we show that this is indeed correct.

For a given pair $(i,j)\in\calS$, we define a sequence of $(\Delta+1)$ hypotheses $\ppp{\calH_{\ell}}_{\ell=0}^{\Delta}$, where
\begin{align}
\mathcal{H}_{\ell}: \quad \mathbf{A}_{i}^T\bA_j=\ell \quad \text{ for } \ell=0,\dots,\Delta.
\end{align} 
For a pair $(i,j)\in\calS$, define
\begin{align}
T_{i,j}\triangleq\sum_{\substack{r \in \mathcal{S} \\ r \neq i,j }} \mathds{1}[\mathbf{Y}_{ir}=1 \cap \mathbf{Y}_{jr}=1].
\end{align}
It is clear that each summand of $T_{i,j}$ is one if $\mathbf{Y}_{ir}=\mathbf{Y}_{jr}=1$, and zero otherwise.  We call the aforementioned event a \emph{triangle} formed by the triplet $(i,j,r)$. Accordingly, the random variable $T_{i,j}$ simply counts/enumerate the number of triangles formed by a given pair $(i,j)\in\calS$. As can be seen from Algorithm~\ref{algo:Noisyinfersupport1}, the count $T_{i,j}$ is main quantity used to infer the value of $\bA_i^T\bA_j$. Accordingly, we need to understand its probabilistic behaviour. For simplicity of notation, in the following we denote by $\prob_{\ell}(\cdot)$ and $\avg_{\ell}(\cdot)$ the probability and the expectation operators conditioned on hypothesis $\mathcal{H}_{\ell}$ being true. Also, let $\mathbf{Q}_{ij}\triangleq \calQ\p{\bA_i^T\bA_j}$. Then, for $k>3\Delta$, it is an easy task to check that for a triplet $(i,j,r)\in[n]$, we have 
\begin{align}
\prob_{\ell}(\mathbf{Q}_{ir}=1 \cap \mathbf{Q}_{jr}=1)&= 1- \prob_{\ell}(\mathbf{Q}_{jr}=0)-\prob_{\ell}(\mathbf{Q}_{ir}=0)+\prob_{\ell}(\mathbf{Q}_{ir}=0 \cap \mathbf{}_{jr}=0)\\
&=1-2\frac{{k-\Delta \choose \Delta}}{{k \choose \Delta}}+\frac{{k-2\Delta+\ell \choose \Delta}}{{k \choose \Delta}}.
\end{align}
In a similar fashion,
\begin{align}
\prob_{\ell}(\mathbf{Q}_{ir}=0 \cap \mathbf{Q}_{jr}=1)&=\prob_{\ell}(\mathbf{Q}_{ir}=1 \cap \mathbf{Q}_{jr}=0)\\
&=\prob_{\ell}(\mathbf{Q}_{ir}=0)-\prob_{\ell}(\mathbf{Q}_{ir}=0 \cap \mathbf{Q}_{jr}=0)\\
&=\frac{{k-\Delta \choose \Delta}}{{k \choose \Delta}}-\frac{{k-2\Delta+\ell \choose \Delta}}{{k \choose \Delta}},
\end{align} 
and,
\begin{align}
\prob_{\ell}(\mathbf{Q}_{ir}=0 \cap \mathbf{Q}_{jr}=0) =\frac{{k-2\Delta+\ell \choose \Delta}}{{k \choose \Delta}}.
\end{align} 
Therefore, using the above results, we obtain by the law of total probability,
\begin{align}
\prob_{\ell}(\mathbf{Y}_{ir}=1 \cap \mathbf{Y}_{jr}=1)&=(1-q)^{2}\cdot\prob_\ell(\mathbf{Q}_{ir}=1 \cap \mathbf{Q}_{jr}=1)+q(1-q)\cdot\prob_\ell(\mathbf{Z}_{ir}=1 \cap \mathbf{Z}_{jr}=1)\nonumber\\
&\quad+q(1-q)\cdot\prob_{\ell}(\mathbf{Q}_{ir}=1 \cap \mathbf{Q}_{jr}=1 )+q^{2}\cdot\prob_{\ell}(\mathbf{Q}_{ir}=0 \cap \mathbf{Q}_{jr}=0)\nonumber\\
&=(1-q)^{2}- 2(1-2q)(1-q)\frac{{k-\Delta \choose \Delta}}{{k \choose \Delta}}+(1-2q)^{2}\frac{{k-2\Delta+\ell \choose \Delta}}{{k \choose \Delta}}.
\end{align}
Accordingly, we obtain
\begin{align}
\avg_{\ell} T_{i,j}&=\avg_{\ell} \sum_{\substack{r \in \mathcal{S} \\ r \neq i,j }} \mathds{1}[\mathbf{Y}_{ir}=1 \cap \mathbf{Y}_{jr}=1]\nonumber\\
&=(|\mathcal{S}|-2)\left((1-q)^{2}-2(1-2q)(1-q)\frac{{k-\Delta \choose \Delta}}{{k \choose \Delta}}+(1-2q)^{2}\frac{{k-2\Delta+\ell \choose \Delta}}{{k \choose \Delta}} \right).\label{avgCalcfir}
\end{align}
Therefore for any two hypotheses $\mathcal{H}_{\ell}$ and $\mathcal{H}_{\ell'}$, we have 
\begin{align}
&|\avg_{\ell} T_{ij}-\avg_{\ell'} T_{ij}|=\frac{(|\mathcal{S}|-2)(1-2q)^2}{{k \choose \Delta}}\cdot\abs{{k-2\Delta+\ell \choose \Delta}-{k-2\Delta+\ell' \choose \Delta}}.\label{meanCalc}
\end{align}
Now, given $(\bA_i,\bA_j)$ it is clear that the random variables $\mathds{1}[\mathbf{Y}_{ir}=1 \cap \mathbf{Y}_{jr}=1]$, for $r\in \mathcal{S}, r \neq i,j$, are statistically independent and therefore we can apply standard concentration inequalities, such as, Chernoff's inequality, to show that the value of the random variable $T_{ij}$ is strongly concentrated around its mean. We state the following classical result (see, e.g., \cite{cover2012elements}).
\begin{lem}{[Chernoff's inequality]}\label{thm:Chernoff}
Let $(X_i)_{i=1}^n$ be a sequence of $n$ i.i.d. $\mathsf{Bernoulli}(p)$ random variables. Then, for any $\mu>p$,
\begin{align}
\pr\pp{\frac{1}{n}\sum_{i=1}^nX_i>\mu}\leq e^{-n\cdot d_{\mathsf{KL}}(\mu||p)}.
\end{align}
\end{lem}
Let $P^{(i,j)}_{\mathsf{error},1}$ designate the average probability of associated Algorithm~\ref{algo:Noisyinfersupport1}, for a given pair $(i,j)\in\calS$. Then, we have
\begin{align}
P^{(i,j)}_{\mathsf{error},1} &= \sum_{\ell=1}^{\Delta}\pr\p{\calH_{\ell}}\pr_\ell\pp{\mathsf{error}}\nonumber\\
& = \sum_{\ell=1}^{\Delta}\pr\p{\calH_{\ell}}\pr_{\ell}\pp{\min_{\ell'\neq\ell}\;\abs{T_{ij}-\avg_{\ell'}T_{ij}}<\abs{T_{ij}-\avg_{\ell}T_{ij}}}.\label{error_prob1}
\end{align}
Now, note that
\begin{align}
\pr_\ell\pp{\min_{\ell'\neq\ell}\;\abs{T_{ij}-\avg_{\ell'}T_{ij}}<\abs{T_{ij}-\avg_{\ell}T_{ij}}}\leq\pr_{\ell}\pp{\abs{T_{ij}-\avg_{\ell}T_{ij}}>\frac{\min_{\ell'\neq\ell}\abs{\avg_{\ell}T_{ij}-\avg_{\ell'}T_{ij}}}{2}}
\end{align}
where we have used the triangle inequality, i.e., $\abs{a-b}\geq\abs{\abs{a}-\abs{b}}$, for any $a,b\in\mathbb{R}$. Then, using Lemma~\ref{thm:Chernoff}, we obtain
\begin{align}
\pr_\ell\pp{\abs{T_{ij}-\avg_{\ell}T_{ij}}>\frac{\min_{\ell'\neq\ell}\abs{\avg_{\ell}T_{ij}-\avg_{\ell'}T_{ij}}}{2}}\leq2\cdot e^{-(|\calS|-2)d_{\mathsf{KL}}\p{\alpha||\beta}}\label{kldivfirst}
\end{align}
where 
$$\beta\triangleq (1-q)^{2}-2(1-2q)(1-q)\frac{{k-\Delta \choose \Delta}}{{k \choose \Delta}}+(1-2q)^{2}\frac{{k-2\Delta+\ell \choose \Delta}}{{k \choose \Delta}},
$$
and
\begin{align*}
\alpha\triangleq \beta+\frac{(1-2q)^{2}}{2{k \choose \Delta}}\min_{\ell'\neq\ell}\abs{{k-2\Delta+\ell \choose \Delta}-{k-2\Delta+\ell' \choose \Delta}}. 
\end{align*}
To simplify the above result, recall Pinsker's inequality, which states that $d_{\mathsf{KL}}(p||q)\geq 2\abs{p-q}^2$, for any $0\leq p,q\leq 1$. Therefore,
\begin{align}
d_{\mathsf{KL}}\p{\alpha||\beta}&\geq 2\abs{\alpha-\beta}^2\nonumber\\
& = \frac{(1-2q)^{4}}{2}\min_{\ell'\neq\ell}\frac{\abs{{k-2\Delta+\ell \choose \Delta}-{k-2\Delta+\ell' \choose \Delta}}^2}{{{k \choose \Delta}}^2}\nonumber\\
&= \frac{(1-2q)^{4}}{2}\frac{\abs{{k-2\Delta+\ell \choose \Delta}-{k-2\Delta+\ell-1 \choose \Delta}}^2}{{{k \choose \Delta}}^2}\triangleq \eta(\ell).
\end{align}
Combining the above results with the fact that $\eta(\ell)$ is monotonically decreasing in $\ell$, we obtain
\begin{align}
P^{(i,j)}_{\mathsf{error},1} &\leq2\cdot\max_{\ell\geq1}e^{-(|\calS|-2)\cdot\eta(\ell)}\\
&= 2\cdot e^{-(|\calS|-2)\cdot\eta(1)}.\label{eq:concunif}
\end{align}
Therefore, at the end of the second stage of Algorithm~\ref{algo:oracle}, we will have an exact estimate of $\bA_{\calS}\bA_{\calS}^T$ if \eqref{eq:concunif} is satisfied for all $(i,j)\in\calS$. By the union bound, we obtain that the overall probability of error associated with second stage of Algorithm~\ref{algo:oracle} is upper bounded by
\begin{align}
P_{\mathsf{error},1}&\leq 2\cdot\binom{|\calS|}{2}\cdot e^{-(|\calS|-2)\cdot\eta(1)}\nonumber\\
&\leq 2n^2\cdot e^{-(|\calS|-2)\cdot\eta(1)}.\label{47eq}
\end{align}
Accordingly, taking $|\mathcal{S}|>\frac{1}{\eta(1)}\log(2 n^{2+\varepsilon})+2$, for any $\varepsilon>0$, is sufficient to bring the probability of error to at most $n^{-\varepsilon}$. Note that the above constraint on $|\calS|$ expands to
\begin{align}
|\mathcal{S}|>\frac{2\binom{k}{\Delta}^2\log(2 n^{2+\varepsilon})}{(1-2q)^4\pp{\binom{k-2\Delta+1}{\Delta}-\binom{k-2\Delta}{\Delta}}^2}+2.\label{calSsizeCond}
\end{align}

Next, given the exact estimate of $\bA_{\calS}\bA_{\calS}^T$ from the second step our algorithm, in the third step we extract the membership of all chosen $|\calS|$ elements by a simple rank factorization procedure as in Algorithm~\ref{algo:oracle_overlap}. We denote the resultant rank factorized matrix by $\hat{\bA}_{\calS}$. Finally, we analyze the fourth step of Algorithm~\ref{algo:oracle}, in which for each index $j\not\in\mathcal{S}$, we observe $\mathbf{Y}_{ij}$, for all $i \in \mathcal{S}$, and from these we would like to infer the leftover inner-products. This is done with the help of Algorithm~\ref{algo:Noisyinfersupport2} which we analyze in the sequel.

In fact the entire analysis of Algorithm~\ref{algo:Noisyinfersupport2} remains almost the same as that for Algorithm~\ref{algo:Noisyinfersupport1} (and therefore it is omitted), except now $T_{ij}$ is a sum of $|\mathcal{S}|-1$ indicator random variables. Indeed, it can be shown that the average probability of error $P^{(i,j)}_{\mathsf{error},2}$ associated with Algorithm~\ref{algo:Noisyinfersupport2} is upper-bounded as follows
\begin{align}
P^{(i,j)}_{\mathsf{error},2} &\leq2\cdot e^{-(|\calS|-1)\cdot\eta(1)},\label{eq:concunif2}
\end{align}
and accordingly, the overall probability of error associated with fourth stage of Algorithm~\ref{algo:oracle} is upper bounded by
\begin{align}
P_{\mathsf{error},2}&\leq 2\cdot|\calS|\cdot(n-|\calS|)\cdot e^{-(|\calS|-1)\cdot\eta(1)}\nonumber\\
&\leq 2n^2\cdot e^{-(|\calS|-1)\cdot\eta(1)}.
\end{align}
Accordingly, taking $|\mathcal{S}|>\frac{1}{\eta(1)}\log(2 n^{2+\varepsilon})+1$, for any $\varepsilon>0$, is sufficient to bring the probability of error to at most $n^{-\varepsilon}$. Thus, we may conclude that with high probability we have the exact values of $\bA_i^T\bA_j$, for all $i\not\in\calS$, and $j\in\calS$. For each $i\not\in\calS$ we denote by $\mathbf{c}_i$ the $|\calS|$ length vector containing the inner-products $\bA_i^T\bA_j$, for $j\in\calS$.

Finally, the only thing that is left to do is solve for the membership vector of each of the $(n-|\calS|)$ elements. This is done similarly as was done in Algorithm~\ref{algo:oracle_overlap} (see the m~\ref{th:overlappdirect}). Specifically, from Lemma~\ref{lem:rank_uniform}, we know that by taking $|\calS|>\frac{{{k}\choose{\Delta}}}{f(\Delta,k)}[\log(ek)+c\log k]+c_2\log n$, for some $c_1>1$ and $c_2>0$, the rows of $\mathbf{A}_{\mathcal{S}}$ form a basis of $\mathbb{F}_{2}^{k}$ with high probability. Note that \eqref{calSsizeCond} is a stringent condition, and thus Lemma~\ref{lem:rank_uniform} holds. Furthermore, were also able to observe that if the rows of $\mathbf{A}_{\mathcal{S}}$ formed a basis, then for an index $i\in\not\in\mathcal{S}$, the set of values of $\{\mathbf{A}_{i}^{T}\mathbf{A}_{j}\}$, for all $j \in \mathcal{S}$ were enough to determine the vector $\mathbf{A}_{i}$. Indeed, given the resultant matrix $\hat{\bA}_{\calS}$ from the rank factorization step in the third step of Algorithm~\ref{algo:oracle}, the unknown membership vector $\mathbf{c}_j$ of the $j\not\in\calS$ element is found by solving $\hat{\bA}_{\calS}\mathbf{c}_j = \mathbf{s}_j$.

Finally, we conclude the proof by noting that the total number of observed entries is
\begin{align}
|\Omega|={|\mathcal{S}| \choose 2}+|\mathcal{S}|(n-|\mathcal{S}|),\label{sampleComplexity}
\end{align}
and that \eqref{calSsizeCond} is the stringent condition which ensures vanishing error probability.

\section{Proof of Theorem~\ref{th:quantizedBern}}
In this section we analyze the Algorithm~\ref{algo:oracle2} for quantized noisy oracle, under the i.i.d. ensemble.

\begin{algorithm}[tb]
\caption{\texttt{Quantized Responses} The algorithm for extracting membership of elements via queries to oracle. \label{algo:oracle2}}
\begin{algorithmic}[1]
\REQUIRE Number of elements: $N$, number of clusters $k$, oracle responses $\calO_{\mathsf{quantized}}(i,j)$ for query $(i,j)\in\Omega$, where $i,j \in [N]$.
\STATE Choose a set $\mathcal{S}$ of elements drawn uniformly at random from $\calN$, and perform all pairwise queries corresponding to these $|\calS|$ elements.
\STATE Run Algorithm \texttt{InferSupportsize} to infer $\norm{\bA_i}_0$, for $i\in\mathcal{S}$. Then, run Algorithm \texttt{InferIntersection1} to infer $\langle \mathbf{A}_i,\mathbf{A}_j \rangle$ for each pair of entries $i,j \in \mathcal{S}$.
\STATE Extract the membership of all the $|\calS|$ elements up-to a permutation of clusters.
\STATE Run Algorithm \texttt{InferSupportsize2} to infer $\norm{\bA_i}_0$, for $i\not\in\calS$. Then, for $i\not\in\calS$, run Algorithm \texttt{InferIntersection2} to infer $\langle \mathbf{A}_i,\mathbf{A}_j \rangle$ for $j\not\in\calS$, and solve for the membership vector for all elements. 
\STATE Return the similarity matrix $\bA\bA^{T}$. 
\end{algorithmic}
\end{algorithm}

\begin{algorithm}[tb]
\caption{\texttt{InferSupportsize1} The algorithm for inferring $\norm{\bA_i}_{0}$ for a fixed entry $i \in \mathcal{S}$. \label{algo:infersupport1iid}}
\begin{algorithmic}[1]
\REQUIRE  Set $\mathcal{S}$ where every pairwise value is observed, and index $i \in \mathcal{S}$.
   \STATE Define $\Delta$ numbers $E_{\ell}= (|\mathcal{S}|-1)\Big(1-q-(1-2q)(1-p)^\ell\Big)$ for $\ell=0,1,\dots,k$
   \STATE Calculate $T_{i}=\sum_{\substack{r \in \mathcal{S} \\ r\neq i }} \mathds{1}[\mathbf{Y}_{ir}=1]$
   \STATE Return $\mathrm{arg}\min_{\ell} |T_{i}-E_{\ell}|$ 
\end{algorithmic}
\end{algorithm}

\begin{algorithm}[tb]
\caption{\texttt{InferSupportsize2} The algorithm for inferring $\norm{\bA_i}_{0}$ for a fixed entry $i \notin \mathcal{S}$.
\label{algo:infersupport2iid}}
\begin{algorithmic}[1]
\REQUIRE  Set $\mathcal{S}$ where every pairwise value is observed, and index $i \notin \mathcal{S}$.
 \STATE Define $\Delta$ numbers $E_{\ell}= |\mathcal{S}|\Big(1-q-(1-2q)(1-p)^\ell\Big)$ for $\ell=0,1,\dots,k$
 \STATE Calculate $T_{i}=\sum_{r \in \mathcal{S}} \mathds{1}[\mathbf{Y}_{ir}=1]$
\STATE Return $\mathrm{arg}\min_{\ell} |T_{i}-E_{\ell}|$ 
\end{algorithmic}
\end{algorithm}

\begin{algorithm}[tb]
\caption{\texttt{InferIntersection1} The algorithm for inferring $\langle \mathbf{A}_i,\mathbf{A}_j \rangle$ for two fixed entries $i,j \in \mathcal{S}$. \label{algo:infersupport1iidin}}
\begin{algorithmic}[1]
\REQUIRE  Set $\mathcal{S}$ where every pairwise value is observed, and indices $i,j \in \mathcal{S}$.
\IF{$\mathbf{Y}_{ij}=0$}
  \STATE Return 0
\ELSE 
   \STATE Define $\Delta_{i,j}$ numbers $E_{\ell}= (|\mathcal{S}|-2)\Big((1-q)^2-(1-q)(1-2q)(1-p)^{\norm{\bA_i}_{0}}-(1-q)(1-2q)(1-p)^{\norm{\bA_j}_{0}}+(1-2q)^2(1-p)^{\norm{\bA_i}_{0}+\norm{\bA_j}_{0}-\ell} \Big)$ for $\ell=1,\dots,\Delta_{i,j}$
   \STATE Calculate $T_{i,j}=\sum_{\substack{r \in \mathcal{S} \\ r\neq i,j }} \mathds{1}[\mathbf{Y}_{ir}=1 \cap \mathbf{Y}_{jr}=1]$
   \STATE Return $\mathrm{arg}\min_{\ell} |T_{i,j}-E_{\ell}|$ 
\ENDIF    
\end{algorithmic}
\end{algorithm}

\begin{algorithm}[tb]
\caption{\texttt{InferIntersection2} The algorithm for inferring $\langle \mathbf{A}_i,\mathbf{A}_j \rangle$ for $i \in \mathcal{S}, j \not \in \mathcal{S}$. \label{algo:infersupport2iid2}}
\begin{algorithmic}[1]
\REQUIRE  Set $\mathcal{S}$ where every pairwise value is observed, and indices $i \in \mathcal{S}, j \not \in \mathcal{S}$.
\IF{$\mathbf{Y}_{ij}=0$}
  \STATE Return 0
\ELSE 
   \STATE Define $\Delta_{i,j}$ numbers $E_{\ell}= (|\mathcal{S}|-1)\Big((1-q)^2-(1-q)(1-2q)(1-p)^{\norm{\bA_i}_{0}}-(1-q)(1-2q)(1-p)^{\norm{\bA_j}_{0}}+(1-2q)^2(1-p)^{\norm{\bA_i}_{0}+\norm{\bA_j}_{0}-\ell} \Big)$ for $\ell=1,\dots,\Delta_{i,j}$
   \STATE Calculate $T_{i,j}=\sum_{\substack{r \in \mathcal{S} \\ r \neq i }} \mathds{1}[\mathbf{Y}_{ir}=1 \cap \mathbf{Y}_{jr}=1]$
   \STATE Return $\mathrm{arg}\min_{\ell} |T_{i,j}-E_{\ell}|$ 
\ENDIF    
\end{algorithmic}
\end{algorithm}

The main difference between Algorithm~\ref{algo:oracle} and Algorithm~\ref{algo:oracle2} lies in the fact that for the i.i.d. ensemble we first need to infer the number of non-zero elements in every row of $\bA$ (or, $\ell_0$ norm of every row), before proceeding with a similar analysis as in the proof of Theorem~\ref{th:quantizedUniform}. 
As in Section~\ref{app:quantizedUniform}, we analyze the probability of error associated with Algorithm~\ref{algo:oracle2}, by investigating each of its steps. Given a set $\calS$, recall that the second step in this algorithm is to infer the number of non-zero elements $\bA_i$, for $i\in\calS$. This is done with the aid of Algorithm~\ref{algo:infersupport1iid}.For every index $i \in \mathcal{S}$, let
\begin{align}
\calT_i \triangleq \sum_{j \in \mathcal{S}: j \neq i} \mathds{1}[\mathbf{Y}_{ij}=1].\label{calTstat}
\end{align}
Also, let $\ppp{H_{\ell}}_{\ell=0}^{k}$ be a sequence of Hypotheses defined as follows:
\begin{align}
\calH_\ell:\;\norm{\bA_i}_{0}=\ell,\quad \ell=0,\ldots,k.
\end{align}
As before, let $\mathbf{Q}_{ij}\triangleq\calQ(\bA_i^T\bA_j)$. Then, it is clear that 
\begin{align}
\prob_\ell\pp{\mathbf{Q}_{ij}=1} = 1-(1-p)^{\ell},
\end{align}
and therefore,
\begin{align}
\avg_{\ell} \calT_{i}&=(|\calS|-1)\pp{(1-q)(1-(1-p)^{\ell})+q(1-p)^{\ell}}\nonumber\\
& = (|\calS|-1)\pp{1-q-(1-2q)(1-p)^{\ell}}.
\end{align}
Accordingly, for any two different hypotheses $H_{\ell}$ and $H_{\ell'}$, we obtain 
\begin{align}
|\avg_{\ell} \calT_{i}-\avg_{\ell'} \calT_{i}| &=(|\mathcal{S}|-1)(1-2q)\cdot\abs{(1-p)^{\ell}-(1-p)^{\ell'}}.
\end{align}
Now, given $\bA_i$ it is clear that the random variables $\mathds{1}[\mathbf{Y}_{ij}=1]$, for $j\in \mathcal{S}\setminus\{i\}$, are statistically independent and therefore we can apply Lemma~\ref{thm:Chernoff}. Specifically, let $P^{(i)}_{\mathsf{error},1}$ designate the average probability of associated Algorithm~\ref{algo:infersupport1iid}, for a given index $i\in\calS$. Then, we have
\begin{align}
P^{(i)}_{\mathsf{error},1} &= \sum_{\ell=0}^{k}\pr\p{\calH_{\ell}}\pr_\ell\pp{\mathsf{error}}\nonumber\\
& = \sum_{\ell=1}^{\Delta}\pr\p{\calH_{\ell}}\pr_{\ell}\pp{\min_{\ell'\neq\ell}\;\abs{\calT_{i}-\avg_{\ell'}\calT_{i}}<\abs{\calT_{i}-\avg_{\ell}\calT_{i}}}.
\end{align}
As in Appendix~\ref{app:quantizedUniform} (see eqs. \eqref{error_prob1}--\eqref{kldivfirst}), we obtain
\begin{align}
\pr_{\ell}\pp{\min_{\ell'\neq\ell}\;\abs{\calT_{i}-\avg_{\ell'}\calT_{i}}<\abs{\calT_{i}-\avg_{\ell}\calT_{i}}}\leq2\cdot e^{-(|\calS|-1)\bar\eta(\ell)}
\end{align}
where 
\begin{align}
\bar\eta(\ell)\triangleq\frac{(1-2q)^2}{2}\pp{(1-p)^{\ell-1}-(1-p)^{\ell}}^2.
\end{align}
Combining the above results with the fact that $\bar\eta(\ell)$ is monotonically decreasing in $\ell$, we obtain
\begin{align}
P^{(i)}_{\mathsf{error},1} &\leq2\cdot\max_{\ell\geq1}e^{-(|\calS|-1)\cdot\bar\eta(\ell)}\\
&= 2\cdot e^{-(|\calS|-1)\cdot\eta(k)}.\label{eq:concunifnorm}
\end{align}
Therefore, at the end of the first step in the second stage of Algorithm~\ref{algo:oracle2}, we will have an exact estimate of $\norm{\bA_i}_0$, for $i\in\calS$, if \eqref{eq:concunifnorm} is satisfied for all $i\in\calS$. By the union bound, we obtain that the overall probability of error associated with this stage of Algorithm~\ref{algo:oracle2} is upper bounded by
\begin{align}
P_{\mathsf{error},1}&\leq 2|\calS|\cdot e^{-(|\calS|-2)\cdot\eta(k)}\nonumber\\
&\leq 2n\cdot e^{-(|\calS|-1)\cdot\bar\eta(k)}.
\end{align}
Accordingly, taking $|\mathcal{S}|>\frac{1}{\bar\eta(k)}\log(2 n^{1+\varepsilon})+1$, for any $\varepsilon>0$, is sufficient to bring the probability of error to at most $n^{-\varepsilon}$. Note that the above constraint on $|\calS|$ expands to
\begin{align}
|\mathcal{S}|>\frac{2\log(2 n^{1+\varepsilon})}{(1-2q)^2\pp{(1-p)^{k-1}-(1-p)^{k}}^2}+1.\label{calSsizeCondnorm}
\end{align}

After inferring the $\ell_0$-norm of each row, in the second step of Algorithm~\ref{algo:oracle2}, we infer $\langle \mathbf{A}_i,\mathbf{A}_j \rangle$, for any $(i,j)\in\mathcal{S}$. This is done using the procedure in Algorithm~\ref{algo:infersupport1iidin}. The analysis of this procedure is very similar to the analysis in Appendix~\ref{app:quantizedUniform}. In the following probabilities and expectations are evaluated conditioned on $\mathbf{A}_{i}^T\bA_j=\ell$ and the values of $\norm{\bA_i}_{0}$ and $\norm{\bA_j}_{0}$. With some abuse of notation we denote these probabilities and expectations by $\prob_\ell$ and $\avg_\ell$, respectively. For a triplet $(i,j,r)\in [n]$, we have
\begin{align}
\prob_\ell(\mathbf{Q}_{ir}=1 \cap \mathbf{Q}_{jr}=1)&=1- \prob_\ell(\mathbf{Q}_{jr}=0)-\prob_\ell(\mathbf{Q}_{ir}=0)+\prob_\ell(\mathbf{Q}_{ir}=0 \cap \mathbf{Q}_{jr}=0)\nonumber\\
&= 1-(1-p)^{\norm{\bA_i}_{0}}-(1-p)^{\norm{\bA_j}_{0}}+(1-p)^{\norm{\bA_i}_{0}+\norm{\bA_j}_{0}-\ell}.\label{Qdistrib}
\end{align}
In a similar fashion,
\begin{align}
\prob_\ell(\mathbf{Q}_{ir}=0 \cap \mathbf{Q}_{jr}=1) &=\prob_\ell(\mathbf{Q}_{ir}=1 \cap \mathbf{Q}_{jr}=0) \\
&=\prob_\ell(\mathbf{Q}_{ir}=0)-\prob_\ell(\mathbf{Q}_{ir}=0 \cap \mathbf{Q}_{jr}=0)\nonumber\\
&=(1-p)^{\norm{\bA_i}_{0}}-(1-p)^{\norm{\bA_j}_{0}}+(1-p)^{\norm{\bA_i}_{0}+\norm{\bA_j}_{0}-\ell},
\end{align} 
and finally,
\begin{align}
\prob_\ell(\mathbf{Q}_{ir}=0 \cap \mathbf{Q}_{jr}=0) =(1-p)^{\norm{\bA_i}_{0}+\norm{\bA_j}_{0}-\ell}.
\end{align} 
Therefore, using the above we obtain by the law of total probability,
\begin{align}
\prob_\ell(\mathbf{Y}_{ir}=1 \cap \mathbf{Y}_{jr}=1)&=(1-q)^{2}\cdot\prob_\ell(\mathbf{Q}_{ir}=1 \cap \mathbf{Q}_{jr}=1)+2q(1-q)\cdot\prob_\ell(\mathbf{Q}_{ir}=1 \cap \mathbf{Z}_{jr}=1)\nonumber\\
&\hspace{0.5cm}+q^{2}\cdot\prob_\ell(\mathbf{Q}_{ir}=0 \cap \mathbf{Q}_{jr}=0)\nonumber\\
&=(1-q)^{2}-(1-2q)(1-q)(1-p)^{\norm{\bA_i}_{0}}-(1-2q)(1-q)(1-p)^{\norm{\bA_j}_{0}}\nonumber\\
&\hspace{0.5cm}+(1-2q)^{2}(1-p)^{\norm{\bA_i}_{0}+\norm{\bA_j}_{0}-\ell} \nonumber \\
&\triangleq \calT_{\mathsf{th}}.
\end{align}
For a pair $(i,j)\in\calS$, let us define
\begin{align}
\Delta_{i,j}\triangleq\begin{cases}
\min(\norm{\bA_i}_{0},\norm{\bA_j}_{0}), &\mathrm{if}\ \norm{\bA_i}_{0}+\norm{\bA_j}_{0}\leq k\\
\norm{\bA_i}_{0}+\norm{\bA_j}_{0}-k,&\mathrm{if}\ \norm{\bA_i}_{0}+\norm{\bA_j}_{0}\geq k
\end{cases}.
\end{align}
Accordingly, define a sequence of $\Delta_{ij}$ hypotheses $\{\bar{\calH}_\ell\}_{\ell}$:
\begin{align}
\bar{\calH}_{\ell}: \quad \mathbf{A}_{i}^T\bA_j=\ell \quad \text{ for } \ell=0,1,\dots,\Delta_{i,j}.
\end{align} 
Furthermore, define
\begin{align}
\bar{T}_{i,j}\triangleq\sum_{\substack{r \in \mathcal{S} \\ r \neq i,j }} \mathds{1}[\mathbf{Y}_{ir}=1 \cap \mathbf{Y}_{jr}=1].
\end{align} 
It follows that, 
\begin{align}
\avg_{\ell} \bar{T}_{i,j}=(|\mathcal{S}|-2)\calT_{\mathsf{th}}, 
\end{align}
and thus, for any two hypotheses $\bar{\mathcal{H}}_{\ell}$ and $\bar{\mathcal{H}}_{\ell'}$, we have 
\begin{align}
|\avg_{\ell} \bar{T}_{ij}-\avg_{\ell'} \bar{T}_{ij}|&=(|\mathcal{S}|-2)(1-2q)^{2}\cdot\abs{(1-p)^{\norm{\bA_i}_{0}+\norm{\bA_j}_{0}-\ell}-(1-p)^{\norm{\bA_i}_{0}+\norm{\bA_j}_{0}-\ell'}}\nonumber\\
&\geq(|\mathcal{S}|-2)(1-2q)^{2}\abs{(1-p)^{k-1}-(1-p)^{k}}.
\end{align}
Then, using the same machinery as in Appendix~\ref{app:quantizedUniform} (see eqs. \eqref{error_prob1}--\eqref{47eq}), it can be shown that the overall probability of error associated with second stage of Algorithm~\ref{algo:oracle2} is upper bounded by
\begin{align}
\bar{P}_{\mathsf{error},1}&\leq 2\cdot\binom{|\calS|}{2}\cdot e^{-(|\calS|-2)\cdot\tilde\eta(k)}\nonumber\\
&\leq 2n^2\cdot e^{-(|\calS|-2)\cdot\tilde\eta(k)}
\end{align}
where 
$$
\tilde{\eta}(k)\triangleq \frac{(1-2q)^4\pp{(1-p)^{k-1}-(1-p)^k}}{2}.
$$
Therefore, at the end of the second stage of Algorithm~\ref{algo:oracle2}, if $|\mathcal{S}|>\frac{1}{\tilde{\eta}(k)}\log(2 n^{2+\varepsilon})+2$, for any $\varepsilon>0$, then we will have an exact estimate of $\bA_{\calS}\bA_{\calS}^T$ with probability of error to at most $n^{-\varepsilon}$. Note that the above constraint on $|\calS|$ expands to
\begin{align}
|\mathcal{S}|&>\frac{2\log(2 n^{2+\varepsilon})}{(1-2q)^4\pp{(1-p)^{k-1}-(1-p)^k}^2}+2\\
&=\frac{2\log(2 n^{2+\varepsilon})}{p^2(1-2q)^4(1-p)^{2k-2}}+2.\label{calScondiidnoisy}
\end{align}

Next, given the exact estimate of $\bA_{\calS}\bA_{\calS}^T$ from the second step of our algorithm, in the third step we extract the membership of all chosen $|\calS|$ elements by a simple rank factorization procedure as in Algorithm~\ref{algo:oracle_overlap}. We denote the resultant rank factorized matrix by $\hat{\bA}_{\calS}$. Finally, we analyze the fourth step of Algorithm~\ref{algo:oracle2}, in which for each index $j\not\in\mathcal{S}$, we observe $\mathbf{Y}_{ij}$, for all $i\in\mathcal{S}$, and from these we would like to infer the leftover inner-products. This is done with the help of Algorithms~\ref{algo:infersupport2iid} and \ref{algo:infersupport2iid2} which we analyze in the sequel.

In fact the entire analysis of Algorithms~\ref{algo:infersupport2iid} and \ref{algo:infersupport2iid2} remains almost the same. Indeed, in Algorithm~\ref{algo:infersupport2iid} we infer $\norm{\bA_i}_0$, for $i\not\in\calS$. To this end, we define
\begin{align}
    \bar{\calT}_i\triangleq\sum_{j \in \mathcal{S}} \mathds{1}[\mathbf{Y}_{ij}=1].
\end{align}
It is evident that $\bar{\calT}$ is very similar to \eqref{calTstat}, and thus, using the same steps as in \eqref{calTstat}--\eqref{calSsizeCondnorm}, it can be shown that if
\begin{align}
|\mathcal{S}|>\frac{2\log(2 n^{1+\varepsilon})}{(1-2q)^2\pp{(1-p)^{k-1}-(1-p)^{k}}^2}.\label{calSsizeCondnorm2}
\end{align}
then with overwhelming probability we correctly infer $\norm{\bA_i}_0$, for $i\not\in\calS$. Then, using the same arguments in \eqref{Qdistrib}--\eqref{calScondiidnoisy}, it can be shown that if
\begin{align}
|\mathcal{S}|&>\frac{2\log(2 n^{2+\varepsilon})}{p^2(1-2q)^4(1-p)^{2k-2}}+1,\label{calScondiidnoisy2}
\end{align}
then Algorithm~\ref{algo:infersupport2iid2} succeeds, namely, with high probability, at the end of the fourth step of Algorithm~\ref{algo:oracle2}, we have the exact values of $\bA_i^T\bA_j$, for all $i\not\in\calS$, and $j\in\calS$. For each $i\not\in\calS$ we denote by $\mathbf{c}_i$ the $|\calS|$ length vector containing the inner-products $\bA_i^T\bA_j$, for $j\in\calS$. Note that \eqref{calScondiidnoisy} is the stringent condition among \eqref{calSsizeCondnorm}, \eqref{calSsizeCondnorm2}, and \eqref{calScondiidnoisy2}, and thus if \eqref{calScondiidnoisy} holds the other conditions hold too.

Finally, the only thing that is left to do is solve for the membership vector of each of the $(n-|\calS|)$ elements. This is done similarly as was done in Algorithm~\ref{algo:oracle_overlap} (see the proof of Theorem~\ref{th:overlappdirect}). Specifically, from Lemma~\ref{lem:rank_iid}, we know that by taking $|\calS|>k-1+\frac{1+\epsilon}{-\log\max(p,1-p)}+c_2\log n$, for some $\epsilon>0$ and $c_2>0$, the rows of $\mathbf{A}_{\mathcal{S}}$ form a basis of $\mathbb{F}_{2}^{k}$ with high probability. Note that \eqref{calScondiidnoisy} is a stringent condition, and thus Lemma~\ref{lem:rank_iid} holds. Furthermore, were also able to observe that if the rows of $\mathbf{A}_{\mathcal{S}}$ formed a basis, then for an index $i\in\not\in\mathcal{S}$, the set of values of $\{\mathbf{A}_{i}^{T}\mathbf{A}_{j}\}$, for all $j \in \mathcal{S}$ were enough to determine the vector $\mathbf{A}_{i}$. Indeed, given the resultant matrix $\hat{\bA}_{\calS}$ from the rank factorization step in the third step of Algorithm~\ref{algo:oracle2}, the unknown membership vector $\mathbf{c}_j$ of the $j\not\in\calS$ element is found by solving $\hat{\bA}_{\calS}\mathbf{c}_j = \mathbf{s}_j$.

Finally, we conclude the proof by noting that the total number of observed entries is
\begin{align}
|\Omega|={|\mathcal{S}| \choose 2}+|\mathcal{S}|\cdot(n-|\mathcal{S}|),
\end{align}
while \eqref{calScondiidnoisy} ensures a vanishing error probability.

\section{Proof of Theorem~\ref{th:quantizednoisyunkown}}

In this section, we analyze Algorithm~\ref{algo:noisy}. At the end of the second stage of Algorithm~\ref{algo:noisy}, we have access to all counts $T_{ij}$, for all pairs $(i,j) \in \mathcal{S}$, and $i\in\mathcal{S}, j\notin\mathcal{S}$. Suppose that these counts satisfy
\begin{align}
&\max_{\substack{T_{i_1j_1} \in H_{\ell} \\ T_{i_2j_2} \in H_{\ell}}} |T_{i_1 j_1}-T_{i_2 j_2}| \le \delta \label{eqn:cond1}\\ 
&\min_{\substack{T_{i_1j_1} \in H_{\ell} \\ T_{i_2j_2} \in H_{\ell'} \\ \ell \neq \ell'}} |T_{i_1 j_1}-T_{i_2 j_2}| > 2\delta\label{eqn:cond2}
\end{align}
where $T_{ij} \in H_{\ell}$ implies that $\bA_i^T \bA_j=\ell$. 
Now, according to the third stage of Algorithm~\ref{algo:noisy}, we group the counts $\{T_{ij}\}$ with the objective of forming $(\Delta+1)$ clusters such that the count difference between any two intra-cluster points is less than the count difference between any two inter-cluster points. We next prove that counts belonging to two distinct hypotheses $H_{\ell}$ and $H_{\ell'}$ must also belong to different clusters. We prove this property by contradiction. 

Indeed, the above claim can be wrong only if one of the following two situations happen: First, there are two clusters $\calA$ and $\calB$ both of which contain counts belonging to $H_{\ell}$ and $H_{\ell'}$. Denote the relevant counts in $\calA$ by $a_{\ell}$ and $a_{\ell'}$, and the counts in $\calB$ by $b_{\ell}$ and $b_{\ell'}$, where $a_{\ell},b_{\ell} \in H_{\ell}$ and $a_{\ell'},b_{\ell'} \in H_{\ell'}$. Then, according to \eqref{eqn:cond1}--\eqref{eqn:cond2}, we must have $|a_{\ell}-a_{\ell'}|>|a_{\ell'}-b_{\ell'}|$, but this clearly contradicts the way the clusters were formed in the third step of Algorithm~\ref{algo:noisy}. The second situation is when all counts belonging to $H_{\ell}$ and $H_{\ell'}$ are in the same cluster. However, since our objective is to find $(\Delta+1)$ clusters, the counts in a particular hypotheses has to split into multiple clusters for this to happen. This implies, for example, that there exists three clusters $\calA$, $\calB$, and $\calC$, and three hypotheses $H_{\ell}$, $H_{\ell'}$, and $H_{\hat{\ell}}$, such that that $\calA$ and $\calB$ contain counts belonging to $H_{\ell}$ only and $\calC$ contains counts from $H_{\ell'}$ and $H_{\hat{\ell}}$. But then there exist counts in $\calC$ whose difference is at least $2\delta$, whereas the maximum difference between counts in $\calA$ and $\calB$ is $\delta$ (since both contain counts from the same hypothesis), which again clearly contradicts the solution of the proposed algorithm. Therefore, we may conclude that, by construction, counts belonging to different hypotheses must belong to different clusters. Since we look for $(\Delta+1)$ clusters, we exactly recover the clusters where each cluster corresponds to the counts of a particular hypothesis only.
Moreover, we can correctly label the clusters as well because of the monotonicity of $\ell$ in the value of the counts belonging to 
hypothesis $H_{\ell}$ provided we have a valid solution by the algorithm. 

In the following, we derive the sufficient conditions under which  \eqref{eqn:cond1}-\eqref{eqn:cond2} are satisfied. First, note that in Algorithm~\ref{algo:noisy} when computing the triangle counts for pairs $(i,j)$, such that $i\in\calS$ and $j\not\in\calS$, we omit one arbitrarily picked element (denoted by $x_j$ where $x_j \neq i$) from $\mathcal{S}$. We do that because we want the expected value of the triangle count under the different hypotheses to be the same as in the case when $(i,j)\in\mathcal{S}$. Accordingly, recall \eqref{meanCalc}. In order to satisfy \eqref{eqn:cond1}-\eqref{eqn:cond2}, it is clear that $T_{ij}$ should deviate from its mean by at most $\min_{\ell,\ell':\ell \neq \ell'} \frac{\abs{\avg_{\ell} T_{ij}-\avg_{\ell'}T_{ij}}}{6}$, which implies that
\begin{align}
&\delta=\frac{(|\mathcal{S}|-2)(1-2q)^{2}}{3{k \choose \Delta}}\cdot\abs{{k-2\Delta+1 \choose \Delta}-{k-2\Delta \choose \Delta}}
. 
\end{align} 
Then, using the same machinery as in Appendix~\ref{app:1} (see, eq. \eqref{error_prob1}--\eqref{47eq}), it can be shown that
at the end of the third step of Algorithm~\ref{algo:noisy}, the overall probability of error is upper bounded by
\begin{align}
P_{\mathsf{error}}&\leq 2n^2\cdot e^{-(|\calS|-2)\cdot\bar{\eta}}\label{477eq}
\end{align}
where
\begin{align}
\bar{\eta}\triangleq\frac{(1-2q)^{4}}{18}\frac{\abs{{k-2\Delta+1 \choose \Delta}-{k-2\Delta \choose \Delta}}^2}{{{k \choose \Delta}}^2}.
\end{align}
Accordingly, taking 
\begin{align}
|\mathcal{S}|>\frac{18\binom{k}{\Delta}^2\log(2 n^{2+\varepsilon})}{(1-2q)^4\pp{\binom{k-2\Delta+1}{\Delta}-\binom{k-2\Delta}{\Delta}}^2}+2,\label{calSsizeCond2unkown}
\end{align}
for any $\varepsilon>0$, is sufficient to bring the probability of error to at most $n^{-\varepsilon}$.

It is evident that for the algorithm to return a valid solution, there must exist counts for all the $(\Delta+1)$ hypotheses. We will show that this event happens with high probability under some conditions. For two indices $i,j \in [N]$, we have
\begin{align}{\label{eq:prob_simple}}
\prob(\bA_{i}^T \bA_{j}=\ell)=\frac{{\Delta \choose \ell}{k-\Delta \choose \Delta-\ell}}{{k \choose \Delta}}.
\end{align}
Then, it is clear that \eqref{eq:prob_simple} is minimized when $\ell=\Delta$, in which case we have $\prob(\bA_{i}^T \bA_{j}=\Delta)=\frac{1}{{k \choose \Delta}}$. If we only focus on an index $i \in \mathcal{S}$ (we are selecting an index in $\calS$ because indices in $\calS$ are queried with every other index in $[N]$), then let $U_{i,\ell}$ be the random variable which describes the number of indices (excluding $i$ itself) such that $\bA_{i}^T \bA_{j}=\ell$. It is clear that $U_{i,\ell}$ can be written as a sum of $(n-1)$ i.i.d. binary random variables, and 
\begin{align}
\avg(U_{i,\ell})=\frac{(n-1){\Delta \choose \ell}{k-\Delta \choose \Delta-\ell}}{{k \choose \Delta}}.
\end{align}
Applying Chernoff's inequality once again, and taking a union bound over all $(\Delta+1)$ hypotheses, we may conclude that if $n>10{k \choose \Delta}\log n$, then $U_{i,\ell}>0$, for all $\ell$, with high probability. 

\section{Proof of Theorem~\ref{th:ditheredExact}}\label{app:dithered}

\begin{algorithm}[tb]
\caption{\texttt{NoisyInferSupport1} The algorithm for inferring $\langle \mathbf{A}_i,\mathbf{A}_j \rangle$ for two fixed entries $i,j \in \mathcal{S}$. \label{algo:infersupport1dithered}}
\begin{algorithmic}[1]
\REQUIRE  Set $\mathcal{S}$ where every pairwise value is observed, and indices $i,j \in \mathcal{S}$
   \STATE Define $\Delta+1$ numbers $E_{\ell}= (|\mathcal{S}|-2)\Bigg[1-2\mathbb{E}_{\ell}\pp{Q\p{\frac{\bA_j^T\bA_r}{\sigma}}}-\mathbb{E}_{\ell}\pp{Q\p{\frac{\bA_j^T\bA_r}{\sigma}}Q\p{\frac{\bA_i^T\bA_r}{\sigma}}}\Bigg]$ for $\ell=0,1,\dots,\Delta$
   \STATE Calculate $T_{ij}=\sum_{\substack{r \in \mathcal{S} \\ r\neq i,j }} \mathds{1}[\mathbf{Y}_{ir}=1 \cap \mathbf{Y}_{jr}=1]$
   \STATE Return $\mathrm{arg}\min_{\ell} |T_{ij}-E_{\ell}|$ 
\end{algorithmic}
\end{algorithm}
\begin{algorithm}[tb]
\caption{\texttt{NoisyInferSupport2} The algorithm for inferring $\langle \mathbf{A}_i,\mathbf{A}_j \rangle$ for $i \in \mathcal{S}, j \not \in \mathcal{S}$. \label{algo:infersupport2dithered}}
\begin{algorithmic}[1]
\REQUIRE  Set $\mathcal{S}$ where every pairwise value is observed, and indices $i \in \mathcal{S}, j \not \in \mathcal{S}$.
   \STATE Define $\Delta$ numbers $E_{\ell}= (|\mathcal{S}|-1)\Bigg[1-2\mathbb{E}_{\ell}\pp{Q\p{\frac{\bA_j^T\bA_r}{\sigma}}}-\mathbb{E}_{\ell}\pp{Q\p{\frac{\bA_j^T\bA_r}{\sigma}}Q\p{\frac{\bA_i^T\bA_r}{\sigma}}}\Bigg]$ for $\ell=0,1,\dots,\Delta$
   \STATE Calculate $T_{ij}=\sum_{\substack{r \in \mathcal{S} \\ r \neq i }} \mathds{1}[\mathbf{Y}_{ir}=1 \cap \mathbf{Y}_{jr}=1]$
   \STATE Return $\mathrm{arg}\min_{\ell} |T_{ij}-E_{\ell}|$ 
\end{algorithmic}
\end{algorithm}
The algorithm for this setting is the same as Algorithm~\ref{algo:oracle}, but with Algorithms~\ref{algo:Noisyinfersupport1} and \ref{algo:Noisyinfersupport2} replaced with Algorithms~\ref{algo:infersupport1dithered} and \ref{algo:infersupport2dithered}. Accordingly, the main difference in the analysis compared to Appendix~\ref{app:1} is the computation of the statistics of the enumerators, and thus we omit some technical details. Specifically, recall that we assume that $\bA$ was generated according to the uniform ensemble. Now, as before, we notice that for three distinct indices $(i,j,r)\in[n]$, we have 
\begin{align}
&\prob_{\ell}(\mathbf{Y}_{ir}=1 \cap \mathbf{Y}_{jr}=1)=1-2\cdot\prob_\ell(\mathbf{Y}_{ir}=0)+\prob_\ell(\mathbf{Y}_{ir}=0 \cap \mathbf{Y}_{jr}=0).
\end{align} 
Then, it is clear that
\begin{align}
\prob_\ell(\mathbf{Y}_{jr}=0)&=\mathbb{E}_{\ell}\pp{Q\p{\frac{\bA_j^T\bA_r}{\sigma}}},\label{firstexpec}
\end{align}
and
\begin{align}
&\prob_\ell(\mathbf{Y}_{ir}=0 \cap \mathbf{Y}_{jr}=0) = \mathbb{E}_{\ell}\pp{Q\p{\frac{\bA_j^T\bA_r}{\sigma}}Q\p{\frac{\bA_i^T\bA_r}{\sigma}}}.\label{secondexpec}
\end{align}
It is also clear that \eqref{firstexpec} is independent of $\ell$ and $(i,j,r)$, while \eqref{secondexpec} depends on $\ell$ only. Therefore,
\begin{align}
\prob_\ell(\mathbf{Y}_{ir}=1 \cap \mathbf{Y}_{jr}=1)=1-\mathbb{E}_{\ell}\pp{2Q\p{\frac{\bA_j^T\bA_r}{\sigma}}-Q\p{\frac{\bA_j^T\bA_r}{\sigma}}Q\p{\frac{\bA_i^T\bA_r}{\sigma}}}.
\end{align}
Next, as before, for a pair of indices $(i,j)\in\calS$, define
\begin{align}
T_{i,j}\triangleq\sum_{\substack{r \in \mathcal{S} \\ r\neq i,j }} \mathds{1}[\mathbf{Y}_{ir}=1 \cap \mathbf{Y}_{jr}=1],
\end{align} 
and thus,
\begin{align}
&\avg_{\ell} T_{i,j}=(|\mathcal{S}|-2)\left[1-2\mathbb{E}_{\ell}\pp{Q\p{\frac{\bA_j^T\bA_r}{\sigma}}}-\mathbb{E}_{\ell}\pp{Q\p{\frac{\bA_j^T\bA_r}{\sigma}}Q\p{\frac{\bA_i^T\bA_r}{\sigma}}}\right].
\end{align}
Accordingly for any two hypotheses $\mathcal{H}_{\ell}$ and $\mathcal{H}_{\ell'}$, we have 
\begin{align}
&\abs{\avg_{\ell} T_{ij}-\avg_{\ell'} T_{ij}}\nonumber\\
&=(|\mathcal{S}|-2)\left|\mathbb{E}_{\ell}\pp{Q\p{\frac{\bA_j^T\bA_r}{\sigma}}Q\p{\frac{\bA_i^T\bA_r}{\sigma}}}-\mathbb{E}_{\ell'}\pp{Q\p{\frac{\bA_j^T\bA_r}{\sigma}}Q\p{\frac{\bA_i^T\bA_r}{\sigma}}}\right|\\
&\triangleq(|\mathcal{S}|-2)\cdot \Gamma_{\ell,\ell'}.
\end{align}
Then, using the same machinery as in Appendix~\ref{app:quantizedUniform} (see eqs. \eqref{error_prob1}--\eqref{47eq}), it can be shown that the overall probability of error associated with second stage of Algorithm~\ref{algo:oracle2} for the dithered oracle is upper bounded by
\begin{align}
P_{\mathsf{error},1}&\leq 2n^2\cdot e^{-(|\calS|-2)\cdot\frac{\Gamma_{1,0}^2}{2}}.
\end{align}
Therefore, at the end of the second stage of Algorithm~\ref{algo:oracle}, if $|\mathcal{S}|>\frac{2}{\Gamma_{1,0}^2}\log(2 n^{2+\varepsilon})+2$, for any $\varepsilon>0$, then we will have an exact estimate of $\bA_{\calS}\bA_{\calS}^T$ with probability of error to at most $n^{-\varepsilon}$. The other parts of the algorithm are handled in the same way (see eqs. \eqref{calSsizeCond}--\eqref{sampleComplexity}, and thus omitted. We emphasize that as before, the over all query complexity $\binom{|\calS|}{2}+|\calS|\cdot(n-|\calS|)$ is dominated by the above condition on $\calS$. 

\section{Information-Theoretic Lower Bounds}\label{app:IT_limit}

\subsection{Proof of Theorem~\ref{thm:coverse_iid}}

\subsubsection{Proof of Eq.~\ref{conv_iid_direct}}

We consider the case where $\mathbf{A}$ was generated according to the i.i.d. ensemble. We observe $\abs{\Omega}$ elements, drawn uniformly at random from the matrix $\mathbf{Y}$, where $\mathbf{Y}_{ij}=\calO_{\mathsf{direct}}(i,j) = \bA_i^T\bA_j$. Let $\mathsf{P_{error}}$ denotes the average probability of error associated with any estimator of $\mathbf{A}$ given the observations $\mathbf{Y}_{\Omega}$, namely, $\mathsf{P_{error}}\triangleq\prob\{\hat{\bA}(\mathbf{Y}_{\Omega})\neq\mathbf{A}\}$. We note that
\begin{align}
H(\mathbf{A}) &= H(\mathbf{A}\vert\Omega)\\
& = I(\mathbf{A};\mathbf{Y}_{\Omega}\vert\Omega)+H(\mathbf{A}\vert\mathbf{Y}_{\Omega},\Omega)\\
&\stackrel{\text{Fano}}{\leq} I(\mathbf{A};\mathbf{Y}_{\Omega}\vert\Omega)+nk\cdot \lambda_{\mathsf{error}}\\
& = H(\mathbf{Y}_{\Omega}\vert\Omega)-H(\mathbf{Y}_{\Omega}\vert\mathbf{A},\Omega)+nk\cdot \mathsf{P_{error}}\\
& \stackrel{H(\mathbf{Y}_{\Omega}\vert\mathbf{A},\Omega)=0}{=} H(\mathbf{Y}_{\Omega}\vert\Omega)+nk\cdot \mathsf{P_{error}}\label{fanofirst}
\end{align}
where the inequality follows from Fano's  inequality \cite{cover2012elements} which implies that
\begin{align}
H(\mathbf{A}\vert\mathbf{Y}_{\Omega})\leq \mathsf{P_{error}}\cdot\log\abs{\calA}\leq nk\cdot \mathsf{P_{error}}
\end{align} 
where $\calA$ is the set of all possible $n\times k$ binary matrices, and thus $\abs{\calA} = 2^{nk}$. Since $\bA$ is an i.i.d. matrix with $\mathsf{Bernoulli}(p)$ elements, we have $H(\bA) = nk\cdot\calH_2(p)$. Therefore, we obtain that
\begin{align}
nk\cdot\calH_2(p)\leq H(\mathbf{Y}_{\Omega}\vert\Omega)+nk\cdot \mathsf{P_{error}}.\label{afterFano}
\end{align}
It is only left to upper bound the entropy $H(\mathbf{Y}_{\Omega}\vert\Omega)+nk\cdot \mathsf{P_{error}}$. It is clear that
\begin{align}
H(\mathbf{Y}_{\Omega}\vert\Omega)&\leq\abs{\Omega}\cdot\max_{i\neq j}H(\bA_i^T\bA_j)\\
&\leq \abs{\Omega}\cdot\log k
\end{align}
where the second inequality follows from the realization that $\bA_i^T\bA_j$ has a maximum value of $k$. Therefore, using \eqref{afterFano}, we obtain
\begin{align}
nk\cdot\calH_2(p)\leq \abs{\Omega}\cdot\log k+nk\cdot \mathsf{P_{error}}.
\end{align}
Accordingly, to achieve $\mathsf{P_{error}}\leq\delta$, it is necessary that
\begin{align}
|\Omega|\geq \frac{nk}{\log k}[\calH_2(p)-\delta],
\end{align}
as claimed.

\subsubsection{Proof of Eq.~\ref{conv_iid_quantized}}

In this subsection we deal with the noisy quantized oracle, i.e., $\mathbf{Y}_{ij}=\calO_{\mathsf{quantized}}(\bA_i^T\bA_j)\oplus W_{ij}$. Similarly to \eqref{fanofirst}, we have
\begin{align}
H(\mathbf{A}) &\leq H(\mathbf{Y}_{\Omega}\vert\Omega)-H(\mathbf{Y}_{\Omega}\vert\mathbf{A},\Omega)+nk\cdot \mathsf{P_{error}}\\
&= H(\mathbf{Y}_{\Omega}\vert\Omega)-|\Omega|\cdot\calH_2(q)+nk\cdot \mathsf{P_{error}}\label{entropyAATran}
\end{align}
where we have used the fact that $H(\mathbf{Z}_{\Omega}\vert\mathbf{A},\Omega)=|\Omega|\cdot\calH_2(q)$. We next evaluate $H(\mathbf{Y}_{\Omega}\vert\Omega)$. Given $\Omega$, the $(i,j)$ element of $\mathbf{Y}$ is a Bernoulli random variable with success probability given by $q\star\beta_{ij}$, where $\beta_{ij}\triangleq\prob\{\mathbf{A}_i^T\mathbf{A}_j>0\}$, and $\star$ denotes the binary convolution. Now, note that for $i=j$,
\begin{align}
\beta_{ii}&=\prob\{\norm{\mathbf{a}_i}^2>0\} = 1-\prob\{\norm{\mathbf{a}_i}^2=0\}\nonumber\\
&= 1-(1-p)^k,\label{betaii}
\end{align}
and $i\neq j$,
\begin{align}
\beta_{ij} &= 1-\prob\{\mathbf{a}_i^T\mathbf{a}_j=0\}\nonumber\\
&= 1-(1-p^2)^{k}.\label{betaij}
\end{align}
Therefore,
\begin{align}
H(\mathbf{Y}_{\Omega}\vert\Omega)&\leq\abs{\Omega}\cdot\max_{i,j}\mathcal{H}_2(q\star \beta_{ij})\nonumber\\ &\leq\abs{\Omega}\cdot\calH_2\p{q\star\pp{1-(1-p^2)^{k}}}.\label{upperboundentropyiid}
\end{align}
Combining \eqref{entropyAATran}, \eqref{upperboundentropyiid}, and the fact that $H(\bA)=nk\cdot\calH_2(p)$, we obtain
\begin{align}
nk\cdot\calH_2(p)&\leq\abs{\Omega}\cdot\calH_2\p{q\star\pp{1-(1-p^2)^{k}}}-|\Omega|\cdot\calH_2(q)+nk\cdot \mathsf{P_{error}}.
\end{align}
Accordingly, to achieve $\mathsf{P_{error}}\leq\delta$, it is necessary that
\begin{align}
|\Omega|\geq nk\cdot\frac{\calH_2(p)-\delta}{\calH_2\p{q\star\pp{1-(1-p^2)^{k}}}-\calH_2(q)},
\end{align}
as claimed.

\subsubsection{Proof of Eq.~\ref{conv_iid_dithered}}\label{app:ditheredUniform}

We now consider the dithered oracle, where $\mathbf{Y}_{ij}=\calQ(\bA_i^T\bA_j+Z_{ij})$, with $Z_{ij}\sim\mathsf{Normal}(0,\sigma^2)$. Here, the analysis is very similar to the previous subsection. In particular, similarly to \eqref{entropyAATran}, we have
\begin{align}
H(\mathbf{A}) &\leq H(\mathbf{Y}_{\Omega}\vert\Omega)-H(\mathbf{Y}_{\Omega}\vert\mathbf{A},\Omega)+nk\cdot \mathsf{P_{error}}\\
&= H(\mathbf{Y}_{\Omega}\vert\Omega)-|\Omega|\cdot\bE\calH_2\pp{ Q\p{\frac{\bA_1^T\bA_2}{\sigma}}}+nk\cdot \mathsf{P_{error}}.\label{ditheredFano}
\end{align}
It is clear that given $\Omega$, the $(i,j)$ element of $\mathbf{Y}$ is a Bernoulli random variable with success probability $\bE Q\p{\frac{\bA_1^T\bA_2}{\sigma}}$. Therefore, we obtain
\begin{align}
H(\mathbf{Y}_{\Omega}\vert\Omega)\leq\abs{\Omega}\cdot\calH_2\pp{\bE Q\p{\frac{\bA_1^T\bA_2}{\sigma}}}.\label{upperboundDitheredEntropy}
\end{align}
Combining the above results and the fact that $H(\bA)=nk\calH_2(p)$, we may conclude that
\begin{align}
&nk\calH_2(p)\leq\abs{\Omega}\cdot\calH_2\pp{\bE Q\p{\frac{\bA_1^T\bA_2}{\sigma}}}-|\Omega|\cdot\bE\calH_2\pp{ Q\p{\frac{\bA_1^T\bA_2}{\sigma}}}+nk\cdot \mathsf{P_{error}}.
\end{align}
Accordingly, to achieve $\mathsf{P_{error}}\leq\delta$, it is necessary that
\begin{align}
|\Omega|\geq\frac{nk\cdot[\calH_2(p)-\delta]}{\calH_2\pp{\bE Q\p{\frac{\bA_1^T\bA_2}{\sigma}}}-\bE\calH_2\pp{ Q\p{\frac{\bA_1^T\bA_2}{\sigma}}}},
\end{align}
as claimed.

\subsection{Proof of Theorem~\ref{thm:coverse_uniform}}

\subsubsection{Proof of Eq.~\ref{conv_uniform_direct}}
We consider the case $\mathbf{A}$ where was generated according to the uniform ensemble, and the oracle response is $\mathbf{Y}_{ij} = \bA_i^T\bA_j$. Similarly as in \eqref{fanofirst}, we have
\begin{align}
H(\mathbf{A}) \leq H(\mathbf{Y}_{\Omega}\vert\Omega)+nk\cdot \mathsf{P_{error}}. 
\end{align}
For the uniform ensemble, note that $H(\mathbf{A}) = n\cdot\log\binom{k}{\Delta}$. Next, as in the previous subsection, note that 
\begin{align}
H(\mathbf{Y}_{\Omega}\vert\Omega)&\leq\abs{\Omega}\cdot\max_{i\neq j}H(\bA_i^T\bA_j)\\
&\leq\abs{\Omega}\cdot\log\Delta
\end{align}
where the second inequality follows from the realization that $\bA_i^T\bA_j$ has a maximum value of $\Delta$. Combining the above, we obtain
\begin{align}
n\cdot\log\binom{k}{\Delta}&\leq\abs{\Omega}\cdot\log\Delta+nk\cdot \mathsf{P_{error}}.
\end{align}
Accordingly, to achieve $\mathsf{P_{error}}\leq\delta$, it is necessary that
\begin{align}
|\Omega|\geq nk\cdot\frac{\frac{1}{k}\log\binom{k}{\Delta}-\delta}{\log\Delta},
\end{align}
as claimed.

\subsubsection{Proof of Eq.~\ref{conv_uniform_quantized}}

We now deal with the noisy quantized oracle, i.e., $\mathbf{Y}_{ij}=\calO_{\mathsf{quantized}}(\bA_i^T\bA_j)\oplus W_{ij}$. Similarly to \eqref{entropyAATran}, we have
\begin{align}
H(\mathbf{A}) &\leq H(\mathbf{Y}_{\Omega}\vert\Omega)-|\Omega|\cdot\calH_2(q)+nk\cdot \mathsf{P_{error}}.
\end{align}
It is clear that given $\Omega$, the $(i,j)$ element of $\mathbf{Y}$ is a Bernoulli random variable with success probability $\beta_{ij}\star q$, where $\beta_{ij}\triangleq\prob\{\mathbf{a}_i^T\mathbf{a}_j>0\}$. Note that for $i=j$,
\begin{align}
\beta_{ii}=\prob\{\norm{\mathbf{a}_i}^2>0\} &= 1,\label{betaii2}
\end{align}
while $i\neq j$,
\begin{align}
\beta_{ij} &= 1-\prob\{\mathbf{a}_i^T\mathbf{a}_j=0\}= 1-\frac{{{k-\Delta}\choose{\Delta}}}{{{k}\choose{\Delta}}}.\label{betaij2}
\end{align}
Therefore, using the above we obtain
\begin{align}
H(\mathbf{Y}_{\Omega}\vert\Omega)\leq\abs{\Omega}\cdot\calH_2\p{q\star\frac{{{k-\Delta}\choose{\Delta}}}{{{k}\choose{\Delta}}}}.
\end{align}
Combining the above results and the fact that $H(\bA)=n\cdot\log\binom{k}{\Delta}$, we may conclude that
\begin{align}
n\cdot\log\binom{k}{\Delta}&\leq\abs{\Omega}\cdot\calH_2\p{q\star\frac{{{k-\Delta}\choose{\Delta}}}{{{k}\choose{\Delta}}}}-\abs{\Omega}\cdot\calH_2(q)+nk\cdot \mathsf{P_{error}}.
\end{align}
Accordingly, to achieve $\mathsf{P_{error}}\leq\delta$, it is necessary that
\begin{align}
|\Omega|\geq nk\cdot\frac{\frac{1}{k}\log\binom{k}{\Delta}-\delta}{\calH_2\p{q\star\frac{{{k-\Delta}\choose{\Delta}}}{{{k}\choose{\Delta}}}}-\calH_2(q)},
\end{align}
as claimed.

\subsubsection{Proof of Eq.~\ref{conv_uniform_dithered}}
We now consider the dithered oracle, where $\mathbf{Y}_{ij}=\calQ(\bA_i^T\bA_j+Z_{ij})$, with $Z_{ij}\sim\mathsf{Normal}(0,\sigma^2)$. Here, the analysis is very similar to the Subsection~\ref{app:ditheredUniform}. In particular, similarly to \eqref{ditheredFano}, we have
\begin{align}
H(\mathbf{A}) &\leq  H(\mathbf{Y}_{\Omega}\vert\Omega)-|\Omega|\cdot\bE\calH_2\pp{ Q\p{\frac{\bA_1^T\bA_2}{\sigma}}}+nk\cdot \mathsf{P_{error}}.
\end{align}
Also, similarly to \eqref{upperboundDitheredEntropy}, we have
\begin{align}
H(\mathbf{Y}_{\Omega}\vert\Omega)\leq\abs{\Omega}\cdot\calH_2\pp{\bE Q\p{\frac{\bA_1^T\bA_2}{\sigma}}}.
\end{align}
Combining the above results and the fact that $H(\bA)=n\log\binom{k}{\Delta}$, we conclude that
\begin{align}
&n\log\binom{k}{\Delta}\leq\abs{\Omega}\cdot\calH_2\pp{\bE Q\p{\frac{\bA_1^T\bA_2}{\sigma}}}\nonumber\\
&\hspace{1cm}-|\Omega|\cdot\bE\calH_2\pp{ Q\p{\frac{\bA_1^T\bA_2}{\sigma}}}+nk\cdot \mathsf{P_{error}}.
\end{align}
Accordingly, to achieve $\mathsf{P_{error}}\leq\delta$, it is necessary that
\begin{align}
|\Omega|\geq\frac{nk\cdot[\frac{1}{k}\log\binom{k}{\Delta}-\delta]}{\calH_2\pp{\bE Q\p{\frac{\bA_1^T\bA_2}{\sigma}}}-\bE\calH_2\pp{ Q\p{\frac{\bA_1^T\bA_2}{\sigma}}}},
\end{align}
as claimed.

\section{Worst Case Model: At Most 2 Clusters}\label{app:Delta2}

\begin{algorithm}[tb]
\caption{\texttt{Worst-Case Quantized Responses for $\Delta=2$} The algorithm for extracting membership of elements via queries to oracle for adversarial data. \label{algo:adv}}
\begin{algorithmic}[1]
\REQUIRE Number of elements: $N$, number of clusters $k$, oracle responses $\calO_{\mathsf{quantized}}(i,j)$ for query $(i,j)\in\Omega$, where $i,j \in [N]$.
\STATE Choose a set $\mathcal{S}$ of elements drawn uniformly at random from $[N]$, and perform all pairwise queries corresponding to these $|\calS|$ elements.
\STATE Construct a graph $\calG=(\calV,\calE)$ where the vertices are the $|\calS|$ sampled elements. There exist an edge between elements $(i,j)$ only if they are determined to be similar by the oracle. 
\STATE Construct the maximal cliques of the graph $\calG$ such that all edges in $\calE$ are covered and no three cliques intersect. Each maximal clique forms a cluster. 
\STATE Query each of the remaining $n-|\calS|$ elements with all elements present in $\calS$. For each cluster, if an element is similar with all the elements in that particular cluster, then assign the element to that cluster.
\STATE Return all the clusters.
\end{algorithmic}
\end{algorithm}
In this section we prove the following special result for $\Delta=2$.
\begin{thm}\label{thm:advDelta2}
Let $\calN_i$ be the set of elements which belong to the $i$'th cluster, and assume that $\Delta=2$. If, for every triplets of distinct clusters $p,q,r \in [k]$, we have $|\calN_p \setminus \{\calN_q \cup \calN_r\}|>\alpha\cdot n$, for some $\alpha>0$, then by using Algorithm~\ref{algo:adv}, ${T \choose 2}+T(n-T)$ queries are sufficient to recover the clusters, where $\alpha\cdot T=3\log k+\log n$.
\end{thm}

For ease of notation, we will say that an element tests \emph{positive} with another element if the response to their query is $1$ (i.e., they have one cluster in common). Otherwise, we will say they test \emph{negative}. We will also say that a cluster is \emph{maximal} if there does not exist any element that does not belong to the cluster but tests positive with every element in the cluster. The proof of Theorem~\ref{thm:advDelta2} hangs on the following theorem. 
\begin{thm}{\label{thm:delta2}}
Let $\calC$ be a given clustering and let $\calN_i$ be the set of elements which belong to the $i$'th cluster. If for every triplets of distinct clusters $p,q,r \in [k]$, we have  $\calN_p \setminus \{\calN_q \cup \calN_r\} \neq \phi$, then the ground truth clustering $\calC$ is the only valid clustering that is consistent with the entire query matrix.
\end{thm}
To prove this result we need the following lemma.
\begin{lem}\label{lem:maximal}
For a given clustering $\calC$, if for every triplets of distinct clusters $p,q,r \in [k]$, we have $\calN_p \setminus \{\calN_q \cup \calN_r\} \neq \phi$, then the clusters $\calN_i$ are maximal.
\end{lem}
\begin{proof}{Proof of Lemma~\ref{lem:maximal}}
We will prove this by contradiction. Suppose there exists a cluster $\calN_i$ which is not maximal and there exist an element $x \not \in \calN_i$ such that $x$ tests positive with every element in $\calN_i$. This is only possible if $\{x\} \cup \calN_i \subset \calN_j$ for some $j$ or if $\{x\} \cup \calN_i \subset \calN_j \cup \calN_k$ (A partition of $\calN_i$ into two sets $\calU$ and $\calV$ such that $\{x\} \cup \calU \subset \calN_j$ and $\{x\} \cup \calV \subset \calN_k$). Both these situations are not allowed according to our guarantees ($\calN_i \setminus \{\calN_j \cup \calN_k\} \neq \phi$), which completes the proof.
\end{proof}
We now prove Theorem~\ref{thm:delta2}.
\begin{proof}[Proof of Theorem~\ref{thm:delta2}]
We will prove this result by induction on the number of clusters. Consider the base case of $k=3$ where there are only three clusters say $\calN_1,\calN_2,\calN_3$. Now the sets $\calN_1\setminus \{\calN_2 \cup \calN_3 \},\calN_2\setminus \{\calN_1 \cup \calN_3 \},\calN_3 \setminus \{\calN_1 \cup \calN_2\}$ are non-empty and disjoint. In any different clustering $\tilde{\calC}$, these three aforementioned sets have to belong to different clusters. Without loss of generality, assume that $\calN_1\setminus \{\calN_2 \cup \calN_3 \} \subset \tilde{\calN_1}$ and $\calN_2\setminus \{\calN_1 \cup \calN_3 \} \subset \tilde{\calN_2}$. In that case, it is easy to see that any element in $\calN_1 \cap \calN_2$ must belong to both $\tilde{\calN_1}$ and $\tilde{\calN_2}$ since it must test positive with elements in both $\calN_1\setminus \{\calN_2 \cup \calN_3 \}$ and $\calN_2\setminus \{\calN_1 \cup \calN_3 \}$. With this argument we get that the clustering $\tilde{\calC}$ is the same as the clustering $\calC$. 

Now, assume that this lemma is true when there are $k$ clusters. Under this assumption, we will prove the statement of the lemma for $k+1$ clusters by contradiction. Assume that there exists a different clustering $\tilde{\calC}$ such that there does not exist any $i,j \in [k]$ for which $\calN_i=\tilde{\calN_j}$. If $\calN_1$ is a disjoint cluster that is $\calN_1 \cap \calN_j =\phi$ for all clusters $\calN_j$, then all elements in $\calN_1$ must belong to a disjoint cluster in $\tilde{\calC}$ and we must have $\tilde{\calC}$ to be the same as $\calC$ by using the induction assumption. So now, we assume that no cluster $\calN_i$ is disjoint. Assume that there exists some $i,j$ such that $\calN_i \subset \tilde{\calN}_j$. Since $\tilde{\calC}$ is a valid clustering, hence all elements in $\tilde{\calN}_j \setminus \calN_i$ must test positive with all element in $\calN_i$. This can happen only if 1) there exists some other cluster $\calN_p$ such that $\calN_i \cup \{\tilde{\calN}_j \setminus \calN_i\} \subset \calN_p$ but this is not allowed since $\calN_i \not \subset \calN_p$. 2) If there exists two other clusters $\calN_p$ and $\calN_q$ such that $\calN_i \cup \{\tilde{\calN}_j \setminus \calN_i\} \subset \calN_p \cup \calN_q$ but again this is not allowed since $\calN_i \not \subset \calN_p \cup \calN_q$ (same argument as in proof of Lemma~\ref{lem:maximal}). So the previous assumption cannot happen and therefore there cannot exist some $i,j$ such that $\calN_i \subset \tilde{\calN}_j$ and by a similar argument there cannot exist $i,j$ such that $\tilde{\calN_i} \subset \calN_j$. Now, without loss of generality, assume that $\calN_1 \cap \calN_2 \neq \phi$. Hence there must exist some $\tilde{\calN_j}$ such that $\tilde{\calN_j} \cap \calN_1 \cap \calN_2 \neq \phi$. Let us denote one such element $x$ that belongs to  $\tilde{\calN_j} \cap \calN_1 \cap \calN_2$. Now there cannot exist an element $y \in \tilde{\calN_j}\setminus \{\calN_1 \cup \calN_2\}$ because $y$ will test positive with $x$ but $x$ cannot belong to three clusters. Hence it must happen that $\tilde{C}_j \subset \calN_1 \cup \calN_2$. Now, consider two elements $z_1,z_2$ such that $z_1 \in \calN_1 \setminus \calN_2$ and $z_2 \in \calN_2 \setminus \calN_1$ such that $z_1$ and $z_2$ test negative. Such a pair of elements must exist otherwise the clusters $\calN_1,\calN_2$ will not be maximal according to Lemma~\ref{lem:maximal}. Now both the elements $z_1,z_2$ cannot belong to $\tilde{\calN}_j$ since they test negative. On the other hand, both of them cannot be outside $\tilde{\calN}_j$ since if $x$ has to test positive with both $z_1,z_2$ then $x$ must belong to three clusters in $\tilde{\calC}$ which is not allowed again.  Hence, without loss of generality, assume that $z_1$ is contained in $\tilde{C}_j$. If $z_1$ only belongs to $\calN_1$, then obviously no element from $\calN_2\setminus \calN_1$ can belong to $\tilde{\calN}_j$ (because $z_1$ will not test positive with that element) and therefore $\tilde{\calN}_j \subset \calN_1$ which is not allowed. Therefore, assume that $z_1$ also belongs to another cluster $\calN_3$ and under this assumption, further assume that an element $z_3 \in \calN_2 \cap \calN_3$ is contained in $\tilde{\calN}_j$ so that $\tilde{\calN}_j \not \subset \calN_1$. However, according to the guarantee that we are provided, there must exist an element $z_4 \in \calN_1 \setminus \{ \calN_2 \cup \calN_3 \}$ and an element $z_5 \in \calN_2 \setminus \{\calN_1 \cup \calN_3\}$. Now, neither of them can be included in $\tilde{\calN}_j$ since $(z_4,z_3)$ and $(z_5,z_1)$ must test negative. If $(z_4,z_5)$ test negative, then this creates a contradiction since one of them have to be included in $\tilde{\calN_j}$. Now if $(z_4,z_5)$ test positive, then one of $z_4$ and $z_5$ must belong to three clusters in $\tilde{\calC}$ to satisfy the following constraints: $(z_4,x),(z_4,z_1),(z_5,z_3),(z_5,x),(z_4,z_5)$ test positive and $(z_4,z_3),(z_5,z_1)$ test negative ($z_1\in \calN_1 \cup \calN_3$ and $z_5 \in \calN_2\setminus \{\calN_1 \cup \calN_3\}$ and similar for $(z_4,z_3)$) which is not allowed. Hence our initial assumption is incorrect and there cannot be a different clustering $\tilde{\calC}$.
\end{proof}
We are now ready to prove Theorem~\ref{thm:advDelta2}. The proof follows from the following three arguments.
\begin{enumerate}
\item Suppose we randomly sample a subset of elements $\calS$ and let $\tilde\calN_i=\calN_i \cap \calS$ be the set of elements in $\calS$ which belong to the $i$'th cluster. A bad event is if there exist three distinct clusters $p,q,r \in [k]$ such that $\tilde\calN_p \subset \tilde\calN_q \cup \tilde\calN_r$. For a particular triplet of clusters, the probability of this event to happen is clearly upper bounded by $(1-\alpha)^{|\calS|} \le e^{-\alpha |\calS|}$. Taking a union bound over all triplets of clusters, the bad event will happen with probability at most $k^3 e^{-\alpha |\calS|}$. Therefore, taking $\alpha\cdot |\calS|=3\log k+\log n$ will make this probability at most $1/n$. 
\item Now, from Theorem~\ref{thm:delta2}, it is easy to see that once we are given all the queries involving elements in $\calS$, we are able to obtain the ground truth clustering and therefore all the clusters $\tilde\calN_i$ produced by an algorithm that returns a valid clustering.
\item Finally, each element not in $\calS$, will be queried with all elements in $\calS$. If an element belongs to the $i$'th cluster, then obviously it will test positive with all elements in $\tilde\calN_i$. If an element does not belong to the $i$'th cluster (say it belongs to the $j$'th cluster and $k$'th cluster) then it will not test positive with all elements in $\tilde\calN_i$ (because of our guarantee). So we will recover the correct cluster every element belongs to.
\end{enumerate}
It remains to show that Steps $2$ and $3$ in Algorithm~\ref{algo:adv} return a valid clustering if all the queries constrained to elements in $\calS$ are provided.  
We know that all elements that belong to a particular cluster form a clique in the graph. We also know that all the edges can be covered by $k$ maximal cliques (the cliques can be overlapping) such that no three cliques intersect. Hence Step $3$ of Algorithm~\ref{algo:adv} will return a valid clustering, which completes the the proof.
 
Finally, we notice that we can in fact show a necessary condition for the case of $\Delta=2$, which almost coincide with Lemma~\ref{lem:maximal}, hinting that the above conditions might be also necessary.
\begin{lem}\label{lem:advDelta2}
Let $\calC$ be a given clustering and let $\calN_i$ be the set of elements which belong to the $ith$ cluster. If for some pair of distinct clusters $p,q \in [k]$,  $\calN_p \subset \calN_q $, then it is not possible to recover the ground truth clustering.
\end{lem}
\begin{proof}[Proof of Lemma~\ref{lem:advDelta2}]
Consider a pair of clusters $\calN_p,\calN_q$ such that $\calN_p \subset \calN_q $. It is easy to see that it is impossible to determine which elements actually belong to the cluster $\calN_q$ even if all possible query responses are provided.
\end{proof}

\section{Proof of Theorem~\ref{thm:DeltaBigger2_2}}\label{app:GeneralDelta}
We start this section by stating a conjecture which is the natural extension of Theorem~\ref{thm:delta2} to any $\Delta>0$.
\begin{conj}
Let $\calC$ be a given clustering and let $\calN_i$ be the set of elements which belong to the $i$'th cluster. If for every ordered subset of $\Delta+1$ distinct clusters $p_1,p_2,\dots,p_{\Delta+1} \in [k]$, we have $\calN_{p_1} \setminus \{\cup_{p_j \neq p_1} \calN_{p_j}\} \neq \phi$, then the ground truth clustering $\calC$ is the only valid clustering that is consistent with the entire query matrix.
\end{conj}
Unfortunately, we could not prove the above result, but rather the following weaker result.
\begin{thm}\label{thm:DeltaBigger2}
Let $\calC$ be a given clustering and let $\calN_i$ be the set of elements which belong to the $i$'th cluster. If $\calN_i\setminus \{\bigcup_{j\neq i} \calN_j\}$ for all clusters $i \in [k]$, then the ground truth clustering $\calC$ is the only valid clustering that is consistent with the entire query matrix.
\end{thm}
\begin{proof}[Proof of Theorem~\ref{thm:DeltaBigger2}]
Notice that the sets $\calN_i\setminus \{\bigcup_{j \neq i}\calN_j\}$, for all $i \in [k]$, are non-empty and disjoint. In any different clustering $\tilde{\calC}$, the elements belonging to these aforementioned sets have to belong to different clusters. Without loss of generality, assume that $\calN_i\setminus \{\bigcup_{j \neq i}\calN_j \neq \calN_i\} \subset \tilde{\calN_i}$. In that case, for any subset $\calS \subseteq [k]$, it is easy to see that any element in $\bigcap_{s\in \calS}\calN_s$ must belong to  $\bigcap_{s \in S}\tilde{\calN_s}$ since it must test positive with elements in $\calN_i\setminus \{\bigcup_{j \neq i}\calN_j \neq \calN_i\}$ for all $i \in \calS$ and tests negative with elements in $\calN_i\setminus \{\bigcup_{j \neq i}\calN_j \neq \calN_i\}$ for all $i \notin \calS$. With this argument we get that the clustering $\tilde{\calC}$ is the same as the clustering $\calC$.
\end{proof}
We are now in a position to prove Theorem~\ref{thm:DeltaBigger2_2}. The proof hangs on the following three arguments.
\begin{enumerate}
\item Suppose we randomly sample a subset of elements $\calS$ and let $\tilde\calN_i=\calN_i \cap \calS$ be the set of elements in $\calS$ which belong to the $ith$ cluster. A bad event is if there exists a cluster $i \in [k]$ such that $\tilde{\calN_i} \setminus \{\bigcup_{j:j\neq i} \tilde{\calN_j}\}=\phi$. For a particular cluster, the probability of this event is upper bounded by $(1-\alpha)^{|\calS|} \le e^{-\alpha |\calS|}$. Taking a union bound over all clusters, the bad event will happen with probability at most $k e^{-\alpha |\calS|}$. Therefore, taking $\alpha\cdot |\calS|=\log k+\log n$, will make this probability at most $1/n$. 
\item Now, from Theorem~\ref{thm:DeltaBigger2}, it is easy to see that once we are given all the queries involving elements in $\calS$, we are able to obtain the ground truth clustering and therefore all the clusters $\tilde\calN_i$ by an algorithm that returns a valid clustering. If the clusters are maximal, then  Step $3$ in Algorithm \ref{algo:advgen} (a slightly modified version of Algorithm \ref{algo:adv}) returns a valid and unique clustering.
\item  Finally, each element not in $\calS$ will be queried with all elements in $\calS$. If an element belongs to the $i$'th cluster, then obviously it will test positive with all elements in $\tilde\calN_i$. If an element does not belong to the $i$'th cluster then it will not test positive with all elements in $\tilde\calN_i$ (because of our guarantee). So we will recover the correct cluster every element belongs to.
\end{enumerate}

\end{appendices}

\end{document}